\documentclass{article} 
\usepackage{iclr2025_conference,times}
\usepackage{enumitem}

\usepackage{amsmath,amsfonts,bm}

\def\eqref#1{equation~\ref{#1}}

\def\floor#1{\lfloor #1 \rfloor}
\def\1{\bm{1}}

\DeclareMathAlphabet{\mathsfit}{\encodingdefault}{\sfdefault}{m}{sl}
\SetMathAlphabet{\mathsfit}{bold}{\encodingdefault}{\sfdefault}{bx}{n}

\def\gA{{\mathcal{A}}}

\def\gS{{\mathcal{S}}}

\def\gX{{\mathcal{X}}}

\newcommand{\R}{\mathbb{R}}

\usepackage{graphicx}
\usepackage{algorithm}               
\usepackage{algpseudocode}     
\usepackage{tikz}
\usepackage{standalone}

\usetikzlibrary{positioning}
\usetikzlibrary{arrows}
\usetikzlibrary{calc,fit}
\usetikzlibrary{shapes.geometric}
\usetikzlibrary{shapes.misc}
\usetikzlibrary{decorations.pathmorphing}
\usetikzlibrary{decorations.pathreplacing}
\usepackage{cancel}

\usepackage{amsthm}

\usepackage{newtxmath}

\usepackage{siunitx}
\usepackage{makecell}
\usepackage{booktabs}
\usepackage{multirow}
\usepackage{tablefootnote}

\newtheorem{proposition}{Proposition}

\usepackage{graphicx}

\usepackage{subcaption}
\usepackage{float}
\usepackage[font={footnotesize}]{caption} 
\definecolor{Gray}{gray}{0.9}
\usepackage{wrapfig}

\usepackage{setspace}                
\usepackage{hyperref}
\hypersetup{colorlinks=true,allcolors=[rgb]{0.,0.5,0.5}}

\usepackage{url}
\DeclareUnicodeCharacter{0301}{\'{e}}
\usepackage[nameinlink,capitalise]{cleveref}

\newcommand{\ie}{\textit{i.e.}}
\newcommand{\eg}{\textit{e.g.}}
\renewcommand{\eqref}[1]{(\ref{#1})} 

\usepackage{xspace}
\newcommand{\teacher}{\textcolor{\teachercolor}{Teacher}\xspace}
\newcommand{\student}{\textcolor{\studentcolor}{Student}\xspace}
\newcommand{\teacherblack}{\textcolor{black}{Teacher}\xspace}

\newcommand{\teachercolor}{black}
\newcommand{\studentcolor}{black}

\usepackage{caption}
\makeatletter
\crefname{section}{\S\@gobble}{\S\@gobble}
\crefname{subsection}{\S\@gobble}{\S\@gobble}
\crefname{proposition}{Prop.}{Props.}
\crefname{figure}{Fig.}{Figs.}

\def\section{\@startsection {section}{1}{\z@}{-0.3ex}{0.3ex}{\large\sc\raggedright}}
\def\subsection{\@startsection{subsection}{2}{\z@}{-0.2ex}{0.2ex}{\normalsize\sc\raggedright}}
\def\subsubsection{\@startsection{subsubsection}{3}{\z@}{-0.1ex}{0.1ex}{\normalsize\sc\raggedright}}
\def\paragraph{\@startsection{paragraph}{4}{\z@}{0ex}{-1em}{\normalsize\bf}}
\def\subparagraph{\@startsection{subparagraph}{5}{\z@}{0ex}{-1em}{\normalsize\sc}}

\makeatother

\title{Adaptive teachers for amortized samplers}

\iclrfinalcopy

\author{Minsu Kim\thanks{equal contribution, correspondence to \texttt{\{min-su, sanghyeok.choi\}@kaist.ac.kr}}\\
Mila, KAIST 
\And
Sanghyeok Choi$^{*}$\\
KAIST
\And
Taeyoung Yun\\
KAIST
\And
Emmanuel Bengio\\
Recursion
\AND
Leo Feng\\
Mila, Universit\'e de Montr\'eal
\And
Jarrid Rector-Brooks\\
Mila, Universit\'e de Montr\'eal
\And
Sungsoo Ahn\\
KAIST\\
\And
Jinkyoo Park\\
KAIST
\AND
Esmeralda S. Whitammer\\
University of Edinburgh
\And
Yoshua Bengio\\
Mila, Universit\'e de Montr\'eal
}

\begin{document}

\maketitle

\begin{abstract}
Amortized inference is the task of training a parametric model, such as a neural network, to approximate a distribution with a given unnormalized density where exact sampling is intractable. When sampling is implemented as a sequential decision-making process, reinforcement learning (RL) methods, such as generative flow networks, can be used to train the sampling policy. Off-policy RL training facilitates the discovery of diverse, high-reward candidates, but existing methods still face challenges in efficient exploration. We propose to use an adaptive training distribution (the \teacher) to guide the training of the primary amortized sampler (the \student). The \teacher, an auxiliary behavior model, is trained to sample high-loss regions of the \student and can generalize across unexplored modes, thereby enhancing mode coverage by providing an efficient training curriculum. We validate the effectiveness of this approach in a synthetic environment designed to present an exploration challenge, two diffusion-based sampling tasks, and four biochemical discovery tasks demonstrating its ability to improve sample efficiency and mode coverage. Source code is available at \url{https://github.com/alstn12088/adaptive-teacher}.
\end{abstract}

\section{Introduction}

Sampling from a complex distribution given its unnormalized density function is a fundamental problem in machine learning \citep{hinton2002training, lecun2006tutorial} and scientific discovery \citep{dellago1998efficient, noe2019boltzmann}. Amortized inference methods aim to fit a generative model that samples from a target distribution, possibly by a sequence of stochastic generation steps which is beneficial because it allows reusing a shared computational module for inference across multiple data points, as opposed to performing inference independently for each data point \citep{margossian2023amortized}.
However, unlike for generative models trained from data, samples from the ground truth distribution may not be available. Multi-step sampling from an unnormalized density function with amortized inference can be achieved with reinforcement learning (RL) methods but raises the challenge of \textit{exploration} -- specifically, the ability to discover new modes of the target distribution during training. This is due to the intractable size of the sample space and the fact that only sampling from the generator itself would be oblivious to modes that the generator misses.

Just as with generative models trained on data, it is often more natural and beneficial to approximate the generation of objects as a sequence of decisions made by a policy rather than using a single parametric family, due to the multi-modal expressivity of hierarchical inference (\eg, diffusion probabilistic models \citep{ho2020ddpm}). Sequential decision algorithms for amortized inference are unified by the theory of generative flow networks \citep[GFlowNets;][]{bengio2021flow}, which are a collection of off-policy RL methods \citep{tiapkin2023generative,deleu2024discrete}. GFlowNets have been used for such amortized inference problems as natural-language and biological sequence design by token-by-token sequence generation \citep{jain2022biological, shen2023towards, hu2023amortizing}, Bayesian inference over data structures~\citep{deleu2022bayesian}, molecular design by incremental addition of atoms or fragments \citep{bengio2021flow, jain2023gflownets}, or image refinement by a diffusion process in continuous space \citep{venkatraman2024amortizing}, \textit{inter alia}. GFlowNets have shown success in the sequential sampling problems at scale due to their advantageous off-policy training ability \citep{malkin2023gflownets}.

The prudent selection of training data is crucial to the success of such RL methods to model the full distribution faithfully, akin to the problem of active learning in supervised problems: to maximize sample efficiency, the most informative samples should be selected for training. To explore the full distribution, some applications of GFlowNets have either used exploration techniques borrowed from RL, such as noisy exploration \citep[first used by][]{bengio2021flow}, (prioritized) replay buffers \citep{deleu2022bayesian,schaul2015prioritized, vemgal2023empirical}, and delayed updates \citep{lau2023dgfn}. Others have employed search techniques in the target space: MCMC and local search \citep{kim2024local, kim2024ant, sendera2024diffusion, phillips2024metagfn}, genetic algorithms \citep{kim2024genetic}, and exploiting samples from the target distribution \citep{zhang2022generative, hu2023gfnem} when available.

\begin{figure}[t]
\vspace*{-1em}
    \centering
    \begin{tikzpicture}[
    every node/.style={inner sep=0pt},
    flow/.style={very thick,color=flowcolor,dashed},
    graph/.style={minimum width=96pt,minimum height=48pt,draw,very thick,rounded corners=12pt,fill=white},
    y=48pt, 
    x=48pt]

\definecolor{teachercolor}{HTML}{ff0000}
\definecolor{studentcolor}{HTML}{0000ff}
\definecolor{flowcolor}{HTML}{448800}
\definecolor{bgcolor}{HTML}{ccffdd}
\definecolor{bgercolor}{HTML}{cceedd}

\pgfdeclarelayer{background}
\pgfdeclarelayer{backerground}
\pgfsetlayers{backerground,background,main}

\node[minimum width=72pt,minimum height=24pt,draw,very thick,rounded corners=6pt,fill=studentcolor!10] 
(student) at (0, 1) {\large\bf Student};

\node at (0, 0.5) {\large\it or};

\node[minimum width=72pt,minimum height=24pt,draw,very thick,rounded corners=6pt,fill=teachercolor!10]
(teacher) at (0, 0) {\large\bf Teacher};

\node at (0, -0.5) {\large\it or};

\node[minimum width=72pt,minimum height=24pt,draw,very thick,rounded corners=12pt,fill=white]
(buffer) at (0, -1) {\large\bf Buffer};

\node[minimum width=72pt,minimum height=48pt,draw,very thick,rounded corners=6pt,fill=studentcolor!10]
(student_nn) at (5,1) {};

\node[minimum width=72pt,minimum height=48pt,draw,very thick,rounded corners=6pt,fill=teachercolor!10]
(teacher_nn) at (5,-1) {};

\node[anchor=south,inner sep=2pt] at (teacher_nn.south) {
    \begin{tikzpicture}[every node/.style={circle,draw,thick,minimum size=5pt,inner sep=0pt},y=7pt,x=24pt]
        \node (a1) at (0,0) {};
        \node (a2) at (0,-1) {};
        \node (a3) at (0,1) {};
        \node (b1) at (1,1.5) {};     
        \node (b2) at (1,0.5) {};
        \node (b3) at (1,-0.5) {};
        \node (b4) at (1,-1.5) {};
        \foreach \x in {1,2,3} \foreach \y in {1,2,3,4} \draw[gray] (a\x) -- (b\y);
    \end{tikzpicture}
};
\node[anchor=north,inner sep=6pt] at (teacher_nn.north) {
    \large\bf Teacher NN
};

\node[anchor=south,inner sep=2pt] at (student_nn.south) {
    \begin{tikzpicture}[every node/.style={circle,draw,thick,minimum size=5pt,inner sep=0pt},y=7pt,x=24pt]
        \node (a1) at (0,0) {};
        \node (a2) at (0,-1) {};
        \node (a3) at (0,1) {};
        \node (b1) at (1,1.5) {};     
        \node (b2) at (1,0.5) {};
        \node (b3) at (1,-0.5) {};
        \node (b4) at (1,-1.5) {};
        \foreach \x in {1,2,3} \foreach \y in {1,2,3,4} \draw[gray] (a\x) -- (b\y);
    \end{tikzpicture}
};
\node[anchor=north,inner sep=6pt] at (student_nn.north) {
    \large\bf Student NN
};

\node[minimum width=72pt,minimum height=48pt,draw,very thick,rounded corners=12pt,fill=white]
(true_dist) at (2.5,1) {
    \begin{tikzpicture}[scale=0.25]
        \draw[thick,color=gray,rounded corners=0pt,smooth,fill=gray!20] 
        plot[domain=-3:3,samples=200] 
        ({\x},{3*exp(-(\x-1)*(\x-1)*8)+1.5*exp(-(\x+1)*(\x+1)*2)});
    \end{tikzpicture}
};

\draw[flow]
(true_dist) -- ([xshift=-24pt]student_nn.west)
node[pos=1,above,color=black,inner sep=4pt] {\large rewards};

\draw[-latex,flow]
([yshift=-5pt]teacher_nn.west) -- ([yshift=-5pt]buffer.east) node[pos=0.5,below,color=black,inner sep=4pt] {\large experiences};

\draw[-latex,flow] ([xshift=-10pt]student_nn.south) |- ([xshift=-15pt,yshift=40pt]teacher_nn.west) |- ([yshift=5pt]buffer.east);

\draw[-latex,flow]
(student_nn) -- (teacher_nn) node[pos=0.5,right,color=black,inner sep=4pt] {\begin{minipage}{72pt}\large loss values\\as rewards\end{minipage}};

\draw[-latex,very thick,color=studentcolor!50]
(student_nn.east) -- ++(0.25,0)
-- ++(0,0.75) node[pos=0.5,right,color=studentcolor!50,inner sep=4pt] (training_student) {\rotatebox{-90}{\large training}}
-| (student_nn.north);

\draw[-latex,very thick,color=teachercolor!50]
(teacher_nn.east) -- ++(0.25,0)
-- ++(0,-0.75) node[pos=0.5,right,color=teachercolor!50,inner sep=4pt] (training_teacher) {\rotatebox{-90}{\large training}}
-| (teacher_nn.south);

\begin{pgfonlayer}{background}
    \node[inner sep=12pt,fit=(student) (teacher) (buffer),rounded corners=12pt,fill=bgcolor,draw,very thick]
    (behaviour_policies) {};
\end{pgfonlayer}

\draw[-latex,flow]
(behaviour_policies) 
-| ([xshift=-24pt]student_nn.west) node[pos=0.25,below,color=black,inner sep=4pt] {\large training trajectories} 
-- (student_nn.west);

\node[above=0pt of behaviour_policies,inner sep=4pt] (behaviour_label) {\large\bf Behavior policy};

\node[above=0pt of true_dist,inner sep=4pt] {\large\bf Target distribution\vphantom{y}};

\node[graph]
(current_student) at (8.5,1) {
    \begin{tikzpicture}[scale=0.25]
        \draw[thin,color=gray,rounded corners=0pt,smooth,fill=studentcolor!20] 
        plot[domain=-3:3,samples=200] 
        ({\x},{3*exp(-(\x-1)*(\x-1)*8)});
        \draw[thick,color=black,rounded corners=0pt,smooth,dashed] 
        plot[domain=-3:3,samples=100] 
        ({\x},{3*exp(-(\x-1)*(\x-1)*8)+1.5*exp(-(\x+1)*(\x+1)*2)});
    \end{tikzpicture}
};

\node[graph]
(current_teacher) at (8.5,-1) {
    \begin{tikzpicture}[scale=0.25]
        \draw[thin,color=gray,rounded corners=0pt,smooth,fill=teachercolor!20] 
        plot[domain=-3:3,samples=200] 
        ({\x},{3*exp(-(\x+1)*(\x+1)*2)});
        \draw[thick,color=black,rounded corners=0pt,smooth,dashed] 
        plot[domain=-3:3,samples=100] 
        ({\x},{3*exp(-(\x-1)*(\x-1)*8)+1.5*exp(-(\x+1)*(\x+1)*2)});
    \end{tikzpicture}
};

\node[graph]
(stationary_student) at (12.5,1) {
    \begin{tikzpicture}[scale=0.25]
        \draw[thin,color=gray,rounded corners=0pt,smooth,fill=studentcolor!20] 
        plot[domain=-3:3,samples=200] 
        ({\x},{3*exp(-(\x-1)*(\x-1)*8)+1.5*exp(-(\x+1)*(\x+1)*2)});
        \draw[thick,color=black,rounded corners=0pt,smooth,dashed] 
        plot[domain=-3:3,samples=100] 
        ({\x},{3*exp(-(\x-1)*(\x-1)*8)+1.5*exp(-(\x+1)*(\x+1)*2)});
    \end{tikzpicture}
};

\node[graph]
(stationary_teacher) at (12.5,-1) {
    \begin{tikzpicture}[scale=0.25]
        \draw[thin,color=gray,rounded corners=0pt,smooth,fill=teachercolor!20] 
        plot[domain=-3:3,samples=200] 
        ({\x},{(3*exp(-(\x-1)*(\x-1)*8)+1.5*exp(-(\x+1)*(\x+1)*2))^0.2-0.2});
        \draw[thick,color=black,rounded corners=0pt,smooth,dashed] 
        plot[domain=-3:3,samples=100] 
        ({\x},{3*exp(-(\x-1)*(\x-1)*8)+1.5*exp(-(\x+1)*(\x+1)*2)});
    \end{tikzpicture}
};

\node[above=0pt of current_student,inner sep=4pt] (current_label) {\large\bf Current distribution\vphantom{y}};

\node[above=0pt of stationary_student,inner sep=4pt] (stationary_label) {\large\bf Converged distribution};

\node[above=0pt of current_teacher,inner sep=4pt,color=teachercolor!50] {\large\bf Teacher};
\node[above=0pt of stationary_teacher,inner sep=4pt,color=teachercolor!50] {\large\bf Teacher};
\node[below=0pt of current_student,inner sep=4pt,color=studentcolor!50] {\large\bf Student};
\node[below=0pt of stationary_student,inner sep=4pt,color=studentcolor!50] {\large\bf Student};

\node[minimum width=24pt,scale=2] (mid_student) at (10.5,1) {\raisebox{-0.25em}[0em][-0.25em]{$\Huge\bf\cdots$}};
\node[minimum width=24pt,scale=2] (mid_teacher) at (10.5,-1) {\raisebox{-0.25em}[0em][-0.25em]{$\Huge\bf\cdots$}};

\draw[-latex,ultra thick,shorten <= 8pt] (current_student) -- (mid_student);
\draw[-latex,ultra thick,shorten <= 8pt] (current_teacher) -- (mid_teacher);
\draw[-latex,ultra thick,shorten >= 8pt] (mid_student) -- (stationary_student);
\draw[-latex,ultra thick,shorten >= 8pt] (mid_teacher) -- (stationary_teacher);

\begin{pgfonlayer}{backerground}
    \node[yshift=-6pt,inner sep=12pt,rounded corners=12pt,fit=(behaviour_label)(training_student)(teacher_nn),fill=bgercolor!30] {};
    \node[yshift=-6pt,inner sep=12pt,rounded corners=12pt,fit=(current_label)(stationary_label)(stationary_teacher),fill=bgercolor!30] {};
\end{pgfonlayer}

\end{tikzpicture}
    \caption{Training an amortized sampler (\student) with an adaptive \teacher. \textbf{Left}: The behavior policy mixes \student, \teacher, and replay buffer policies to generate trajectories that train \student and store experiences. \teacher is updated based on \student's loss. \textbf{Right}: \student and \teacher distributions co-evolve, with \teacher targeting uncovered modes until \student converges to the target distribution.   }
    \label{fig:main-figure}
    \vspace{-2.5em}
\end{figure}
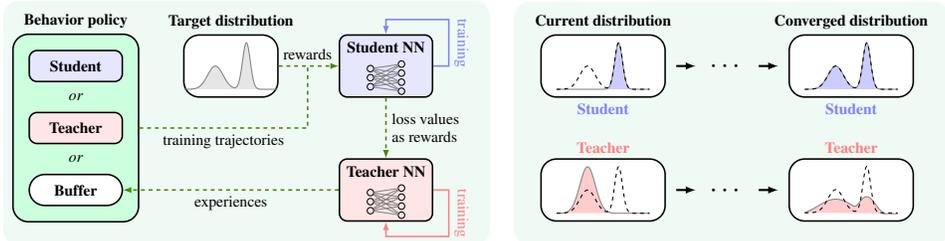

However, challenges in mode coverage remain. All of the forementioned methods promote exploration through perturbation of the policy, \eg, replaying the samples \citep{vemgal2023empirical}, augmenting the reward function \citep{pan2023generative}, and local search starting from generated samples~\citep{kim2024local}. 
These exploration methods focus on capturing the modes that are already near those generated by the current policy and can hardly capture the ones sufficiently separated from the already explored modes. 

In this paper, we propose to explicitly explore the regions of high loss by introducing a \textbf{\teacher} model that guides the training of the primary, or \textbf{\student} sampler (\Cref{fig:main-figure}). Here, we believe that trajectories with high loss are particularly informative for mode coverage, as they are likely to lead to regions of the target distribution that are either undersampled (dropped modes) or oversampled (collapsed modes).
The \textbf{\teacher} is an adaptive behavior policy that is trained to sample target space regions where the \textbf{\student} model receives a high loss. In turn, the \textbf{\student} model is trained on samples from the \textbf{\teacher} model. 

Our approach can be seen as amortizing an ideal prioritized experience replay \citep[PER;][]{schaul2015prioritized} which samples high-loss objects from the entire sample space, instead of the finite-size replay buffer. Compared to (non-ideal) PER, the \textbf{\teacher} model has the potential to generalize across the high-loss regions of the {\bf \student}, without regard for whether they have previously been sampled or discovered. 

In contrast, prioritized replay requires the poorly captured modes to have already been visited in order for the model to learn from them.

We test our method on a diverse set of domains where GFlowNets have been used, including discrete tasks (biological sequence design and molecular design) and continuous tasks (diffusion-based sampling benchmarks). Comprehensive experiments demonstrate that our algorithm is effective in improving mode coverage and training efficiency across all tasks.

\section{Preliminaries}
\label{sec:preliminaries}

We give a summary of GFlowNets as algorithms for amortized sampling by sequential decision making. For simplicity, this exposition is about discrete space, where GFlowNets were originally defined \citep{bengio2021flow}, but GFlowNets have been extended to the case of continuous variables \citep{lahlou2023theory}, which is conceptually similar. The reader is directed to \citet{bengio2023gflownet} for an extended overview, \citet{malkin2023gflownets} for an introduction focused on connections to hierarchical variational inference, and to \citet{deleu2024discrete} for a maximum-entropy RL point of view.

GFlowNets are policies in deterministic Markov decision processes (MDPs), trained so as to sample from a distribution over terminal states whose mass function is proportional to a given reward. The MDP is assumed to be represented as a finite directed acyclic graph $G=(\gS,\gA)$, where $\gS$ is the set of \emph{states} and $(s\rightarrow s')\in\gA$ if there is a possible action to be taken at state $s$ leading to state $s'$. A policy is then the same as a collection of distributions $P_F(\cdot\mid s)$ over the children\footnote{If $(s\rightarrow s')\in\gA$, then $s'$ is a \emph{child} of $s$; the converse relation is called \emph{parent}.} of every state $s$ that has at least one child. As in other deep RL methods, the policy could be a neural network $P_F(s'\mid s;\theta)$ taking a representation of the state $s$ as input and outputting the logits of a distribution over children $s'$. We will sometimes leave out the $\theta$ as implicit to lighten notation.

We assume the existence of a unique state $s_0\in\gS$, called the \emph{initial state} that is not the child of any state. Conversely, states without children are called \emph{terminal}, and the set of terminal states is denoted $\gX$. A policy $P_F$ induces a distribution over \emph{complete trajectories} -- sequences $\tau=s_0\rightarrow s_1\rightarrow\dots\rightarrow s_n$ with $s_n\in\gX$ -- which can be sampled by starting at $s_0$ and iteratively transitioning to child states sampled according to $P_F$ until a terminal state is reached. This in turn a \emph{terminating distribution} $P_F^\top$ over $\gX$, which is the marginal distribution over the final states of trajectories sampled in this way. To be precise,
\begin{equation}\label{eq:terminating_dist}
    P_F^\top(x)=\sum_{\tau\rightsquigarrow x}P_F(\tau),\quad P_F(\tau=(s_0\rightarrow\dots\rightarrow s_n)):=\prod_{i=0}^{n-1}P_F(s_{i+1}\mid s_i),
\end{equation}
where $x\in \gX$ and $\tau\rightsquigarrow x$ indicates that the sum is restricted to trajectories $\tau$ whose last state is $x$.

Let $R:\gX\to\R_{>0}$ be a function on the terminal states, called the \emph{reward function}, and set $Z:=\sum_{x\in\gX}R(x)$ as the normalization constant. We would like to train $P_F$ so as to make $P_F^\top(x)=R(x)/Z$ for all $x\in\gX$, \ie, to make $P_F$ sample terminal states with probability proportional to the reward. 

Because the sum in \eqref{eq:terminating_dist} may be intractably large (if many trajectories could lead to the same $x$), achieving this requires introducing auxiliary objects into the optimization. One popular option is to use the trajectory balance (TB) objective \citep{malkin2022trajectory}. To train a model with TB, one introduces an additional \emph{backward policy} $P_B(\cdot\mid\cdot)$, which is a collection of distributions over the parents of every noninitial state (\ie, a policy on the reverse MDP, which can be either fixed -- as in our case -- or learned), as well as an estimate of the total reward $Z_\theta$ (usually parametrized in the log domain, making $\log Z_\theta$ a learnable parameter). We define the TB discrepancy for a trajectory $\tau$ with final state $x$ by
\begin{equation}\label{eq:tb}
\delta(\tau; \theta):= \underbrace{[\log R(x) + \log P_B(\tau\mid x)]}_{\text{backward flow}} - \underbrace{[\log Z_\theta + \log P_F(\tau; \theta)]}_{\text{forward flow}},
\end{equation}
where $P_B(\tau\mid x;\theta)$ is defined analogously to \eqref{eq:terminating_dist}, by
\begin{equation}
 P_B(\tau=(s_0\rightarrow\dots\rightarrow s_n)\mid x)=\prod_{i=0}^{n-1}P_B(s_i\mid s_{i+1}).   
\end{equation}
It can be shown that if $\delta(\tau;\theta)=0$ for all trajectories $\tau$, then $Z_\theta=Z$ and $P_F^\top(x)=R(x)/Z$ for all $x$, meaning that $P_F$ solves the sampling problem. Intuitively, this is the case because the reward distribution and $P_B$ would then determine the same distribution over trajectories as $P_F$, but factorized in reverse order. One thus attempts to enforce this by minimizing a loss, such as $\delta(\tau; \theta)^2$, on trajectories $\tau$ sampled from some behaviour policy $\pi$. (If $\pi$ is the current policy $P_F$ itself, the optimization is said to be \emph{on-policy}, otherwise \emph{off-policy}.)

We note that there exist other training procedures that use learned estimators and loss functions associated with individual states or transitions, rather than full trajectories, such as detailed balance \citep[DB;][]{bengio2023gflownet} and subtrajectory balance \citep[SubTB;][]{madan2023learning}, which all have advantages under certain conditions. This paper mostly focuses on TB due to its simplicity and popularity as a default choice in the literature; we investigate DB in \cref{app:db} to show our method's flexibility over objective functions.

\section{The \teacherblack: An adaptive training distribution}
\label{sec:method}

In this section, we introduce the \textbf{\teacher}, which is a secondary GFlowNet designed to enhance the efficiency of off-policy training for the primary, or \textbf{\student}, GFlowNet. The two GFlowNets share the same state and action space, but have different rewards.

The \teacher's role is to generate an adaptive training distribution for the \student, aiming to sample trajectories that yield high loss for the \student. Intuitively, samples with high loss tend to be less (or never) visited by the \student, implying high probability of the samples being in the unexplored modes. To this end, we train the \teacher with GFlowNet objective for amortization of sampling trajectories with high loss. Note that the \student's target distribution does not depend on the \teacher, but the \teacher's target distribution depends on the \student.

We henceforth denote the parameters of the \student GFlowNet by $\theta$ and those of the \teacher by $\phi$.

\subsection{Reward design for \teacherblack}
\label{sec:teacher-reward}

We define the reward function for the \teacher using the TB loss of the \student, \(\delta(\tau; \theta)^2\). In basic form, we could define the \teacher's reward as
\begin{equation}
    \log R^{\rm basic}_{\text{\teacher}}(x; \theta) = \mathbb{E}_{P_B(\tau \mid x; \theta)} \left[\log \left( \delta \left(\tau; \theta\right)^2 \right)\right] .\label{eq:teacher-reward0} 
\end{equation}
In~\cref{eq:teacher-reward0}, the \student's loss is marginalized over trajectories $\tau$ in the log domain over the backward policy \(P_B(\tau \mid x; \theta)\) of the \student, given a terminal state $x$. This is because we aim to train the \teacher as a sampler of terminal states, although what we obtain when training the \student is a \emph{trajectory-level} error $\delta(\tau; \theta)$, which we need to convert into a function of the terminal state $x$ only to form the \teacher's reward.  Having the \teacher model terminal states $x$ rather than full trajectories $\tau$ is motivated by the desire to obtain full mode coverage in the space of terminal states, but not necessarily in the space of trajectories that lead to these terminal states. 
In practice, this expectation is estimated using Monte Carlo sampling with a single sample $\tau\sim P_B(\tau \mid x; \theta)$, relying on the fact that stochastic gradient descent training of the \teacher will automatically average out the variability resulting from this sampling. In fact, this gives an unbiased estimator of the gradient if training with the full expectation \citep[see, \eg,][]{deleu2022bayesian,bengio2023gflownet}.

We propose two modifications to \eqref{eq:teacher-reward0} to facilitate mode discovery.

\paragraph{Favoring undersampled regions.} First, we hypothesize that because the \teacher should encourage the \student to discover unvisited modes, it should favor regions of the state space where the target density exceeds the \student's sampling probability. To this end, we increase the weight of the \teacher's reward for states where the backward flow exceeds the forward flow (cf.\ \eqref{eq:tb}), while adding a smoothing constant $\epsilon$:
\begin{align}
    \log R_{\text{\teacher}}^{\rm weighted}(x; \theta) = \mathbb{E}_{P_B(\tau \mid x; \theta)} \left[\log \left( \epsilon + \left(1 + C\mathbb{I}_{\delta \left(\tau; \theta\right)>0}\right)\delta \left(\tau; \theta\right)^2 \right)\right] ,\label{eq:teacher-reward1} 
\end{align}
The weighted term $(1 + C\mathbb{I}_{\delta \left(\tau; \theta\right)>0})$ gives additional weight when the TB discrepancy is positive, which indicates the \student is undersampling a high-rewarded terminal sample. Here, $C > 0$ represents the weighting constant, which we set to $C=19$ for every task; see \cref{app:C-ablation} for ablation study on our choice of $C$.

\paragraph{Reward mixing.} To ensure the \teacher covers the missing modes (the high-reward regions that the \student missed), it is important to focus the \teacher's search space more on the high-reward regions than the low-reward ones. This approach helps target both high-loss and high-reward areas effectively. To achieve this, we propose to mix the reward \eqref{eq:teacher-reward1}, which is based on the \student's loss, with the \student's log reward: 
\begin{equation} \label{eq:mixing} \log {R}_{\text{\teacher}}(x):= \log R^{\rm weighted}_{\text{\teacher}}(x) + \alpha \log R(x). \end{equation} 
This approach encourages the \teacher to sample regions with both high loss and high reward. Here the mixing constant $\alpha$ is a hyperparameter that trades off between high loss and high reward; see \cref{app:alpha} for analysis of the effect of the choice of $\alpha$.

\subsection{Jointly training \teacher and \student}
\label{sec:training}

Using the $R_{\text{\teacher}}(x; \theta)$, 
the training process is a joint optimization of the \teacher parameters \(\phi\) and the \student parameters \(\theta\) to jointly minimize the following loss functions:
\begin{align}
    \mathcal{L}_{\text{\student}}(\tau; \theta) &= \delta(\tau; \theta)^2 = \left( \log \frac{Z_{\theta} P_F(\tau; \theta)}{R(x) P_B(\tau \mid x)} \right)^2,
    \label{eq:student_loss}
    \\
    \mathcal{L}_{\text{\teacher}}(\tau; \phi) &= \delta_{\text{\teacher}}(\tau;  \phi)^2 = \left( \log\frac{Z_{\phi} P_F(\tau; \phi)}{R_{\text{\teacher}}(x; \theta) P_B(\tau \mid x)} \right)^2,
    \label{eq:teacher_loss}
\end{align}
Notice that \eqref{eq:student_loss} is the loss for regular TB training of the \student, while \eqref{eq:teacher_loss} relies on the \student's loss to provide a reward for the \teacher via \eqref{eq:mixing}.

To simplify the training process, we adopt in our experiments a fixed backward policy \(P_B(\tau \mid x)\), without trainable parameters, which is used by both the \teacher and the \student.

\begin{figure}[t!]
\vspace*{-2em}
\centering
\begin{minipage}[t!]{0.9\textwidth}
\centering
\begin{algorithm}[H]
\small
\begin{spacing}{1.05}
\caption{\teacher-\student Training of GFlowNets}
\label{algorithm1}
\small
\begin{algorithmic}[1]
\State $\mathcal{Q}_{\text{buffer}} \leftarrow \emptyset$ \Comment{{\it Initialize replay buffer with queue structure}}
\For{$t=1,\ldots, T$} \Comment{{\it Iteration of training rounds}}
\State Select behavior policy $P_{\beta}(\tau)=
\begin{cases} 
P_F(\tau;\theta), & \text{if select \student}  \\
P_F(\tau; \phi), & \text{if select \teacher} \\
P_B(\tau|x)P(x|\mathcal{Q}_{\text{buffer}}), & \text{if select Prioritized Buffer}.
\end{cases}$
\label{line3}
\State Sample trajectories $\tau_1, \ldots, \tau_B \sim P_{\beta}(\tau)$.  \Comment{{\it Exploration}}

\State \textcolor{gray}{(Optional) Refine trajectories using local search.}
\label{algo:local_search}
\State Compute rewards: $R(x_1), \ldots, R(x_B)$.
\label{line6}
\State Compute TB discrepancy of \student: $\delta(\tau_1;\theta), \ldots, \delta(\tau_B; \theta)$.
\label{line7}
\State Compute $R_{\text{\teacher}}(x_1), \ldots, R_{\text{\teacher}}(x_B)$ using $\{R(x_i)\}_{i=1}^B$ and $\{\delta(\tau_i; \theta)\}_{i=1}^B$.
\label{line8}
\State Compute TB discrepancy of \teacher: $\delta(\tau_1;\phi), \ldots, \delta(\tau_B; \phi)$.
\label{line9}
    \State Update \student parameters: $\theta \gets  \text{Optimizer}\left(\frac{1}{B}\sum_{i=1}^B \delta(\tau_i; \theta)^2\right)$ 
    \Comment{{\it \student training}}
    \label{line10}
    \State Update \teacher parameters: $\phi \gets  \text{Optimizer}\left(\frac{1}{B}\sum_{i=1}^B \delta(\tau_i; \phi)^2\right)$ 
    \Comment{{\it \teacher training}}
    \label{line11}
\State Add experiences to buffer: $\mathcal{Q}_{\text{buffer}} \leftarrow \mathcal{Q}_{\text{buffer}} \cup \{x_i, R(x_i), R_{\text{\teacher}}(x_i)\}_{i=1}^B$ 

\EndFor
\end{algorithmic}
\end{spacing}

\end{algorithm}

\end{minipage}
\vspace{-2em}
\end{figure}

\paragraph{Behavior policy for joint optimization.} Algorithm~\ref{algorithm1} describes the \teacher-\student training procedure, which simultaneously minimizes the loss functions of both the \teacher and the \student. In line~\ref{line3}, we select the behavior policy by choosing either the \student, the \teacher, or a prioritized buffer, ensuring that all three are sufficiently utilized (see \cref{app:implementation} for details on how they are chosen). Given a terminal state $x$ sampled from $P(x \mid Q_{\text{buffer}})$, we then generate a trajectory $\tau$ using the backward policy $P_B(\tau \mid x)$. This approach is similar to previous works \citep{shen2023towards, sendera2024diffusion}, which store only the terminal states $x$ in the buffer and sample trajectories $\tau$ using $P_B$.

The behavior policy can produce an adaptive distribution of trajectories $\tau$ with respect to the \student's learning state $\theta$ because the \teacher iteratively focuses on high-loss trajectories of the \student during training. This adaptivity is hypothesised to result in highly effective training for the \student.

\paragraph{Existence of a stationary point of training process.} The joint optimization for the parameters \(\phi\) and \(\theta\) over the support of $P_{\beta}(\tau)$ has a stationary point where the \student GFlowNet becomes an exact sampler and the \teacher GFlowNet samples proportional to $\epsilon R(x)^{\alpha}$; see \cref{thm:joint-optimum} in \cref{app:theory}.

\subsection{Mitigating non-stationarity with local search}
\label{sec:ls}

Joint optimization of the parameters $\phi, \theta$ with a non-stationary target $R_{\text{\teacher}}(x; \theta)$ poses significant challenges as the \teacher's reward is nonstationary, evolving as the \student learns. To address this issue, we use a local search method (\cref{algo:local_search}) that locally optimizes $R_{\text{\teacher}}(x; \theta)$ based on the \teacher's samples. We expect the dynamic nature of $R_{\text{\teacher}}(x; \theta)$ to be effectively managed by such search, as the \teacher's main role is to generalize to modes poorly modeled by the \student, while the search helps the \teacher track the \emph{local} changes in the \student's loss landscape.

Local search using a kernel defined by the policies -- first used by \citet{zhang2022generative,hu2023gfnem} and extensively studied by \citet{kim2024local} -- involves iteratively backtracking trajectories and reconstructing them to produce new samples. The method consists of the following steps:
\begin{enumerate}[left=0pt,nosep]
    \item \textbf{Backtracking}: Starting from a terminal state $x$, we backtrack to an intermediate state $s$ using the backward policy $P_B$, denoted as $(x \dashrightarrow \ldots \dashrightarrow s)$.
    \item \textbf{Reconstruction}: From the intermediate state $s$, we reconstruct a new terminal state $x'$ using the \teacher's forward policy $P_F$, represented as $(s \rightarrow \ldots \rightarrow x')$.
    \item \textbf{Accept or reject}: We accept the new sample $x'$ in place of $x$ with acceptance probability $A$. 
\end{enumerate}
The acceptance probability \( A \) can be determined using either a stochastic Metropolis-Hastings (MH) approach or a deterministic ascent criterion (see \citet{kim2024local} for details). This process is repeated iteratively to progressively improve the samples so that their reward better matches (with MH) or locally maximizes (with the deterministic ascent version) the target reward function. Ultimately, we use the enhanced sample $x'$ to train both the \teacher and the \student by generating trajectories \(\tau \sim P_B(\tau\mid x')\) and taking gradient steps on the losses \eqref{eq:student_loss} and \eqref{eq:teacher_loss}.

\renewcommand{\teachercolor}{black}
\renewcommand{\studentcolor}{black}

\section{Related work}
\textbf{GFlowNets.} GFlowNets were originally introduced by \citet{bengio2021flow} and extensively extended by \citet{bengio2023gflownet}. They aim to develop a sequential decision-making policy with a form of deep reinforcement learning that aims at sampling from the unnormalized density associated with a positive reward function. Aiming to improve credit assignment over long trajectories, \citet{malkin2022trajectory} introduced the trajectory balance (TB) objective mainly used in this paper. Building on this, \citet{madan2023learning} introduced a mixing scheme that combines losses associated with subtrajectories, trading off the lower variance of DB with the lower bias of TB, \citet{pan2023better} studied inductive biases that use partial reward information, and \citet{jang2023learning} extended this idea to learnable reward shaping schemes. \citet{shen2023towards} and \citet{jang2024pessimistic} studied auxiliary losses for better training of the backward policy. 

Orthogonal to those studies, other works focus on improving off-policy training. \citet{deleu2022bayesian,shen2023towards, vemgal2023empirical} studied the use of replay buffers in GFlowNets to enhance sample efficiency. \citet{kim2024local, kim2024genetic, kim2024ant, sendera2024diffusion} investigated local search methods to guide GFlowNets toward high-reward regions. \citet{kim2023learning} proposed to adjust the exploration-exploitation trade-off via amortized conditioning on reward temperature. Similarly, \citet{qgfn} introduced a method that mixes Deep Q-Networks \citep{mnih2013playing} (exploitation) with GFlowNets (exploration) to balance the exploration-exploitation trade-off. Our proposed method is also an off-policy training approach for GFlowNets. In contrast to the methods above, which use off-policy training to focus GFlowNets on high-reward regions, our method aims to address missing modes and underexplored regions. Note that the aforementioned reward-seeking off-policy methods are complementary to our approach; for example, local search with a \teacher is studied in \cref{sec:grid}.

While the above algorithmic work mostly concerns discrete space, \citet{lahlou2023theory} introduced the theory of GFlowNets in continuous space, leading to subsequent work on diffusion samplers~\citep{zhang2023diffusion,sendera2024diffusion}, posterior sampling under diffusion priors  \citep{venkatraman2024amortizing}, and applications to molecular dynamics \citep{seong2024collective}. Our proposed algorithms are effective in both discrete and continuous space (\cref{sec:diffusion}).

\textbf{Adaptive training distributions.} Adaptive training distributions are essential techniques in deep learning, ensuring that the training data evolves appropriately during model training. For instance, curriculum learning methods \citep{bengio2009curriculum} schedule the difficulty of training tasks by gradually increasing from easy to hard, thereby facilitating more efficient model training. These methods are also widely applied in reinforcement learning, \eg, \citet{narvekar2020curriculum}.

Active learning \citep{gal2017deep}, which is usually built for supervised learning, also falls into this category, where the training dataset actively changes to discover better strategies. Especially, uncertainty sampling-based methods \citep{sener2017active, yoo2019learning, kirsch2019batchbald, ash2019deep} that prioritize sampling data points having high predictive uncertainty of the classifier are relevant to our idea. The key difference is that our method is built for reinforcement learning, where we do not rely on the predictive uncertainty of a classifier but on the loss value defined by the compositional policy.

Our work is also relevant to few-shot experimental design \citep{wang2024robust}, as our Teacher plays the similar role as the entropy-regularized adversary that generates tasks.

\section{Experiments}

This section provides empirical validation of our method. Our primary goal is to demonstrate that our approach is beneficial over other off-policy training methods for trajectory balance (TB), a representative learning objective for amortized samplers. We also aim to show that our method can effectively be integrated with existing off-policy search methods, such as local search and replay buffer techniques. We benchmark our approach on three major tasks: two discrete tasks (\cref{sec:grid,sec:bioseq}) and one continuous task (\cref{sec:diffusion}). We also include three ablation studies and a pilot experiment on integrating local search in \cref{app:exp}. Additionally, we tested the versatility of our method by adopting another training objective, detailed balance~\citep[DB;][]{bengio2023gflownet} in~\cref{app:db}.

\subsection{Deceptive grid world}
\label{sec:grid}

\paragraph{Setting.} The deceptive grid world is a synthetic environment modified from the grid task introduced by \citet{bengio2021flow}. It consists of a $d$-dimensional hypercube of side length $H$, resulting in a search space of size $\mathcal{O}(H^d)$. The agent starts from the origin (position $\bm{0}$) and can only move in directions that increase a coordinate by 1 or terminate to receive the reward. The reward of each terminal state $x=(x_1,\ldots,x_d)$ is given by
\begin{equation}
R(x) = R_0 + R_1 \prod_{i=1}^{d} \mathbb{I}\left[ 
\left\vert \frac{x_i}{H - 1} - 0.5 \right\vert < 0.1 \right] + R_2 \prod_{i=1}^{d} \mathbb{I} \left[ 
\left\vert \frac{x_i}{H - 1} - 0.5 \right\vert \in (0.3, 0.4) \right].    
\end{equation}

The modes with the highest rewards ($R_2 + R_0$) are surrounded by walls with low rewards ($R_0$). In between them are deceptive regions offering relatively high rewards ($R_1 + R_0$), which can lure the agent into getting trapped. We set $R_0=10^{-5}$, $R_1=0.1$, and $R_2=2$. Following previous works, we use the number of modes discovered and the empirical $L_1$ distance from the target distribution as evaluation metrics. See~\cref{app:gridworld} for more details about the settings.

\begin{table*}[t!]
\vspace*{-1em}
    \caption{Evaluation results on deceptive grid worlds with dimension $d$ and grid length $H$. The number of modes discovered (\# modes) and empirical $L_1$ distance between target and sampled distributions are reported as mean $\pm$standard deviation over five runs. The $L_1$ distances are scaled appropriately for readability.}\vspace*{-0.75em}
    \label{tab:grid-results}
    \resizebox{1\linewidth}{!}{
    \begin{tabular}{@{}lcccccccc}
        \toprule
        Grid config. $\rightarrow$
        & \multicolumn{2}{c}{$d = 2$, $H = 128$}
        & \multicolumn{2}{c}{$d = 2$, $H = 256$}
        & \multicolumn{2}{c}{$d = 4$, $H = 16$}
        & \multicolumn{2}{c}{$d = 4$, $H = 32$}
        \\
        \cmidrule(lr){2-3}\cmidrule(lr){4-5}\cmidrule(lr){6-7}\cmidrule(lr){8-9}
        Algorithm $\downarrow$ Metric $\rightarrow$ 
        & \# modes ($\uparrow$) & $L_1 \times 10^{-5}$ ($\downarrow$) 
        & \# modes ($\uparrow$) & $L_1 \times 10^{-5}$ ($\downarrow$) 
        & \# modes ($\uparrow$) & $L_1 \times 10^{-5}$ ($\downarrow$) 
        & \# modes ($\uparrow$) & $L_1 \times 10^{-6}$ ($\downarrow$) \\
        \midrule
        TB \hfill(on-policy $\rightarrow$)
        & $645.4 \pm \text{\footnotesize 41.5}$ & $2.20 \pm \text{\footnotesize 0.58}$ 
        & $733.6 \pm \text{\footnotesize 25.1}$ & $1.74 \pm \text{\footnotesize 0.04}$ 
        & $6.6 \pm \text{\footnotesize 2.5}$ & $1.027 \pm \text{\footnotesize 0.012}$ 
        & $16.6 \pm \text{\footnotesize 4.8}$ & $1.635 \pm \text{\footnotesize 0.000}$ 
        \\
        \hspace{2pt} + $\epsilon$-expl.
        & $555.2 \pm \text{\footnotesize 66.2}$ & $3.59 \pm \text{\footnotesize 0.66}$ 
        & $672.6 \pm \text{\footnotesize 16.3}$ & $1.75 \pm \text{\footnotesize 0.02}$ 
        & $6.6 \pm \text{\footnotesize 4.2}$ & $1.030 \pm \text{\footnotesize 0.006}$ 
        & $24.2 \pm \text{\footnotesize 3.2}$ & $\mathbf{1.634} \pm \text{\footnotesize 0.000}$ 
        \\
        \hspace{2pt} + GAFN
        & $675.4 \pm \text{\footnotesize 0.5}$ & $\mathbf{1.66} \pm \text{\footnotesize 0.11}$ 
        & $1044.8 \pm \text{\footnotesize 276.4}$ & $1.55 \pm \text{\footnotesize 0.17}$ 
        & $11.8 \pm \text{\footnotesize 3.9}$ & $1.057 \pm \text{\footnotesize 0.022}$ 
        & $24.6 \pm \text{\footnotesize 7.4}$ & $1.664 \pm \text{\footnotesize 0.002}$ 
        \\
        \hspace{2pt} + PRT
        & $\mathbf{676.0} \pm \text{\footnotesize 0.0}$ & $4.54 \pm \text{\footnotesize 0.14}$ 
        & $2165.2 \pm \text{\footnotesize 64.5}$ & $1.55 \pm \text{\footnotesize 0.05}$ 
        & $38.8 \pm \text{\footnotesize 10.0}$ & $1.097 \pm \text{\footnotesize 0.006}$ 
        & $120.4 \pm \text{\footnotesize 19.1}$ & $1.648 \pm \text{\footnotesize 0.001}$ 
        \\
        \hspace{2pt} + PER
        & $669.0 \pm \text{\footnotesize 3.8}$ & $4.88 \pm \text{\footnotesize 0.44}$ 
        & $2055.2 \pm \text{\footnotesize 56.3}$ & $1.71 \pm \text{\footnotesize 0.08}$ 
        & $16.0 \pm \text{\footnotesize 2.8}$ & $1.129 \pm \text{\footnotesize 0.034}$ 
        & $46.6 \pm \text{\footnotesize 14.6}$ & $1.639 \pm \text{\footnotesize 0.001}$ 
        \\
        \midrule
        \hspace{2pt} + \teacher (\emph{ours})
        & $\mathbf{676.0} \pm \text{\footnotesize 0.0}$ & $2.13 \pm \text{\footnotesize 0.18}$ 
        & $\mathbf{2452.6} \pm \text{\footnotesize 21.7}$ & $\mathbf{0.94} \pm \text{\footnotesize 0.03}$ 
        & $\mathbf{51.4} \pm \text{\footnotesize 4.0}$ & $\mathbf{1.019} \pm \text{\footnotesize 0.016}$ 
        & $\mathbf{246.6} \pm \text{\footnotesize 14.7}$ & $\mathbf{1.634} \pm \text{\footnotesize 0.001}$ 
        \\
        \bottomrule
    \end{tabular}
    }
    \vspace{-5pt}
\end{table*}

\begin{figure}[t!]
\centering
\captionsetup[subfigure]{labelformat=empty}
\begin{minipage}[t]{0.98\textwidth}
\begin{subfigure}[b]{0.14\textwidth}
    \centering
\includegraphics[width=\textwidth]{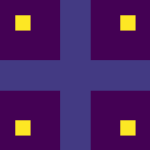}
    \caption{Target distr.}
\end{subfigure}\hfill
\begin{subfigure}[b]{0.14\textwidth}
    \centering
\includegraphics[width=\textwidth]{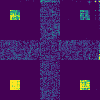}
    \caption{\teacher (\textit{ours})}
\end{subfigure}\hfill
\begin{subfigure}[b]{0.14\textwidth}
    \centering
\includegraphics[width=\textwidth]{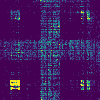}
    \caption{PER}
\end{subfigure}\hfill
\begin{subfigure}[b]{0.14\textwidth}
    \centering
\includegraphics[width=\textwidth]{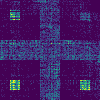}
    \caption{PRT}
\end{subfigure}\hfill
\begin{subfigure}[b]{0.14\textwidth}
    \centering
\includegraphics[width=\textwidth]{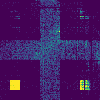}
    \caption{GAFN}
\end{subfigure}\hfill
\begin{subfigure}[b]{0.14\textwidth}
    \centering
\includegraphics[width=\textwidth]{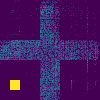}
    \caption{$\epsilon$-expl.}
\end{subfigure}\hfill
\begin{subfigure}[b]{0.14\textwidth}
    \centering
\includegraphics[width=\textwidth]{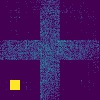}
    \caption{On policy}
\end{subfigure}
\end{minipage}
\vspace{-5pt}
\caption{Empirical distribution plots of $10^5$ test samples from policies on the $(d=2, H=256)$ grid.}
\vspace{-15pt}
\label{fig:grid_baselines}
\end{figure}

\paragraph{Baselines.} We compare our method with on-policy TB and TB with off-policy exploration methods, such as $\epsilon$-exploration \citep{bengio2021flow}, and GAFN~\citep{pan2023generative}, along with a baseline using a replay buffer prioritizing rewards (PRT) or \teacher rewards inspired by Prioritized Experience Replay~\citep[PER;][]{schaul2015prioritized}. PER can be seen as a non-amortized version of our method.

\paragraph{Results.} \cref{tab:grid-results} summarizes the results. TB with \teacher consistently outperforms baselines. The significant margin in the number of modes discovered in the larger-scale setting indicates that the \teacher effectively guides the \student to visit undiscovered modes. Please refer to~\cref{app:grid-results} for comprehensive results.

\paragraph{Effects of local search.} We assess the local search (LS) effect (\cref{sec:ls}) on a $(d=4, H=32)$ grid. The \teacher's local search uses the \teacher's reward $R_{\text{\teacher}}$ for acceptance, compared to two baselines -- on-policy TB and TB with PER -- which use the \student's reward $R$. Furthermore, local search accelerates mode discovery, likely by reducing the sensitivity to nonstationarity in \teacher learning (See \cref{app:teacher-local}).

\begin{table*}[t!]
\vspace*{-1em}
    \caption{\looseness=-1 Evaluation on multimodal continuous sampling tasks. Log-partition function estimation errors (evidence lower bound (ELBO), importance sampled ELBO (ELBO-IS), evidence upper bound (EUBO)) and 2-Wasserstein distances ($W_2^2$) to target samples are reported as mean{\scriptsize$\pm$std} over five runs. We compare MCMC methods (SMC, GGNS), a differentiable simulation method (PIS), and GFlowNets trained using the TB objective with various off-policy strategies, such as loss prioritized replay (PER) and reward prioritized replay (PRT). }\vspace*{-0.75em}
    \label{tab:diffusion_sampling}
    \resizebox{1\linewidth}{!}{
    \begin{tabular}{@{}lcccccccc}
        \toprule
        Energy $\rightarrow$
        & \multicolumn{4}{c}{25GMM ($d=2$, $\log Z = 0$, $W_{2}^{2} = 0.29$)} 
        & \multicolumn{4}{c}{Manywell ($d=32$, $\log Z = 164.6956753$, $W_{2}^{2} = 5.36$)} 
        \\
        \cmidrule(lr){2-5}\cmidrule(lr){6-9}
        Algorithm $\downarrow$ Metric $\rightarrow$ & ELBO ($\uparrow$) &  ELBO-IS ($\uparrow$) & EUBO ($\downarrow$)  & $W_2^2$ ($\downarrow$)  & ELBO ($\uparrow$) & ELBO-IS ($\uparrow$) & EUBO  ($\downarrow$) &$W_2^2$ ($\downarrow$) \\
        \midrule
        SMC
        & \multicolumn{3}{c}{- 0.569{\scriptsize$\pm$0.010}} & 0.86{\scriptsize$\pm$0.10}
        & \multicolumn{3}{c}{ 149.706{\scriptsize$\pm$1.078}} & 8.28{\scriptsize$\pm$0.32}
        \\
        GGNS
        & \multicolumn{3}{c}{- 0.016{\scriptsize$\pm$0.042}} & 1.19{\scriptsize$\pm$0.17}
        & \multicolumn{3}{c}{164.404{\scriptsize$\pm$0.454}} & 6.51{\scriptsize$\pm$0.32}
        \\
        \midrule
        PIS
        & -1.192{\scriptsize$\pm$0.177}
        & -1.192{\scriptsize$\pm$0.176}
        & 26.733{\scriptsize$\pm$5.107}
        & 4.95{\scriptsize$\pm$0.73}
        & 160.516{\scriptsize$\pm$1.025}
        & 162.017{\scriptsize$\pm$0.980}
        & 581.464{\scriptsize$\pm$240.916}
        & 6.15{\scriptsize$\pm$0.01}
        \\
        \midrule
        TB \hfill(on-policy $\rightarrow$)
        & -1.105{\scriptsize$\pm$0.007}
        & -1.008{\scriptsize$\pm$0.009}
        & 18.321{\scriptsize$\pm$0.932}
        & 4.64{\scriptsize$\pm$0.01}
        
        & 161.048{\scriptsize$\pm$0.036}
        & 162.019{\scriptsize$\pm$0.645}
        & 427.850{\scriptsize$\pm$80.082}
        & 6.15{\scriptsize$\pm$0.01}
        \\
        \hspace{2pt} + $\epsilon$-expl.
        & -1.056{\scriptsize$\pm$0.117}
        & -0.956{\scriptsize$\pm$0.118}
        & 15.135{\scriptsize$\pm$1.861}
        & 4.58{\scriptsize$\pm$0.08}
        & 161.064{\scriptsize$\pm$0.036}
        & 162.008{\scriptsize$\pm$0.062}
        & 355.787{\scriptsize$\pm$4.761}
        & 6.15{\scriptsize$\pm$4.761}
        \\
        \hspace{2pt} + PRT
        &  -0.750{\scriptsize$\pm$0.138}
        &  -0.640{\scriptsize$\pm$0.144}
        & 12.103{\scriptsize$\pm$2.273}
        & 4.28{\scriptsize$\pm$0.59}
        & 161.071{\scriptsize$\pm$0.085}
        & 161.998{\scriptsize$\pm$0.111}
        & 379.623{\scriptsize$\pm$77.409}
        & 6.14{\scriptsize$\pm$0.03}
        \\
         \hspace{2pt} + PER
        & -0.282{\scriptsize$\pm$0.158}
        & -0.147{\scriptsize$\pm$0.162}
        & 1.833{\scriptsize$\pm$2.366}
        & 1.87{\scriptsize$\pm$1.23}
        
        & 161.537{\scriptsize$\pm$0.186}
        & 162.582{\scriptsize$\pm$0.268}
        & 210.440{\scriptsize$\pm$6.888}
        & 5.91{\scriptsize$\pm$0.08}
        \\
        \hspace{1pt} + \teacher (\emph{ours})
        & \textbf{-0.137}{\scriptsize$\pm$0.004} 
        & -\textbf{0.005}{\scriptsize$\pm$0.007}
        & \textbf{0.115}{\scriptsize$\pm$0.009}
        & \textbf{0.86}{\scriptsize$\pm$0.07}
        
        & \textbf{163.484}{\scriptsize$\pm$0.049} 
        & \textbf{164.676}{\scriptsize$\pm$0.048}
        & \textbf{165.800}{\scriptsize$\pm$0.045}
        & \textbf{5.46}{\scriptsize$\pm$0.01}
        \\   
        \bottomrule
    \end{tabular}
    }
    \vspace{-5pt}
\end{table*}

\begin{figure}[t!]
\centering
\captionsetup[subfigure]{labelformat=empty}
\begin{minipage}[t]{0.9\textwidth}
\begin{subfigure}[b]{0.16\textwidth}
    \centering
\includegraphics[width=\textwidth]{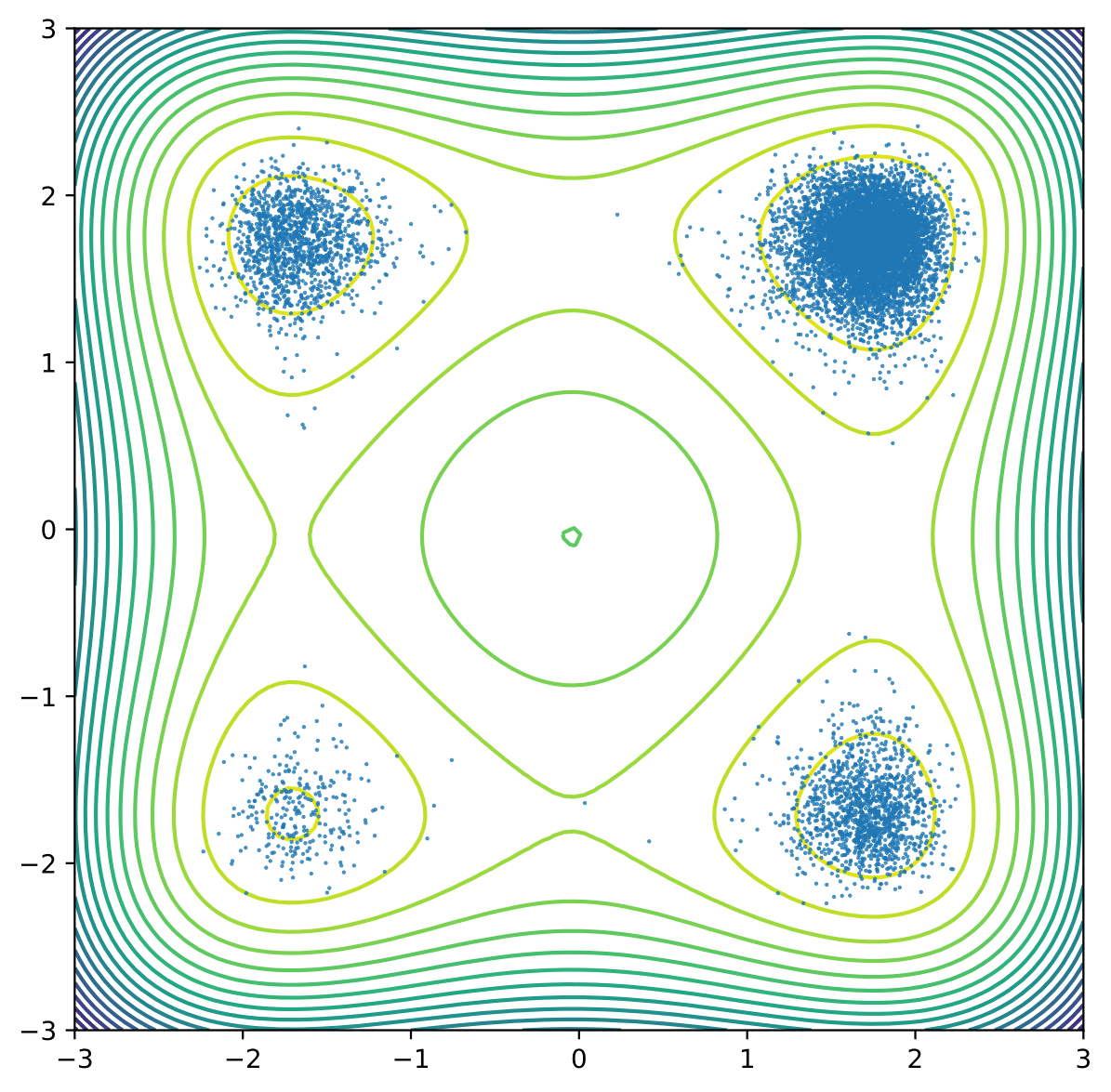}
    \caption{Target distr.}
\end{subfigure}\hfill
\begin{subfigure}[b]{0.16\textwidth}
    \centering
\includegraphics[width=\textwidth]{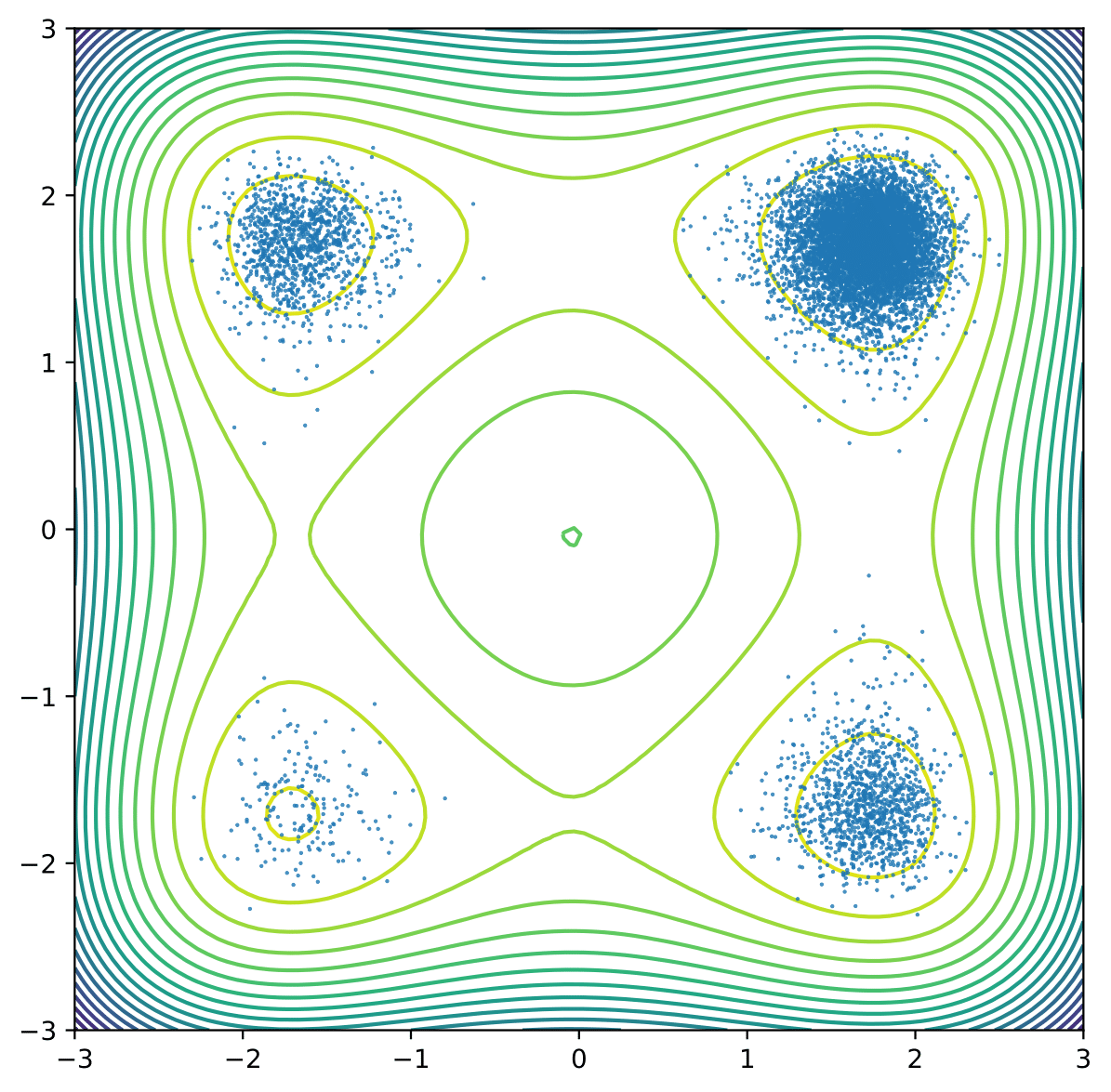}
    \caption{\teacher (\textit{ours})}
\end{subfigure}\hfill
\begin{subfigure}[b]{0.16\textwidth}
    \centering
\includegraphics[width=\textwidth]{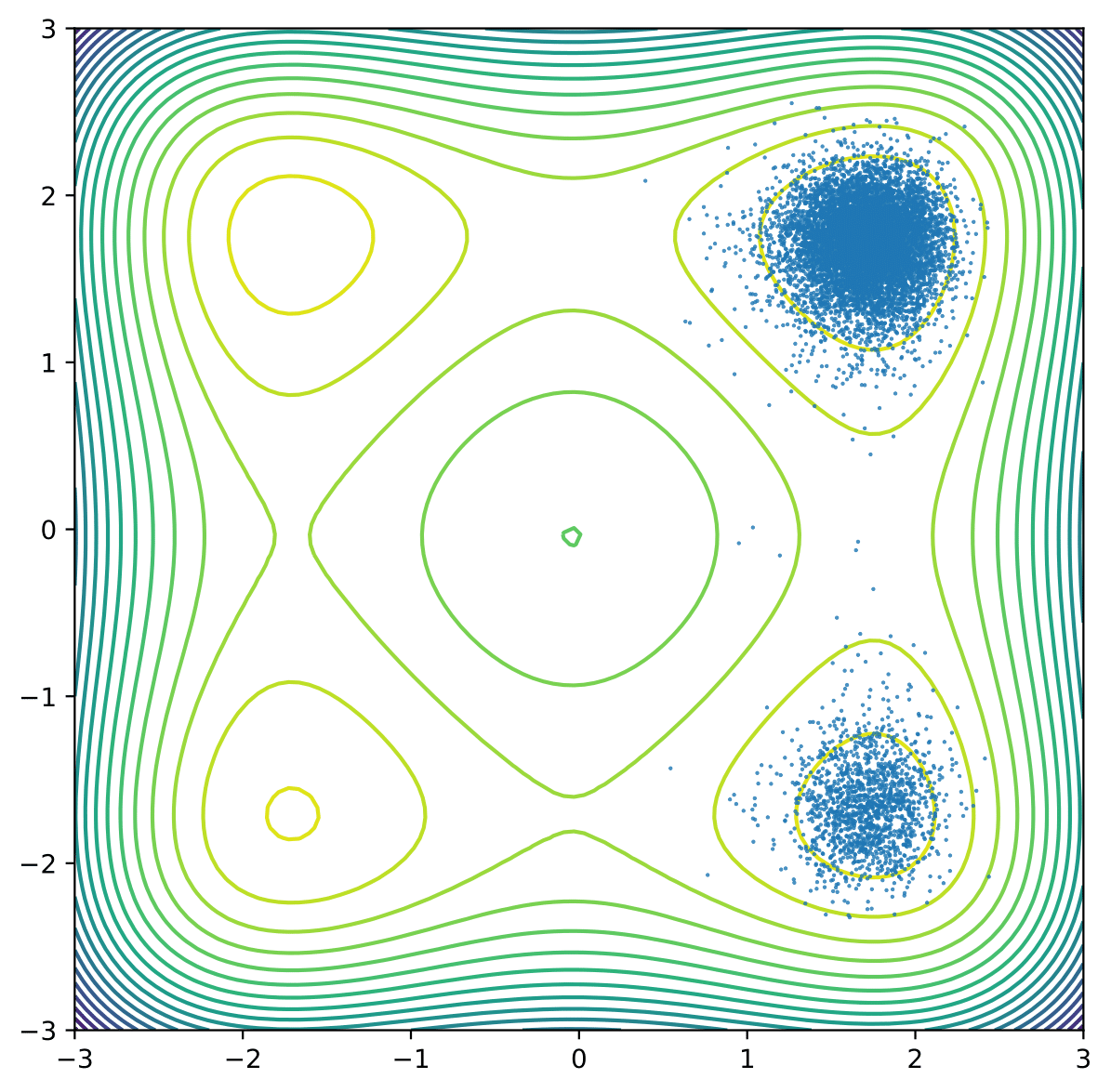}
    \caption{PER}
\end{subfigure}\hfill
\begin{subfigure}[b]{0.16\textwidth}
    \centering
\includegraphics[width=\textwidth]{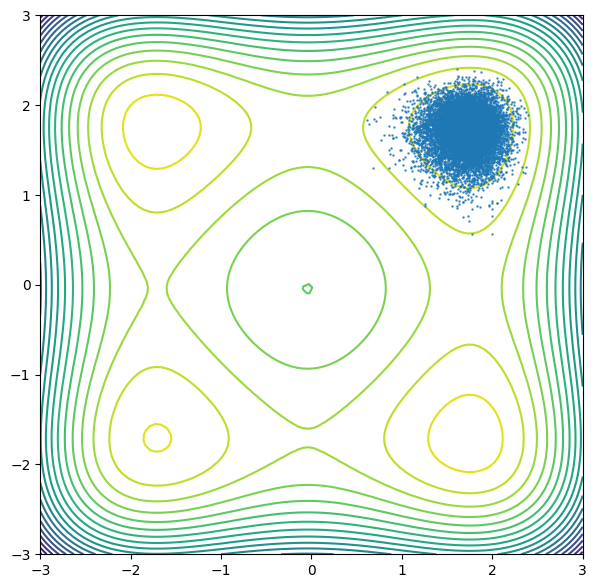}
    \caption{PRT}
\end{subfigure}\hfill
\begin{subfigure}[b]{0.16\textwidth}
    \centering
\includegraphics[width=\textwidth]{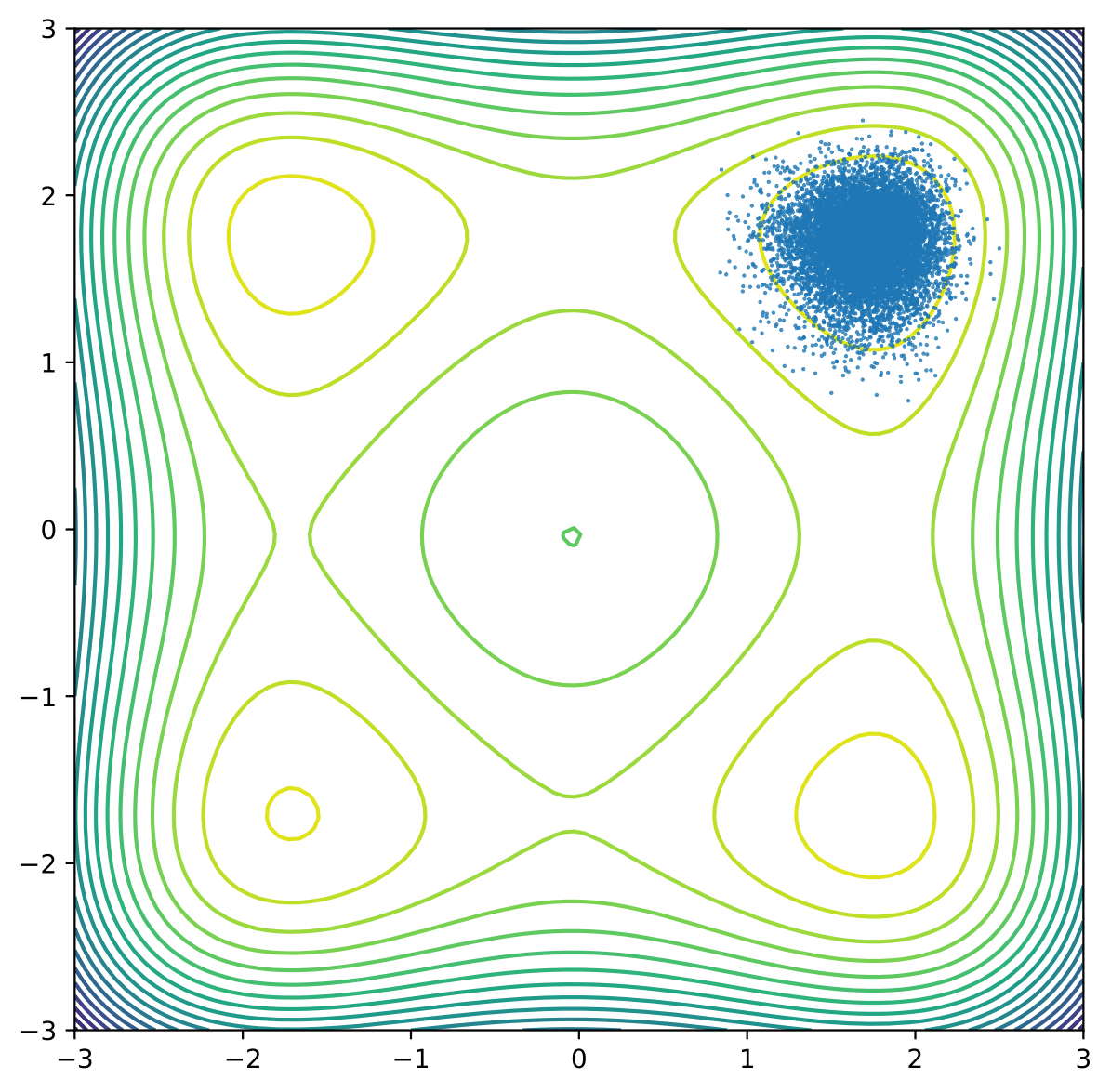}
    \caption{$\epsilon$-expl.}
\end{subfigure}\hfill
\begin{subfigure}[b]{0.16\textwidth}
    \centering
\includegraphics[width=\textwidth]{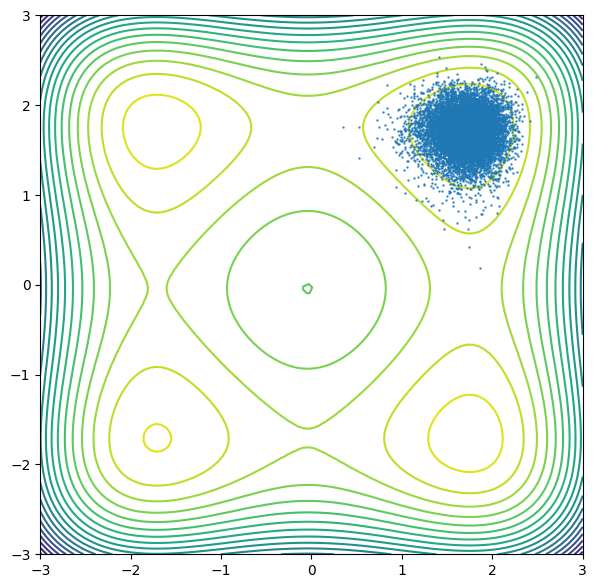}
    \caption{On policy}
\end{subfigure}
\end{minipage}
\vspace{-5pt}
\caption{Samples from trained models on the Manywell task (projected onto the first two dimensions).}
\label{fig:Manywell_plots}
\end{figure}

\begin{figure}[t!]
\vspace{-10pt}
    \captionsetup[subfigure]{labelformat=empty}
    \centering
    \begin{minipage}{0.49\textwidth}
        \centering
        \begin{subfigure}[b]{0.32\textwidth}
            \centering
            \includegraphics[width=\textwidth]{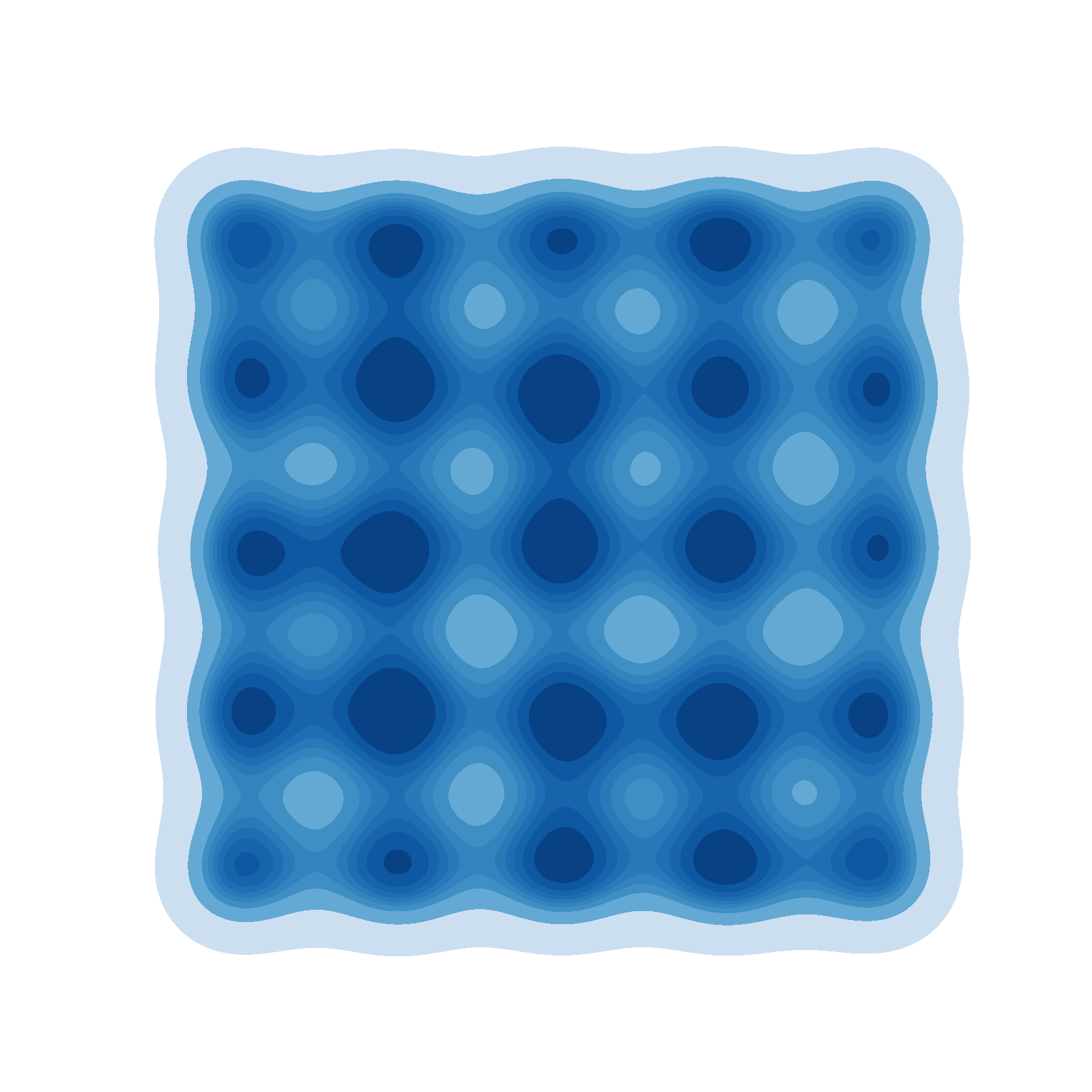}
            \caption{Goal distr.}
        \end{subfigure}\hfill
        \begin{subfigure}[b]{0.32\textwidth}
            \centering
            \includegraphics[width=\textwidth]{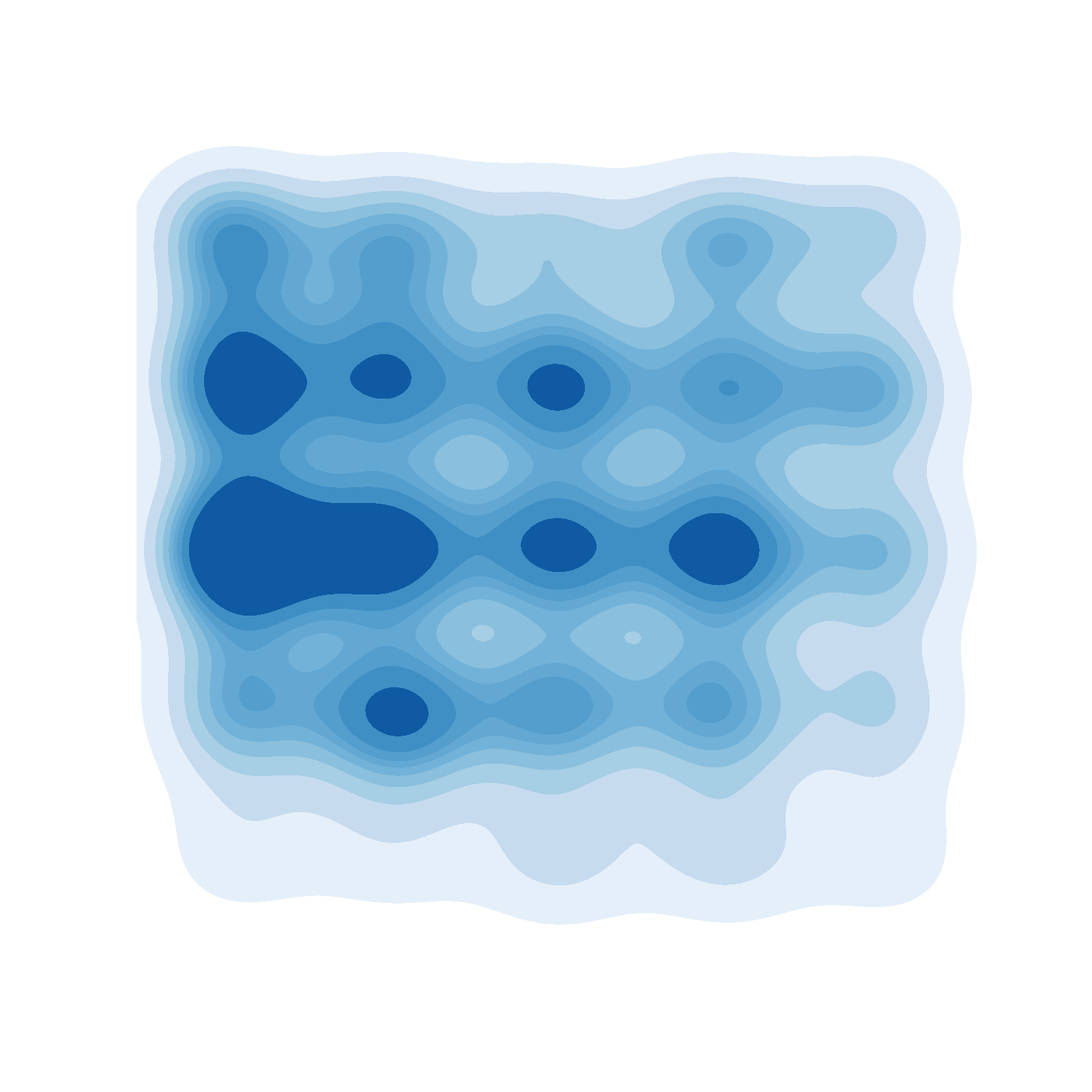}
            \caption{\student (1/5)}
        \end{subfigure}\hfill
        \begin{subfigure}[b]{0.32\textwidth}
            \centering
            \includegraphics[width=\textwidth]{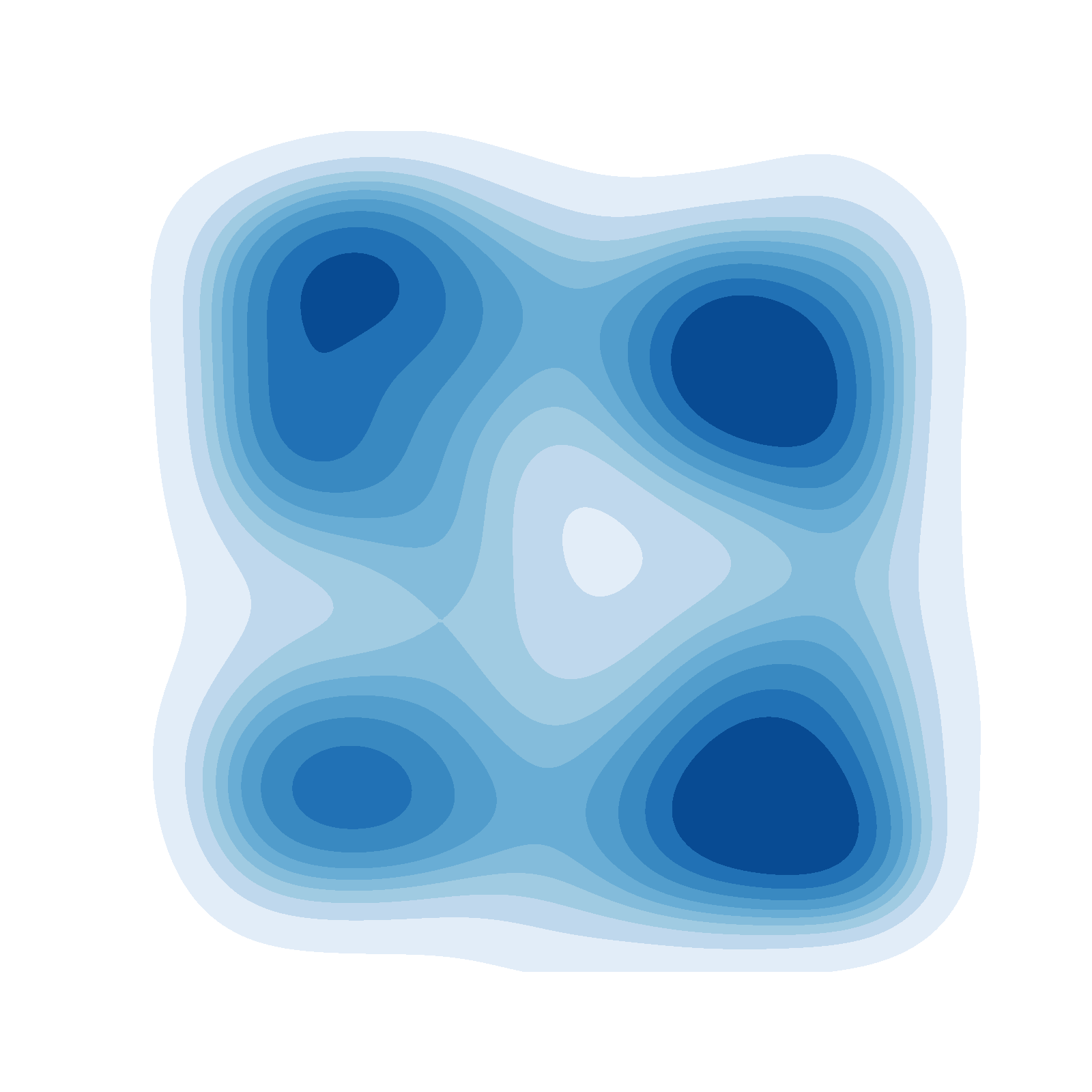}
            \caption{\teacher (1/5)}
        \end{subfigure}
    \end{minipage}
    \hfill
    \begin{minipage}{0.49\textwidth}
        \centering
        \begin{subfigure}[b]{0.32\textwidth}
            \centering
            \includegraphics[width=\textwidth]{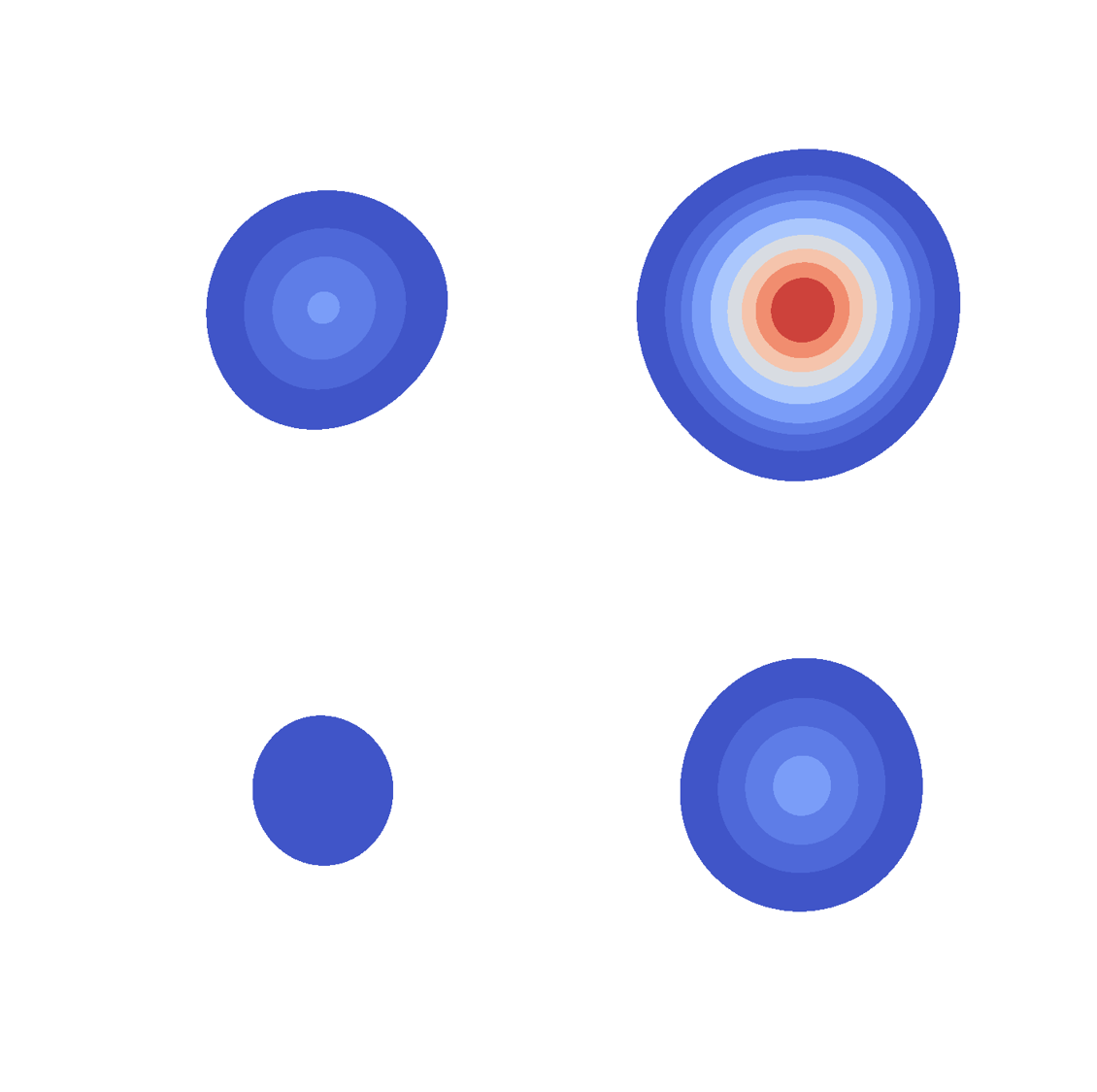}
            \caption{Goal distr.}
        \end{subfigure}\hfill
        \begin{subfigure}[b]{0.32\textwidth}
            \centering
            \includegraphics[width=\textwidth]{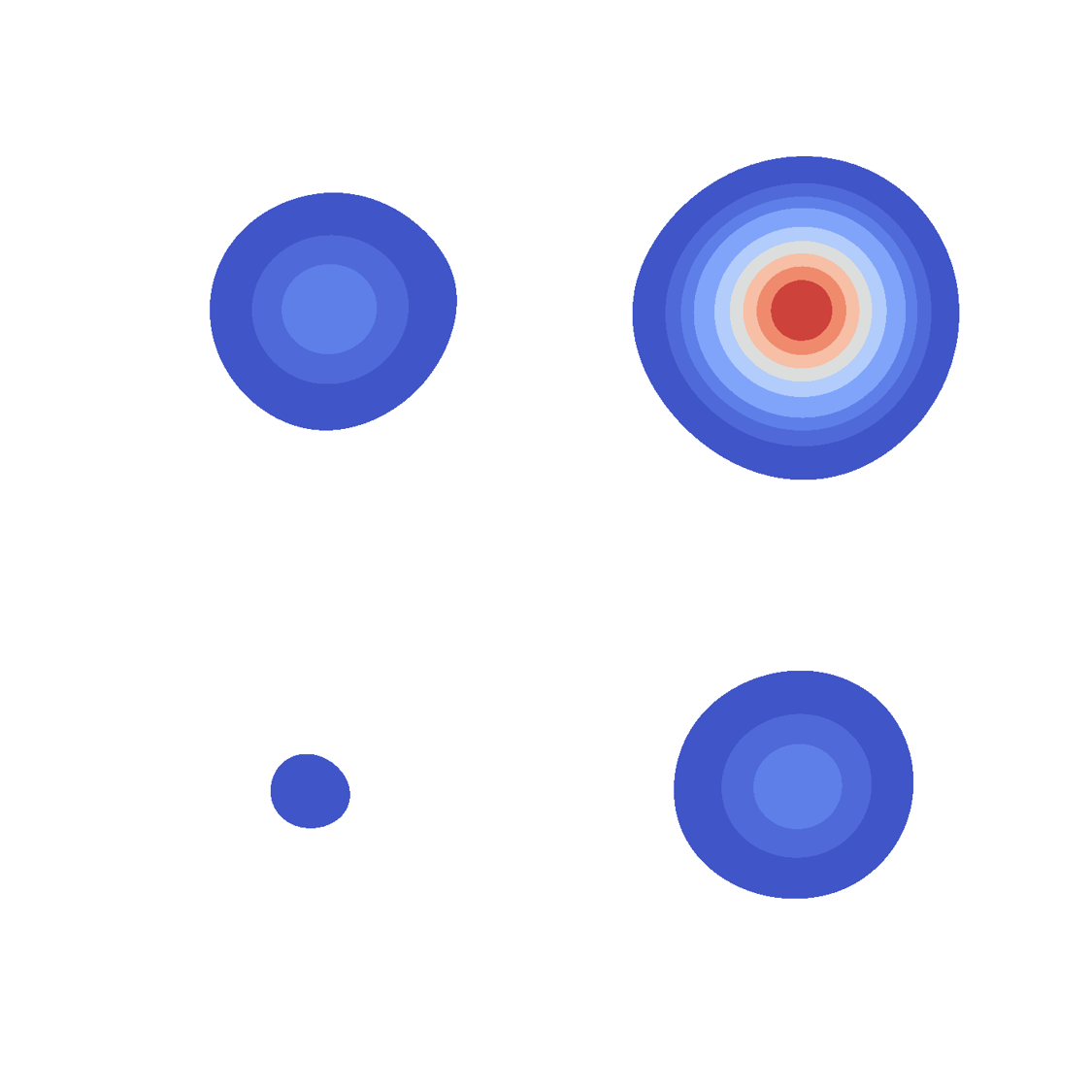}
            \caption{\student (2/5)}
        \end{subfigure}\hfill
        \begin{subfigure}[b]{0.32\textwidth}
            \centering
            \includegraphics[width=\textwidth]{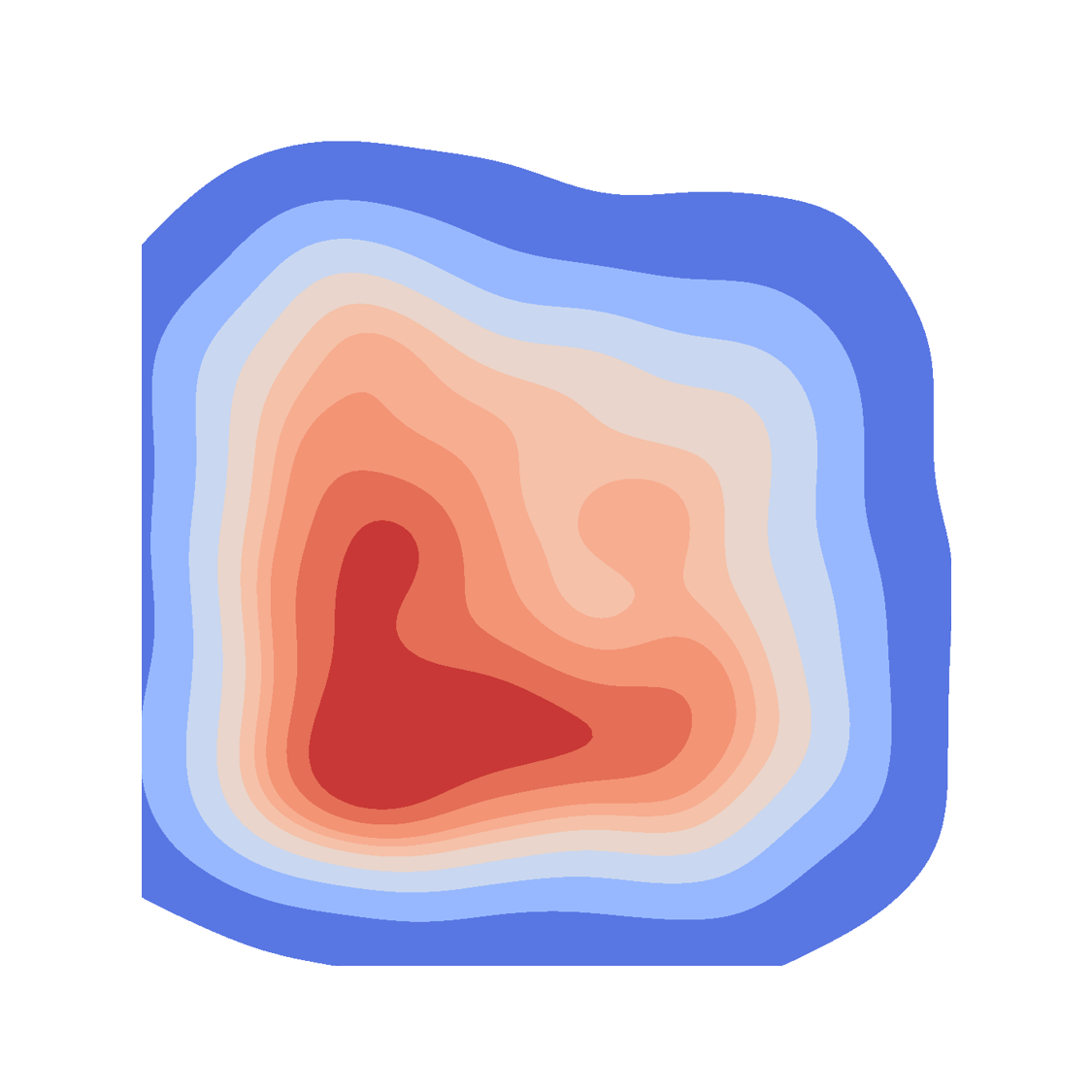}
            \caption{\teacher (2/5)}
        \end{subfigure}
    \end{minipage}
    \vspace{-5pt}
    \caption{KDE plots for 25GMM (left three) and Manywell (right three) at intermediate states of training. The \student (\textit{ratio}) indicates the fraction of total training steps completed. The \teacher adaptively adjusts the training distribution in response to the modes that the \student is missing.}
    \label{fig:diffusion-teacher-student}
    \vspace{-10pt}
\end{figure}

\subsection{Diffusion sampling}
\label{sec:diffusion}

\paragraph{Setting.} The task of learning a \textit{diffusion sampler} is to invert a diffusion process in order to sample the target density function.
These experiments largely follow the setup of \citet{sendera2024diffusion}.
Here, we aim to model a distribution over trajectories $\tau = (\boldsymbol{0}=x_0 \rightarrow x_{\Delta t} \rightarrow x_{2 \Delta t} \rightarrow \ldots \rightarrow x_1)$ (with $\Delta t=\frac1T$; for us, $T=100$), so that $x_1$ is distributed according to an unnormalized density function $p(x_1) \propto R(x_1) = e^{-\mathcal{E}(x_1)}$. 
The transitions are parametrized as to the Euler-Maruyama discretization of a neural stochastic differential equation \citep{tzen2019neural}: the sampler begins at the initial state $(\boldsymbol{0}, t=0)$, and the transition from $x_t$ to $x_{t+\Delta t}$ is sampled from an appropriately scaled Gaussian with mean given by a trained model taking $x_t$ and $t$ as input. The detailed setting and policy parametrization are described in \cref{app:diffusion_sampler}. 

The main challenge in this task to capture the multimodality of $R(x_1)$ without having access to samples from target distribution during training, which is difficult when there are many well-separated modes. It is not possible to apply the forward KL (\ie, log-likelihood variational bound) objectives typically used for diffusion models \citep{song2021maximum}, since target distribution datapoints are not available; thus \textit{exploration} is crucial to mode discovery.

In this work, we benchmark diffusion samplers on two established tasks in the diffusion sampling literature: a 2-dimensional Gaussian mixture with 25 modes (25GMM) and a 32-dimensional Manywell distribution. When performing off-policy exploration, we assume a black-box property for the energy function $\mathcal{E}$, where $\nabla \mathcal{E}$ is not accessible, similar to settings in reinforcement learning. This assumption is meant to mimic a common setting in scientific discovery, where we may have black-box energies requiring expensive simulations to compute; we note that for these tasks, effective methods that use the energy gradient exist (see \citet{sendera2024diffusion}).

\paragraph{Baselines.} We consider two representative MCMC baselines: Sequential Monte Carlo (SMC) and the state-of-the-art GGNS \citep{lemos2023ggns}. We also include the simulation-based continuous stochastic control method Path Integral Sampler \citep[PIS;][]{zhang2021path}. The major baselines are off-policy training methods based on trajectory balance (TB). We include on-policy TB in continuous space and the $\epsilon$-exploration method introduced by \citet{malkin2023gflownets,lahlou2023theory}, which adds additional Gaussian noise at least policy sampling step during training, as main baselines. Additionally, we compare with two replay buffer methods combined with TB: one that prioritizes reward (PRT) as studied by \citet{sendera2024diffusion}, and another that prioritizes loss (PER) \citep{schaul2015prioritized}. The gradient-based local search introduced by \citet{sendera2024diffusion}, while effective, is excluded as it requires access to $\nabla \mathcal{E}(x)$ for the search; but we also study potential integration of the \teacher with these techniques in \cref{app:teacher-local}.

\paragraph{Results.} As shown in \cref{tab:diffusion_sampling}, on-policy TB and PIS yield similar performance, consistent with the fact that they have identical expected gradients \citep{malkin2023gflownets,lahlou2023theory}. This suggests that TB could benefit from additional off-policy methods for improvement. Indeed, the $\epsilon$-exploration techniques improve slightly over on-policy methods. PRT provides larger benefits on 25GMM but shows no meaningful benefits on Manywell. PER offers significant improvements over on-policy methods compared to others. Our method, \teacher, achieves the highest results across all metrics, including the Evidence Lower Bound (ELBO), Importance-Sampled ELBO (ELBO-IS), Evidence Upper Bound (EUBO) (see \cref{app:setting} for detailed definitions of those metrics). Especially in EUBO, a metric suggested by \citet{blessing2024beyond} to measure mode coverage, baseline methods face significant challenges. This indicates that existing methods struggle to perform proper \textit{exploration} across modes, as confirmed by the sample plots in \Cref{fig:Manywell_plots}. Our method offers clear advantages on EUBO metrics.

In \cref{fig:diffusion-teacher-student}, we depict the training dynamics of the \teacher and \student by plotting kernel density estimates of their samples midway through training. This figure illustrates the mechanism by which the \teacher promotes mode discovery by the \student. As the \student struggles to find some modes, especially those with lower rewards than others, the \teacher puts high probability on them, encouraging the \student to reduce its loss in those regions of the space.

\begin{figure}[t]
\vspace*{-1.0em}
    \centering
\begin{minipage}[t]{0.85\textwidth}
\includegraphics[width=1.0\linewidth]{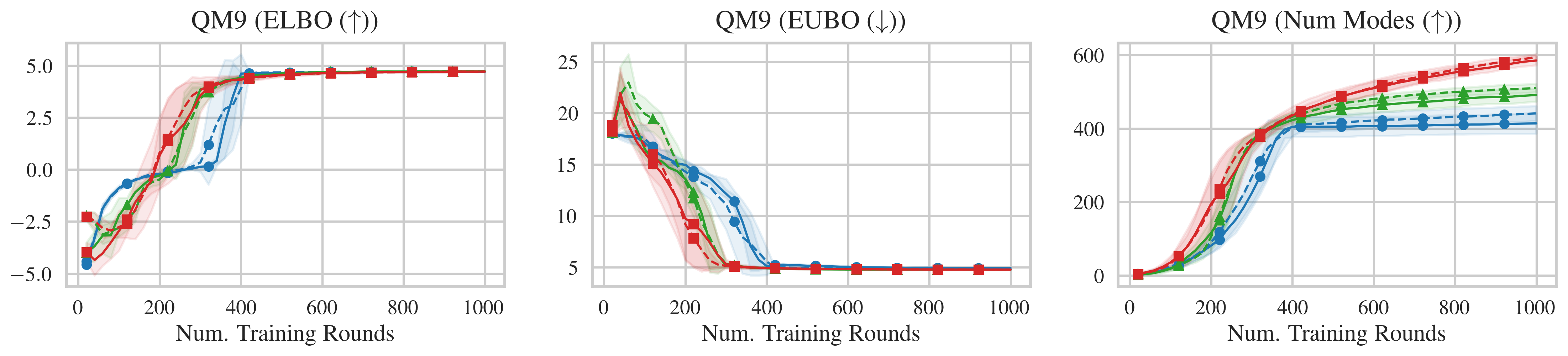}
\includegraphics[width=1.0\linewidth]{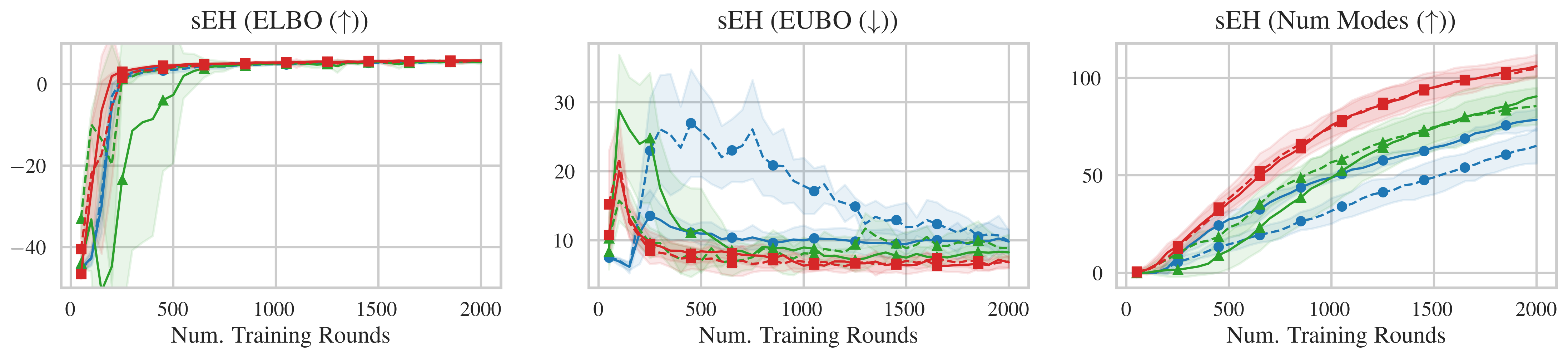}
\includegraphics[width=1.0\linewidth]{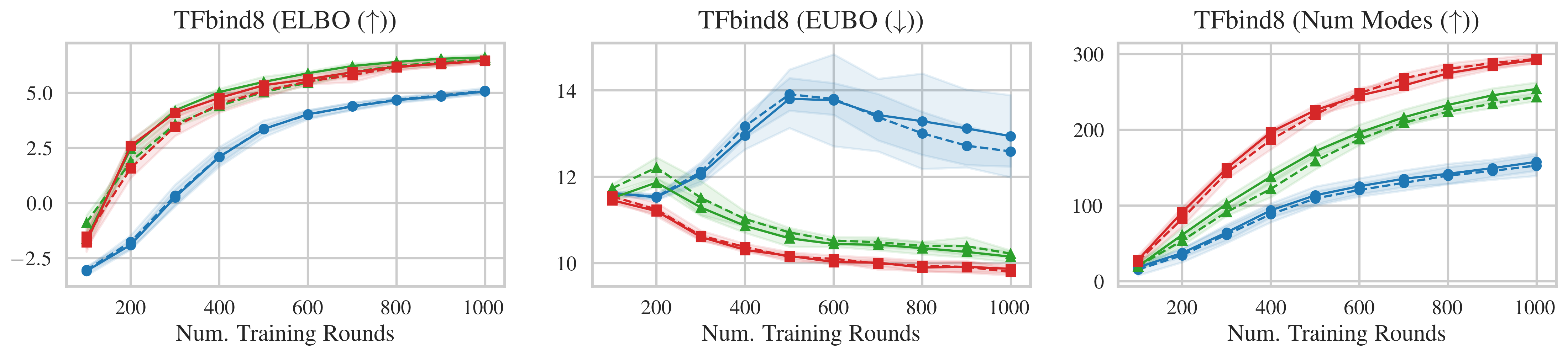}
\includegraphics[width=1.0\linewidth]{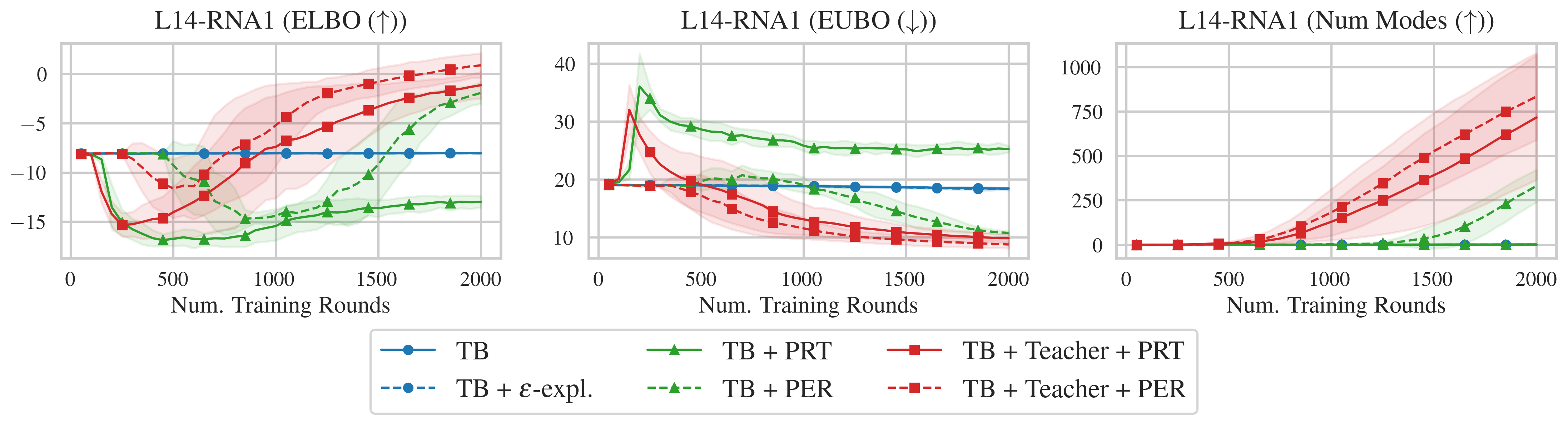}
\vspace*{-1em}
    \caption{Training graphs for molecule design (QM9, sEH) and biological sequence design (TFbind8, L14-RNA1) tasks. Mean and standard deviation over five  runs are shown.}
    \label{fig:tfbind8_l14_rna1_metrics_over_rounds}
\end{minipage}
\vspace*{-2em}
\end{figure}

\subsection{Biological and chemical discovery}
\label{sec:bioseq}

\paragraph{Setting.} GFlowNets have been used to generate biological and chemical structures by sequentially adding predefined substructures. In molecules, the actions add atoms or fragments; in biological sequences, nucleotides or amino acids. We aim to match a target probability distribution over structures given by some proxy reward model. Following \citet{shen2023towards}, we benchmark the number of discovered modes, as well as probabilistic metrics like ELBO and EUBO. We study four biological and chemical discovery problems, following \citet{shen2023towards,kim2024local}:

\begin{itemize}[left=0pt,nosep]
\item
\textbf{QM9.} The objects being sampled are small molecular graphs. Molecules are generated using 12 building blocks with 2 stems, and each molecule contains 5 blocks. The reward function is a HOMO-LUMO gap on the target transcription factor, which is obtained via a pre-trained MXMNet proxy from \citet{zhang2020molecular}. We use a reward exponent of 5. We define modes as the top $0.5\%$ quantile of $R(x)$.
\item
\textbf{sEH.} The generated objects are molecular graphs. Molecules are built using 18 blocks with 2 stems and 6 blocks per molecule. The reward function is a binding affinity to soluble epoxide hydrolase (sEH), which is provided by the pre-trained proxy model from \citet{bengio2021flow}. We use a reward exponent of 6. We define modes as the top $0.01\%$ quantile of $R(x)$, with filtering to exclude candidates too similar based on Tanimoto similarity, following~\citet{kim2023learning}.
\item
\textbf{TFbind8.} The generated objects are DNA sequences with 8 nucleotides. The reward function is a binding affinity to a human transcription factor \citep{barrera2016survey}, which is obtained via a pre-trained proxy model provided by \citet{trabucco2022design}. We use a reward exponent of 3. We use a pre-defined set of modes provided by \citet{shen2023towards}.
\item
\textbf{L14-RNA1.} The generated objects are RNA sequences of length 14. The reward function is a binding affinity to a human transcription factor, which is obtained via a pre-trained proxy model from \citet{sinai2020adalead}. We use a reward exponent of 8. We define modes as the top $0.01\%$ quantile of $R(x)$, with diversity filtering whose threshold is 1 unit of Levenstein distance, also following~\citet{kim2023learning}.
\end{itemize}

Detailed training and hyperparameter settings for each task can be found in \Cref{app:biochemical}.

\textbf{Baselines.} We compare our method with on-policy TB, $\epsilon$-exploration, and the PRT replay buffer method designed for biochemical tasks \citep{shen2023towards}. Additionally, we evaluate it against a loss-prioritized replay buffer (PER) \citep{schaul2015prioritized}. Since the local search method \citep{kim2024local} targets reward exploitation with a different purpose to ours, a separate analysis of its complementarity to our method is presented in \cref{app:teacher-local}.

\textbf{Results.} As shown in \cref{fig:tfbind8_l14_rna1_metrics_over_rounds}, the \teacher method improves mode discovery when combined with both PER and PRT buffers, outperforming on-policy methods on every task. Similar trends are observed across other metrics, with faster convergence under the \teacher method. For TFbind8, our method's dominance is particularly evident for mode coverage metrics like EUBO and the number of modes, although ELBO remains comparable to using PER or PRT. In the largest task, L14-RNA1, the \teacher method surpasses PER and PRT in ELBO, EUBO, and the number of modes.

We draw special attention to the comparison between PER and PRT. While the differences are minimal in the first three tasks (QM9, TFbind8, sEH), in the larger-scale tasks, loss-prioritizing with PER shows a clear advantage. This suggests that loss information is more important for exploration in large-scale tasks, where the \teacher method with PER achieves the best overall performance.

\section{Discussion}
\looseness=-1
We have introduced a \textbf{\teacher} that adaptively generates states for training a \textbf{\student} amortized sampler. The \teacher favors states for which the \student has high loss and thus promotes the discovery of new modes.

This approach paves the way for numerous future research directions. For example, an adaptive \teacher could be applied to amortize intractable inference in large language models (LLMs) \citep{hu2023amortizing} and diffusion models \citep{venkatraman2024amortizing}, or to enhance exploration in automatic red-teaming for LLMs \citep{lee2024learning}. Applying our method to probabilistic models such as amortized Bayesian causal discovery \citep{deleu2022bayesian, deleu2023jsp, nishikawa2022bayesian} and amortized inference in graphical models \citep{falet2024deltaai} are also promising directions. Methodologically, the concept of using amortized prioritized experience replay (PER) as a \teacher can be extended to train general agent-based systems, not only amortized samplers.

\section*{Acknowledgement}

We thank to Anirudh Goyal, David Dobre, Gauthier Gidel for helpful discussion for this project. This research is based on funding from Samsung, Intel, CIFAR and the CIFAR AI Chair program. The research was enabled in part by computational resources provided by the Digital Research Alliance of
Canada (\url{https://alliancecan.ca}), Mila (\url{https://mila.quebec}), and NVIDIA. This work was also partially supported by the National Research Foundation of Korea (NRF) grant funded by the
Korea government (MSIT) (No. RS-2024-00410082).

\bibliography{iclr2025_conference}
\bibliographystyle{iclr2025_conference}

\appendix

\appendix

\clearpage

\section{Limitations}

The primary limitation of the adaptive \teacher is the added complexity in training, as the \teacher's policy network must be trained in addition to the \student's. This introduces a trade-off between increased training complexity and enhanced mode-seeking capabilities. While we believe that the ability to discover multiple modes outweighs the additional complexity, it is important to apply this technique judiciously. For tasks where the reward is unimodal or that do not require extensive exploration, using a \teacher may not be necessary. However, in scenarios where the model tends to collapse to specific modes of the multimodal target distribution, employing a \teacher is a beneficial choice.

Additionally, since the \student struggles to cover entire modes on its own, in significantly larger search spaces the \teacher may also have difficulty covering the full range of modes that the \student fails to discover, potentially collapsing into specific modes. In such large-scale settings, we expect that a multi-agent \teacher system—with multiple agents collaboratively covering the space—could be a beneficial direction for future work to mitigate this limitation.

\section{Theoretical analysis of stationary distributions}
\label{app:theory}

\begin{proposition}
\label{thm:joint-optimum}
Let the behavior policy \( P_{\beta}(\tau) \) be a distribution over trajectories \( \tau \in \mathcal{T} \) that satisfies full support.

If the parameters \( \theta^* \) and \( \phi^* \) of the \student and \teacher policies, respectively, jointly optimize the objective functions to 0 in expectation over \( P_{\beta}(\tau) \), then:
\begin{enumerate}[label=(\alph*),left=0pt,nosep]
    \item The marginal distribution of the \student policy over terminal states satisfies
    \[
    P_F^\top(x; \theta^*) \propto R(x),
    \]
    \item The marginal distribution of the \teacher policy over terminal states satisfies
    \[
    P_F^\top(x; \phi^*) \propto  R(x)^{\alpha},
    \]
\end{enumerate}
where:
\begin{itemize}[left=0pt,nosep]
    \item \( R(x) \) is the reward function,
    \item \( \epsilon > 0 \) is an offset constant introduced in \eqref{eq:teacher-reward1},
    \item \( \alpha > 0 \) is the mixing constant for the \teacher introduced in \eqref{eq:mixing}.
\end{itemize}
\end{proposition}
\begin{proof}
According to the trajectory balance training theorem \citep{malkin2022trajectory}, the \student policy achieves optimal loss (\ie, $\mathcal{L}_{\text{TB}}(\tau; \theta^*)=0$ for all $\tau$) if and only if its marginal distribution over terminal states $p(x,\theta^*)\propto R(x)$.

Suppose now that the \student policy has reached this optimum. Then the reward for the \teacher is
\begin{align}
\log \tilde{R}_{\text{\teacher}}(x; \theta^*) &= \mathbb{E}_{P_B(\tau \mid x; \theta^*)} \left[ \log \left( \epsilon + \left(1 + C \cdot \mathbb{I}_{\delta(\tau; \theta^*) > 0} \right) \delta(\tau; \theta^*)^2 \right) \right] + \alpha \log R(x) \nonumber \\
&= \mathbb{E}_{P_B(\tau \mid x; \theta^*)} \left[ \log(\epsilon) \right] + \alpha \log R(x) \nonumber \\
&= \log(\epsilon) + \alpha \log R(x) \nonumber \\
&= \log(\epsilon R(x)^{\alpha}).\label{eq:stationary-teacher}
\end{align}
Again by the TB training theorem, the loss is minimized if and only if the marginal distribution of the \student policy over terminal states is proportional to $\epsilon R(\cdot)^\alpha$, equivalently, to $R(\cdot)^\alpha$.

\end{proof}

\clearpage

\section{Detailed implementation} \label{app:implementation}

For all tasks, we set $\alpha = 0.5$ except for the exploration-intensive deceptive grid world tasks, where we use $\alpha = 0.0$. $C$ is set to 19 for all tasks. Our hyperparameter analysis is provided in \cref{app:alpha}. In a diffusion sampling task, to concentrate the teacher on high-reward regions within a vast continuous space, we limit the teacher's training set to samples where \( x > r_{\text{threshold}} \). Here, \( r_{\text{threshold}} \) is the 90th percentile reward from the untrained student's samples.

For neural networks, we use identical architectures for both the \student and \teacher models. Specifically, for the GFN architecture design, we match the architectures used by each baseline for every task. Detailed descriptions are provided in \cref{app:setting}.

Regarding the replay buffer, we adhere to existing implementations for each task and follow prioritized replay rules. In grid world and diffusion sampling tasks, we use rank-based priority as introduced by \citet{tripp2020sample} and \citet{sendera2024diffusion}. For biochemical tasks, we employ portion-wise priority as proposed by \citet{shen2023towards}. Additionally, we implemented two variants of a loss-prioritized buffer: one prioritizes based on TB loss, and the other uses $R_{\text{\teacher}}$ as the priority, serving a similar purpose. We analyze the performance differences between these variants and refer to the use of $R_{\text{\teacher}}$ as ``PER."

When selecting the behavior policy during training, we periodically choose among the \student, \teacher, and buffer in specific proportions: a ratio of 1:1:0 for Grid World tasks, 3:1:2 for Diffusion Sampler tasks, and 2:1:3 for Biochemical tasks. In the diffusion sampler, we use PER for the replay buffer. For the biochemical task, we benchmark the teacher using both PER (referred to as PER + Teacher) and a reward-prioritized buffer (referred to as PRT + Teacher).

A higher \student ratio encourages exploitation, a higher \teacher ratio promotes exploration, and a larger buffer enhances sample efficiency until reaching the \textit{replay ratio barrier} \citep{doro2023sample}.

In deceptive grid world tasks, exploration is crucial, so we assign a higher \teacher ratio. For diffusion sampling tasks, using a prioritized replay buffer with \teacher rewards yields strong performance. Blending this approach slightly with the \teacher enhances both performance and convergence speed, achieving mode coverage for both EUBO and ELBO. In biochemical tasks, sample efficiency is critical. Therefore, existing tasks are set to be on-policy with a buffer ratio of 1:1. We adjust the proportions to favor on-policy learning by setting the \student-to-teacher ratio to 2:1, as biochemical tasks require some level of exploitation, as reported by \citet{shen2023towards}.

Other implementation details and hyperparameters are maintained as in each task's experimental setting.

\section{Detailed experimental setting} \label{app:setting}

\begin{wraptable}{r}{0.4\textwidth}
    \centering
    \vspace{-10pt}
    \caption{The total number of terminal states ($\vert \gX \vert$) and modes ($\vert\mathcal{M}\vert$) for each grid setting, where $d$ is the dimension and $H$ is the horizon of the hypergrid.}
    \begin{tabular}{lcc}
    \toprule
     & $\vert \gX \vert$ & $\vert \mathcal{M} \vert$ \\
    \midrule
    $d=2, H=128$ & 16383 & 676 \\
    $d=2, H=256$ & 65535 & 2601 \\
    $d=4, H=16$ & 64125 & 81 \\
    $d=4, H=32$ & 1042685 & 1296 \\
    \bottomrule
    \end{tabular}
    \label{tab:grid_space}
    \vspace{0pt}
\end{wraptable}

\subsection{Deceptive Grid World}
\label{app:gridworld}

\paragraph{Evaluation metrics.} Following~\cite{bengio2021flow}, we use the number of modes discovered and the empirical $L_1$ distance from target distribution as evaluation metrics. We define the \textit{modes} as $x$'s with $R(x)=R_2+R_0$. The exact size of the search space $\vert \gX \vert$ is $(H-1)^d + d(H-1)$, with $d$ and $H$ representing the dimension and horizon of the hypergrid, respectively. $\vert \gX \vert$ and the total number of modes $\vert\mathcal{M}\vert$ for each grid setting are reported in~\cref{tab:grid_space}. The $L_1$ distance is calculated by $\frac{1}{\vert \gX \vert} \sum_{x \in \gX} \left[ \vert p_\theta(x) - R(x)/Z \vert \right]$, where $Z=\sum_{x} R(x)$, which is known in this synthetic task. Unlike previous works~\citep{bengio2021flow, malkin2022trajectory} where a portion of the final training samples was used to approximate the expectation, we generate $10^5$ new samples from policy to calculate $p_{\theta}(x)$ for evaluation. We use one sample for only one gradient step by default, \ie, update-to-data (UTD) ratio~\citep{chen2021randomized} is 1, except for the case using replay buffer, where the ratio increases to 2.

\paragraph{Hyperparameters.} Following~\citet{bengio2021flow, malkin2022trajectory}, we use a two-layer MLP with 256 hidden units for the parameterized policy $P_F(\cdot;\theta)$ along with a learnable parameter for $\log Z_\theta$. We train them using the Adam optimizer with a learning rate of $10^{-3}$ for policy and $10^{-1}$ for $\log Z_\theta$. The backward policy $P_B$ is fixed as a uniform random policy. When using a replay buffer, its size is dynamically set to $\floor{0.1 \vert \gX \vert}$. The total reward call budget is capped at 96,000, except for $(d=4, H=32)$, which is increased to 384,000, considering the significantly larger search space. We use a batch size of 16. The total number of gradient steps equals the number of reward calls divided by the batch size, and this number doubles when the replay buffer is used.

\paragraph{Baseline implementations.} In our experiments, we set $\epsilon$ for $\epsilon$-exploration to 0.01, as this value yielded the lowest $L_1$ error among the tested options $\{0.001, 0.003, 0.01, 0.03, 0.1\}$ in the setting where $d=2$ and $H=128$. Similarly, for GAFN~\citep{pan2023generative}, we set the intrinsic reward scale to 0.01, which also resulted in the lowest $L_1$ error among the values $\{1.0, 0.1, 0.01, 0.001\}$ under the same conditions. For PRT and PER, we use the same buffer size and prioritization scheme as above. Note that, unlike the original PER for value-based RL, the buffer we used contains only a terminal $x$ rather than all state transitions. From a sampled $x$, the trajectory can be constructed by backward generation using $P_B$. We also explore the integration of a transition-based replay buffer with the detailed balance objective in~\cref{app:db}.

\subsection{Diffusion sampling} \label{app:diffusion_sampler}

In our diffusion sampler, we primarily adhere to the settings of \citet{sendera2024diffusion}, which in turn build upon those of \citet{zhang2021path}.

\paragraph{Tasks.} We benchmark the following two tasks:

\textit{Gaussian Mixture Model with 25 Modes (25GMM).} The 25GMM consists of a two-dimensional Gaussian mixture with 25 modes, each having a variance of $0.3$. The mode centers are positioned on the grid $\{-10, -5, 0, 5, 10\} \times \{-10, -5, 0, 5, 10\}$.

\textit{Manywell \citep{noe2019boltzmann}.} The Manywell is a 32-dimensional distribution formed as the product of 16 identical two-dimensional double-well distributions. Each component is defined by the potential function
\begin{equation}
    \mu(x_1, x_2) = \exp\left(-x_1^4 + 6x_1^2 + 0.5x_1 - 0.5x_2^2\right).
\end{equation}

\paragraph{Evaluation metrics.} We measure evidence lower bound (ELBO), importance sampled ELBO (ELBO-IS), and evidence upper bound (EUBO).

For estimating ELBO, we draw $M$ samples from current policy and take average value of estimated $\log Z$, which is  $\log R(\cdot) + \log P_B(\cdot) - \log P_F(\cdot)$ as follows:
\begin{equation}
    \text{ELBO} \approx \frac{1}{M} \sum_{i=1}^M  (\log R(x^{i}_1) + \log P_B(\tau^{i}|x^{i}_1) - \log P_F(\tau^{i};\theta)), \quad \tau^i \sim P_F(\tau;\theta), \tau^i\rightsquigarrow x_1^i,
\end{equation}
where $\tau^i\rightsquigarrow x_1^i$ means that $x_1^i$ is the final state of $\tau^i$.

Calculation of ELBO-IS is similar to ELBO: 
\begin{equation}
    \text{ELBO-IS} \approx \log \frac{1}{M} \sum_{i=1}^M \exp (\log R(x^{i}_1) + \log P_B(\tau^{i}|x^{i}_1) - \log P_F(\tau^{i};\theta)), \quad \tau^i \sim P_F(\tau;\theta), \tau^i\rightsquigarrow x_1^i.
\end{equation}
EUBO was introduced as a metric that measures mode coverage by \citet{blessing2024beyond}. To calculate EUBO, we sample $M$ samples from the target distribution and take the average of their variational log-likelihood bounds as follows:
\begin{equation}
    \text{EUBO} \approx \frac{1}{M} \sum_{i=1}^M  (\log R(x^{i}_1) + \log P_B(\tau^{i}|x^{i}_1) - \log P_F(\tau^{i};\theta)) \quad x_1^i \sim P^{*}(x_1), \tau^i\sim P_B(\tau\mid x_1^i).
\end{equation}

\paragraph{Forward and backward transition modeling.} The diffusion sampler models discretized SDE trajectories $\tau = (x_0 \rightarrow x_{\Delta t} \rightarrow \ldots \rightarrow x_1)$, starting from $x_0 = (\boldsymbol{0}, t=0)$. Here, $\Delta t = 1/T$, where $T$ is the number of discrete time steps.

The forward policy $P_F(x_{t+\Delta t} \mid x_t; \theta)$ is modeled as a Gaussian distribution with mean $x_t + u(x_t, t; \theta) \Delta t$ and covariance $\sigma^2 \Delta t \, \mathbb{I}$:
\begin{equation}
P_F(x_{t+\Delta t} \mid x_t; \theta) = \mathcal{N}\left(x_{t+\Delta t};\ x_t + u(x_t, t; \theta) \Delta t,\ \sigma^2 \Delta t \, \mathbb{I}\right).
\end{equation}
Here, $u(x_t, t; \theta)$ is the learnable score function, $\sigma$ is the standard deviation, and $\mathbb{I}$ denotes the identity matrix to ensure isotropic covariance.

The backward policy $P_B(x_{t-\Delta t} \mid x_t)$ is defined as a discretized Brownian bridge with noise rate $\sigma$:
\begin{equation}
P_B(x_{t-\Delta t} \mid x_t) = \mathcal{N}\left(x_{t-\Delta t};\ \frac{t - \Delta t}{t} x_t,\ \frac{t - \Delta t}{t} \sigma^2 \Delta t \, \mathbb{I}\right).
\end{equation}

The densities of the distributions over complete forward and backward trajectories are given by:
\begin{equation}
P_F(\tau; \theta) = \prod_{i=0}^{T-1} P_F(x_{(i+1)\Delta t} \mid x_{i\Delta t}; \theta), \quad 
P_B(\tau \mid x_1) = \prod_{i=1}^{T-1} P_B(x_{i\Delta t} \mid x_{(i+1)\Delta t}).
\end{equation}
The diffusion policy parameter $\theta$ is trained using the TB loss.

\paragraph{Hyperparameters.} We set $\sigma^2 = 5.0$ for 25GMM and $\sigma = 1.0$ for Manywell, with the number of time steps $T=100$, following \citet{sendera2024diffusion}. We employ the same architecture as \citet{zhang2021path} and \citet{sendera2024diffusion}, increasing the hidden dimension from 64 to 256 for Manywell to accommodate the 32-dimensional tasks, and apply this adjustment to all baselines. For replay buffer capacity, we set it to 5,000 for the 25GMM task and 20,000 for the Manywell task. All learning hyperparameters remain identical to those in \citet{sendera2024diffusion}. For evaluation, we set the number of samples to $M = 2,000$.

\subsection{Biological and Chemical Discovery}\label{app:biochemical}
For biochemical design tasks, we mostly follow the setting of \citet{shen2023towards}. For all tasks, we use a prepend-append MDP (PA-MDP), where the action is defined as adding a token at the beginning or the end of a partial sequence. This MDP formulation makes it possible to have multiple trajectories associated with the same design configuration $x$. 

\paragraph{Hyperparameters.} For training GFlowNets, we use similar setting proposed by prior works \citep{shen2023towards, kim2024local}. We use Adam optimizer \citep{kingma2014adam} with learning rate $10^{-2}$ for log $Z_{\theta}$, $10^{-4}$ for forward policy, $5 \times 10^{-4}$ for teacher policy. To parametrize forward policy, we adopt relative edge flow policy parametrization mapping (SSR) from \citet{shen2023towards}. For QM9 and sEH tasks, we employ a two-layer architecture with 1024 hidden units, while for the other tasks, we choose to use a two-layer architecture with 128 hidden units. We initialize $\log Z_{\theta}$ to 5.0 for all methods. For backward policy, we use a fixed uniform policy. In terms of reward exponent, we use a value of 20 for both QM9 and TFbind8. For sEH and L14-RNA1, we use relatively higher values, 200 and 40, respectively. 

\paragraph{Evaluation metrics.} For evaluation, we compute the number of modes using all the samples collected over the course of training. What we count as a mode should have a high reward.  Unlike with other designs, we define ``mode" as a configuration whose reward is above a certain threshold. This is different from previously used metrics of mode counting to assess diversity. For QM9 and TFbind8, we use a default mode set suggested by \citep{shen2023towards}. For sEH, we set the reward threshold as the top 0.01\% of $\gX$ in terms of the reward and the diversity threshold as 0.4 Tanimoto diversity. For L14-RNA1, we set the reward threshold as the top 0.01\% of $\gX$ in terms of the reward and the diversity threshold as 1 unit of Levenstein distance. We also report ELBO and EUBO which are described in~\cref{app:diffusion_sampler}. We generate $M=2,048$ samples from both the trained policy and the target distribution for the ELBO and EUBO calculation. To sample from the target distribution, we evaluate all possible configurations and sample them proportionally to their reward. As the number of possible configurations is enormously large for sEH and L14-RNA1 tasks, we use the Gumbel-max trick \citep{gumbel1954statistical} for sampling from the discrete probability distribution.

\clearpage

\section{Additional experimental results}
\label{app:exp}

\begin{figure}[t]
    \centering
\begin{minipage}[t]{0.92\textwidth}
\includegraphics[width=1.0\linewidth]{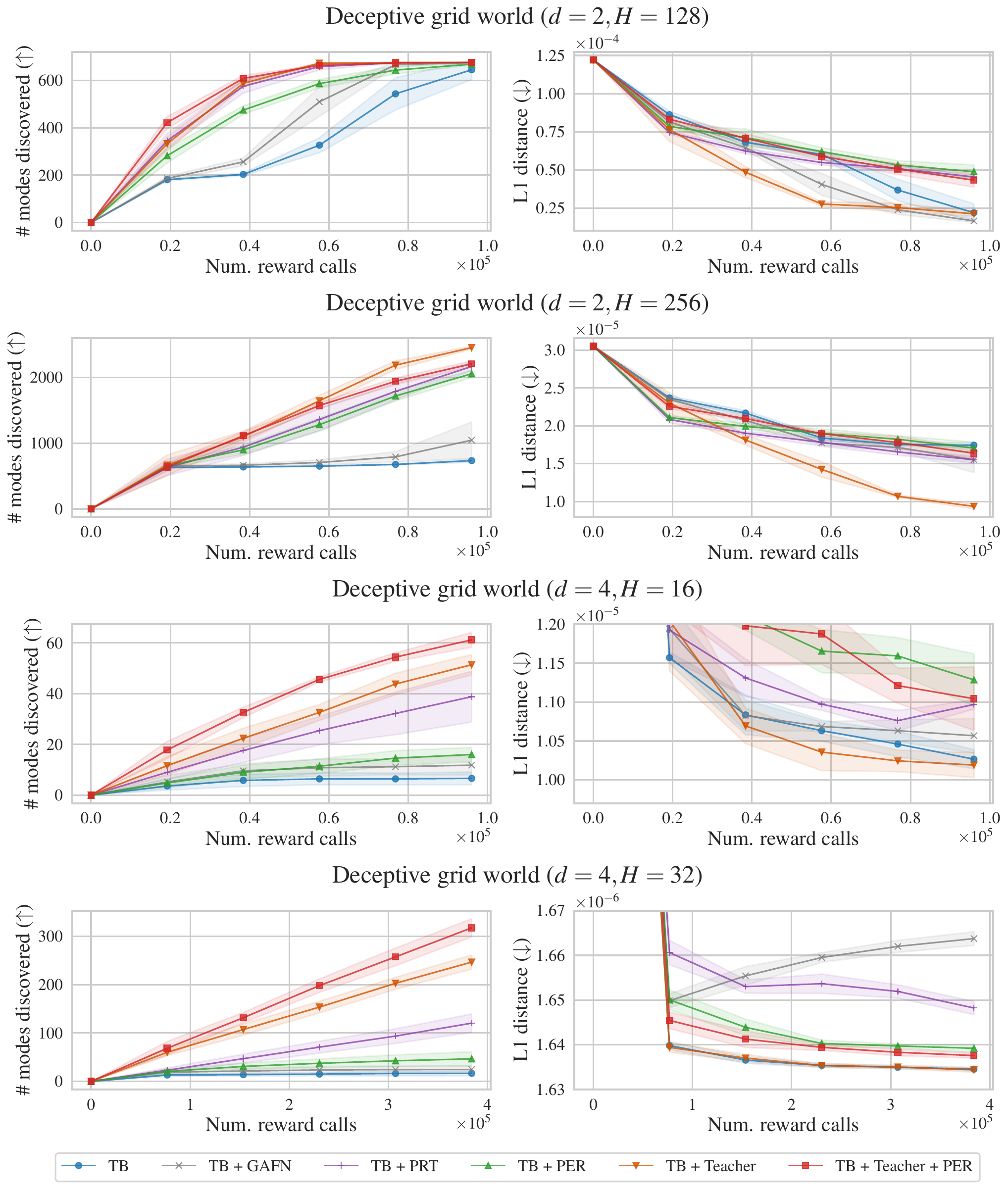}
\caption{Evolution of evaluation metrics for each method in deceptive grid world task. The mean value with standard deviation is depicted from five independent runs.}
    \label{fig:grid_all}
\end{minipage}
\vspace{-10pt}
\end{figure}

\subsection{Extended experimental results of deceptive grid world}
\label{app:grid-results}

\cref{fig:grid_all} shows the changes in two evaluation metrics during training: the number of modes discovered and the empirical $L_1$ distance between the target and sampled distributions. The results indicate that \teacher is highly effective at discovering modes and also generally performs well in terms of the $L_1$ distance. In contrast, while PER and PRT improve mode discovery, they tend to slightly worsen $L_1$ performance. We believe this is because the replay buffer provides a regularizing signal that prevents mode collapse, making the policy learning process more challenging compared to using purely on-policy methods without regularization. This leads to slight underfitting. On the other hand, on-policy methods can easily overfit to the target distribution (but into specific modes), resulting in descent $L_1$ performance. However, without off-policy regularization, they tend to drop modes and fail to cover all modes in the target distribution. \teacher achieves high performance in both $L_1$ and mode coverage, indicating that it not only offers off-policy regularization to ensure comprehensive mode coverage but also provides efficient curricula for faster convergence to the target distribution, leading to strong performance across both metrics.

\clearpage

\begin{table*}[t]\label{tab:nmc}
\vspace*{-0.5em}
    \caption{The effect of number of Monte Carlo samples ($N_{\text{MC}}$) in deceptive grid world task.}\vspace*{-0.75em}
\label{tab:grid_mc}
\centering
    \resizebox{0.75\linewidth}{!}{
    \begin{tabular}{@{}lcccc}
        \toprule
        Grid config. $\rightarrow$
        & \multicolumn{2}{c}{$d = 2$, $H = 256$} 
        & \multicolumn{2}{c}{$d = 4$, $H = 32$} 
        \\
        \cmidrule(lr){2-3}\cmidrule(lr){4-5}
        Algorithm $\downarrow$ Metric $\rightarrow$
        & \# modes ($\uparrow$) & $L_1 \times 10^{-5}$ ($\downarrow$) 
        & \# modes ($\uparrow$) & $L_1 \times 10^{-6}$ ($\downarrow$) 
        \\
        \midrule
        $N_{\text{MC}} = 1$ (default)
        & $2452.6 \pm \text{\footnotesize 21.7}$ & $\mathbf{0.94} \pm \text{\footnotesize 0.03}$ 
        & $\mathbf{246.6} \pm \text{\footnotesize 14.7}$ & $\mathbf{1.634} \pm \text{\footnotesize 0.001}$ 
        \\
        $N_{\text{MC}} = 3$
        & $2489.6 \pm \text{\footnotesize 15.8}$ & $0.96 \pm \text{\footnotesize 0.07}$ 
        & $234.6 \pm \text{\footnotesize 14.2}$ & $1.635 \pm \text{\footnotesize 0.000}$ 
        \\
        $N_{\text{MC}} = 5$
        & $2490.4 \pm \text{\footnotesize 30.3}$ & $\mathbf{0.94} \pm \text{\footnotesize 0.06}$ 
        & $230.8 \pm \text{\footnotesize 3.4}$ & $\mathbf{1.634} \pm \text{\footnotesize 0.000}$ 
        \\
        $N_{\text{MC}} = 10$
        & $\mathbf{2492.0} \pm \text{\footnotesize 26.3}$ & $0.96 \pm \text{\footnotesize 0.05}$ 
        & $239.0 \pm \text{\footnotesize 11.8}$ & $\mathbf{1.634} \pm \text{\footnotesize 0.000}$ 
        \\   
        \bottomrule
    \end{tabular}
    }
\end{table*}

\subsection{Study on Monte Carlo approximation for \texorpdfstring{$R_{\text{\teacher}}$}{TEXT}} \label{app:r_teacher_mc}

We use Monte Carlo estimate with a single trajectory to approximate the expectation in \cref{eq:teacher-reward0} and \cref{eq:teacher-reward1}. To validate that this single-sample approximation is reasonable, we test our algorithm in the deceptive grid world task with an increased number of samples, ranging from 3 to 10. We do not use the replay buffer in this analysis. As described in \cref{tab:nmc}, the performance is not significantly affected by $N_{\text{MC}}$. This suggests that using a Monte Carlo approximation with a sample size of 1 to estimate the stochastic reward $R_{\text{Teacher}}$ was reasonable.

\begin{table*}[t]
\vspace*{-0.5em}
    \caption{Ablation study on $C$ in deceptive grid world task.}\vspace*{-0.75em}
\label{tab:grid_C}
\centering
    \resizebox{0.75\linewidth}{!}{
    \begin{tabular}{@{}lcccc}
        \toprule
        Grid config. $\rightarrow$
        & \multicolumn{2}{c}{$d = 2$, $H = 256$} 
        & \multicolumn{2}{c}{$d = 4$, $H = 32$} 
        \\
        \cmidrule(lr){2-3}\cmidrule(lr){4-5}
        Algorithm $\downarrow$ Metric $\rightarrow$
        & \# modes ($\uparrow$) & $L_1 \times 10^{-5}$ ($\downarrow$) 
        & \# modes ($\uparrow$) & $L_1 \times 10^{-6}$ ($\downarrow$) 
        \\
        \midrule
        $C = 0$
        & $\mathbf{2469.4} \pm \text{\footnotesize 29.7}$ & $1.06 \pm \text{\footnotesize 0.08}$ 
        & $253.2 \pm \text{\footnotesize 11.3}$ & $\mathbf{1.634} \pm \text{\footnotesize 0.000}$ 
        \\
        $C = 9$
        & $2454.2 \pm \text{\footnotesize 18.7}$ & $0.95 \pm \text{\footnotesize 0.02}$ 
        & $\mathbf{256.0} \pm \text{\footnotesize 15.4}$ & $1.635 \pm \text{\footnotesize 0.000}$ 
        \\
        $C = 19$ (default)
        & $2452.6 \pm \text{\footnotesize 21.7}$ & $\mathbf{0.94} \pm \text{\footnotesize 0.03}$ 
        & $246.6 \pm \text{\footnotesize 14.7}$ & $\mathbf{1.634} \pm \text{\footnotesize 0.001}$ 
        \\
        $C = 29$
        & $2465.2 \pm \text{\footnotesize 30.6}$ & $\mathbf{0.94} \pm \text{\footnotesize 0.04}$ 
        & $243.4 \pm \text{\footnotesize 8.5}$ & $\mathbf{1.634} \pm \text{\footnotesize 0.000}$ 
        \\   
        \bottomrule
    \end{tabular}
    }
\end{table*}

\subsection{Ablation study on the choice of \texorpdfstring{$C$}{TEXT} value}
\label{app:C-ablation}

We set the hyperparameter $C$ in \cref{eq:teacher-reward1} to 19 across all experiments without an extensive hyperparameter search. To evaluate this choice, we perform an ablation study on the deceptive grid world task, testing alternative values of 0, 9, and 29. Note that the replay buffer is not used for this analysis. The results are summarized in \cref{tab:grid_C}. Although the effect on the number of modes discovered is somewhat mixed, $C=0$ performs slightly worse than the other values regarding the empirical $L_1$ error, supporting our hypothesis that focusing more on undersampled regions is beneficial. We also found that the results are not highly sensitive to the choice of $C$.

\clearpage

\subsection{Study on \texorpdfstring{$\alpha$}{TEXT} of reward mixing} \label{app:alpha}
We introduce $\alpha$ to mix the reward based on the \student's loss and \student's log reward to help \teacher target both high-loss and high-reward areas effectively. In this section, we investigate the effect of $\alpha$ on the performance of our method.

\begin{table*}[h]
\vspace*{-0.5em}
    \caption{Mixing component study on deceptive grid world task}\vspace*{-0.75em}
\label{tab:grid_alpha}
\centering
    \resizebox{0.7\linewidth}{!}{
    \begin{tabular}{@{}lcccc}
        \toprule
        Grid config. $\rightarrow$
        & \multicolumn{2}{c}{$d = 2$, $H = 256$} 
        & \multicolumn{2}{c}{$d = 4$, $H = 32$} 
        \\
        \cmidrule(lr){2-3}\cmidrule(lr){4-5}
        Algorithm $\downarrow$ Metric $\rightarrow$
        & \# modes ($\uparrow$) & $L_1 \times 10^{-5}$ ($\downarrow$) 
        & \# modes ($\uparrow$) & $L_1 \times 10^{-6}$ ($\downarrow$) 
        \\
        \midrule
        $\alpha = 0.0$
        & $2452.6 \pm \text{\footnotesize 21.7}$ & $0.94 \pm \text{\footnotesize 0.03}$ 
        & $246.6 \pm \text{\footnotesize 14.7}$ & $1.634 \pm \text{\footnotesize 0.001}$ 
        \\
        $\alpha = 0.5$
        & $2415.2 \pm \text{\footnotesize 262.8}$ & $0.90 \pm \text{\footnotesize 0.11}$ 
        & $85.6 \pm \text{\footnotesize 8.5}$ & $1.634 \pm \text{\footnotesize 0.000}$ 
        \\   
        \bottomrule
    \end{tabular}
    }
\end{table*}

\textbf{Deceptive grid world.} \cref{tab:grid_alpha} shows that reward mixing with $\alpha = 0.5$ degrades performances for mode seeking in high dimensional tasks as deceptive grid world task is exploration intensive task; teacher solely focusing on high loss region is more beneficial. Still mixing with $\alpha = 0.5$ outperforms other baselines. 

\begin{table*}[h]
    \caption{Mixing component study on diffusion sampler task}\vspace*{-0.8em}
\label{tab:diffusion_sampling_hyperparameter}
\centering
    \resizebox{0.9\linewidth}{!}{
    \begin{tabular}{@{}lcccccc}
        \toprule
        Energy $\rightarrow$
        & \multicolumn{3}{c}{25GMM ($d=2$, $\log Z = 0$)} 
        & \multicolumn{3}{c}{Manywell ($d=32$, $\log Z = 164.696$)} 
        \\
        \cmidrule(lr){2-4}\cmidrule(lr){5-7}
        Algorithm $\downarrow$ Metric $\rightarrow$ & ELBO ($\uparrow$) &  ELBO-IS ($\uparrow$) & EUBO ($\downarrow$)  & ELBO ($\uparrow$) & ELBO-IS ($\uparrow$) & EUBO  ($\downarrow$) \\
        \midrule
        $\alpha = 0.0$
        & -0.144{\scriptsize$\pm$0.001} 
        & -0.009{\scriptsize$\pm$0.006}
        & 0.122{\scriptsize$\pm$0.010}
        
        & 163.447{\scriptsize$\pm$0.063} 
        & 164.694{\scriptsize$\pm$0.060}
        & 166.024{\scriptsize$\pm$0.001}
        \\   
        $\alpha = 0.5$
        & -0.137{\scriptsize$\pm$0.004} 
        & -0.005{\scriptsize$\pm$0.007}
        & 0.115{\scriptsize$\pm$0.009}
        
        & 163.484{\scriptsize$\pm$0.049} 
        & 164.676{\scriptsize$\pm$0.048}
        & 165.800{\scriptsize$\pm$0.045}
        \\   
        \bottomrule
    \end{tabular}
    }
\end{table*}

\textbf{Diffusion sampling.} As shown in \cref{tab:diffusion_sampling_hyperparameter}, mixing with $\alpha = 0.5$ shows slightly better performance, though both achieve significantly higher sampling efficiency compared to the baselines. Both $\alpha = 0.0$ and $\alpha = 0.5$ come close to reaching the target $\log Z$.

\begin{figure}[h]
    \centering
\begin{minipage}[t]{0.7\textwidth}
\includegraphics[width=1.0\linewidth]{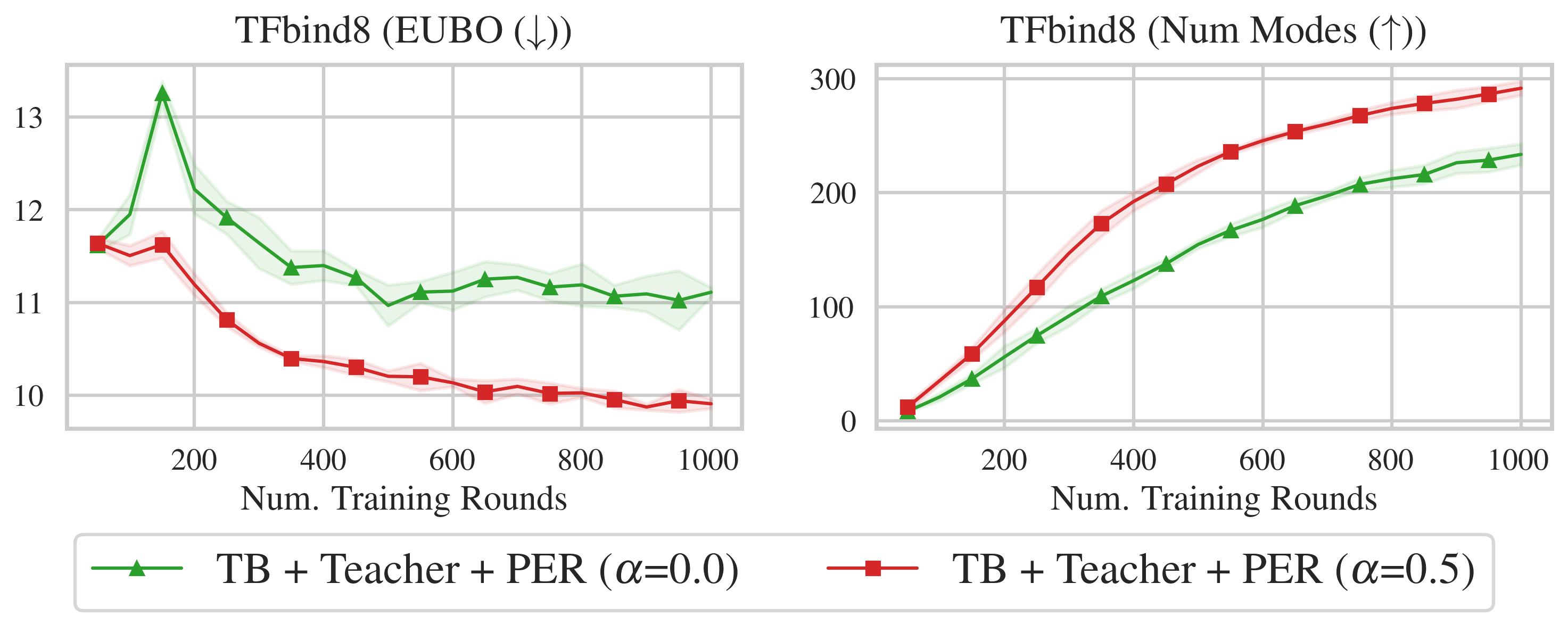}
\caption{Training graph for TFbind8 task by varying $\alpha$ of reward mixing. Mean value with standard deviation is depicted five independent runs.}
    \label{fig:tfbind8_alpha_ablation}
\end{minipage}
\end{figure}

\textbf{Biological and Chemical Discovery (TFbind8).} As shown in \cref{fig:tfbind8_alpha_ablation}, mixing with $\alpha=0.5$ yields significantly better performance in TFbind8 task. This highlights the importance of having the teacher focus on both high-loss and high-reward areas.

\clearpage
\subsection{\teacher with Local search} \label{app:teacher-local}

\begin{figure}[t]
    \centering
\begin{minipage}[t]{0.95\textwidth}
\includegraphics[width=1.0\linewidth]{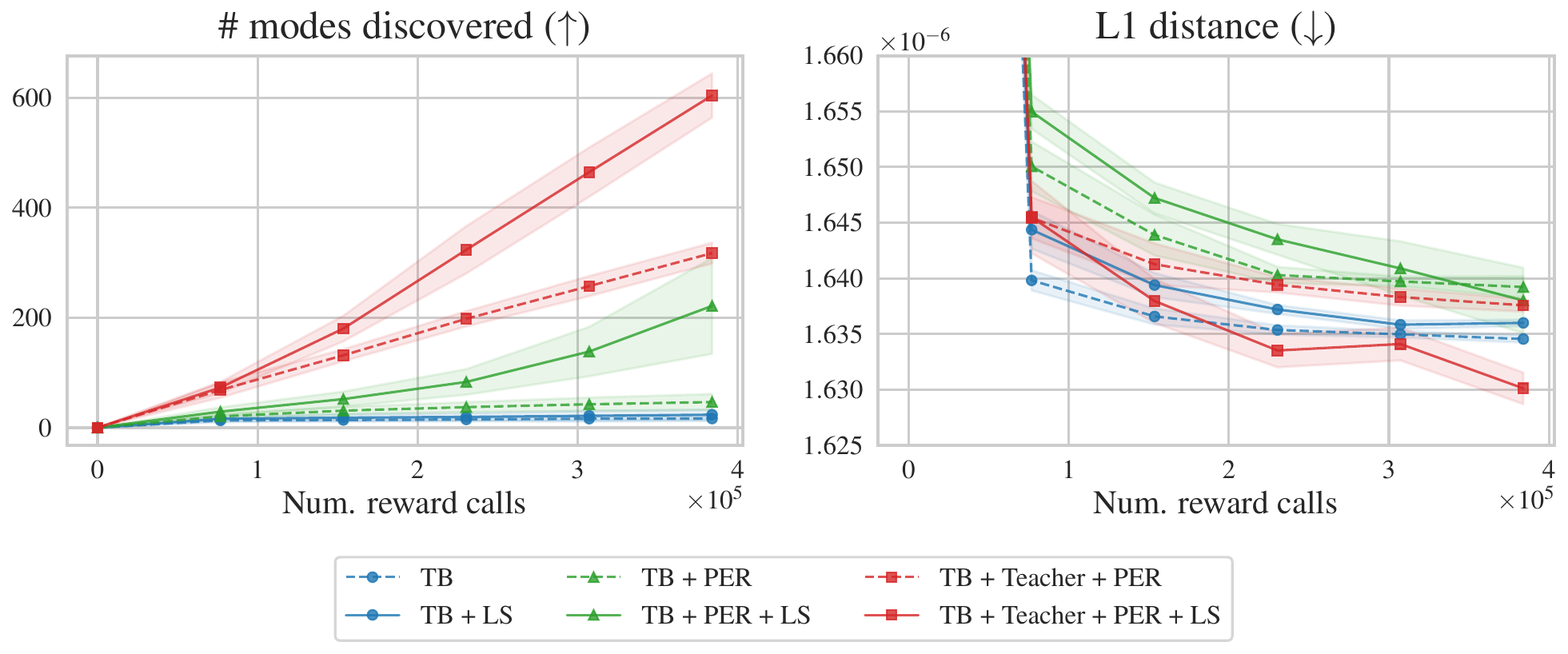}
\caption{Evolution of evaluation metrics with or without the local search in deceptive grid worlds ($d=4, H=32$). The mean value with standard deviation is depicted from five independent runs.}
    \label{fig:grid_LS_modes_and_l1}
\end{minipage}
\end{figure}

\begin{table*}[t]
\vspace*{-0.5em}
    \caption{Comparison with local search (LS) \citep{sendera2024diffusion} on Manywell}\vspace*{-0.8em}
\label{tab:diffusion_local_search}
\centering
    \resizebox{0.6\linewidth}{!}{
    \begin{tabular}{@{}lccc@{}}
        \toprule
        Energy $\rightarrow$ 
        & \multicolumn{3}{c}{Manywell ($d=32$, $\log Z = 164.696$)} 
        \\
        \cmidrule(lr){2-4} 
        Algorithm $\downarrow$ & ELBO ($\uparrow$) & ELBO-IS ($\uparrow$) & EUBO ($\downarrow$)  
        \\
        \midrule
        PER 
        & 161.537{\scriptsize$\pm$0.186}
        & 162.582{\scriptsize$\pm$0.268}
        & 210.440{\scriptsize$\pm$6.888}
        \\
        PER + LS 
        & 163.308{\scriptsize$\pm$0.189}
        & 164.508{\scriptsize$\pm$0.168}
        & 168.953{\scriptsize$\pm$3.209}
        \\
        \midrule
        Teacher
        & \textbf{163.484}{\scriptsize$\pm$0.049}
        & 164.676{\scriptsize$\pm$0.048}
        & 165.800{\scriptsize$\pm$0.045}
        \\
        Teacher + LS
        & 163.472{\scriptsize$\pm$0.086}
        & \textbf{164.685}{\scriptsize$\pm$0.063}
        & \textbf{165.787}{\scriptsize$\pm$0.044}
        \\
        \bottomrule
    \end{tabular}
    }
\end{table*}

Local search is a useful technique to improve the sampling quality of GFlowNets \citep{hu2023gfnem,kim2024local}. In this section, we investigate the possible integration of local search and \teacher and compare it with existing local search integrated solely on \student.

\textbf{Deceptive grid world.} We tested the backtracking-and-reconstruction local search method introduced in~\cref{sec:ls} in the deceptive grid world, using grid configurations of $(d=2, H=256)$ and $(d=4, H=32)$. Four iterative local searches were performed every 16th training batch. The backtracking ratio is 0.5, meaning the last half of a trajectory is destroyed and reconstructed by policy. We applied a deterministic acceptance rule, accepting a new trajectory if it had a higher $R_{\text{\teacher}}$. For comparison, we used two baselines: on-policy TB and TB with PER, both using the same local search but using the task reward $R$ to determine the acceptance.

The experimental results extending is illustrated in~\cref{fig:grid_LS_modes_and_l1}. \teacher with PER outperforms both PER and on-policy TB with a large margin in terms of mode coverage, regardless of whether local search is applied. When combined with local search, \teacher achieves the best results in both the number of modes discovered and the empirical $L_1$ distance. We believe this performance gain is largely due to the reduction of non-stationarity~\cref{sec:ls}, though isolating the exact contribution is complex and left for future work.

\textbf{Diffusion sampling.} We utilize the Manywell task to compare the effect of local search on the \teacher model. Specifically, we employ parallel local search methods \citep{sendera2024diffusion} that leverage Metropolis-Hastings-guided Langevin dynamics (MALA) on samples from the replay buffer to refine sample quality. As shown in the \cref{tab:diffusion_local_search}, by integrating this local search with Prioritized Experience Replay (PER), we observe significant performance improvements.

Remarkably, the \teacher model—even without local search—still outperforms these results. This is notable because the \teacher's exploration does not require gradient information of the energy function, whereas MALA relies on such gradient information. Since the \teacher rapidly achieves optimal sampling quality on the Manywell task, we observe not much improvement when applying local search to the \teacher in the diffusion sampling task.

\begin{figure}[t]
    \centering
\begin{minipage}[t]{0.7\textwidth}
\includegraphics[width=1.0\linewidth]{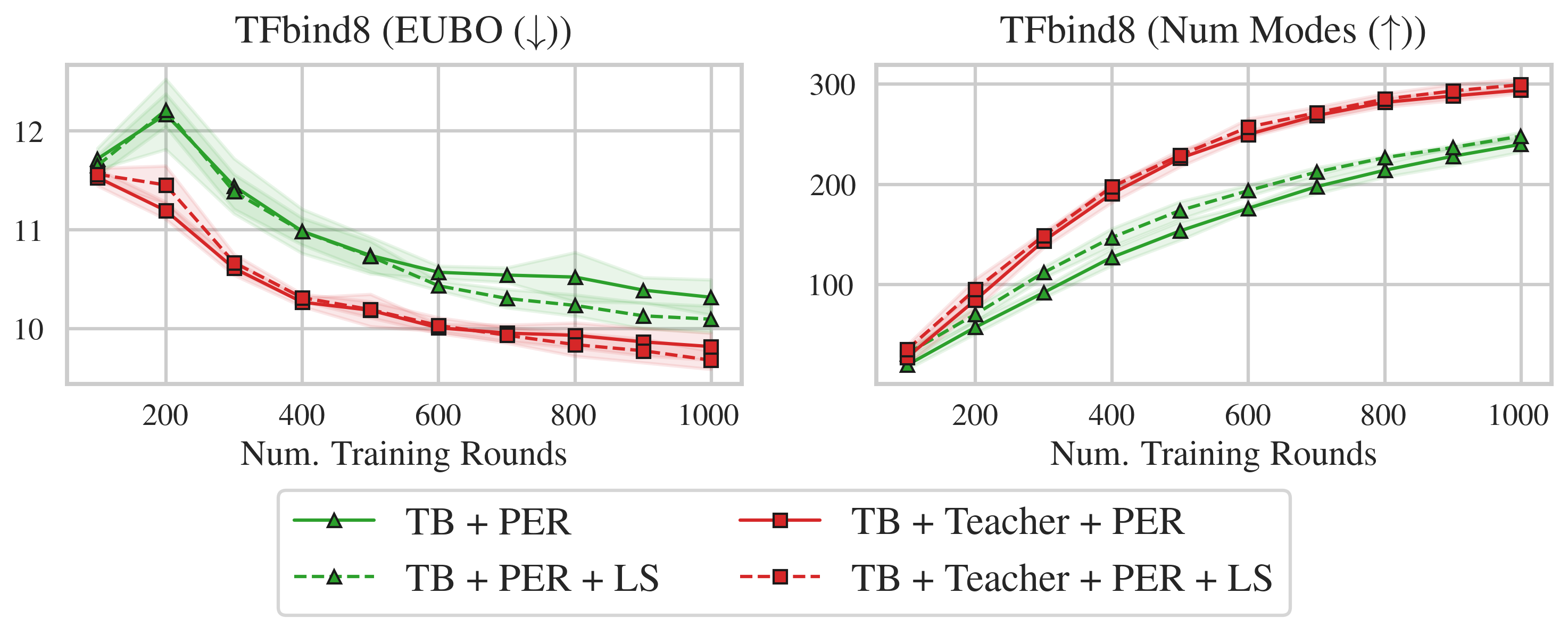}
\caption{Training graph for TFbind8 task by integrating local search \citep{kim2024local}. Mean value with standard deviation is depicted five independent runs.}
    \label{fig:tfbind8_metrics_over_rounds_ls}
\end{minipage}
\end{figure}

\textbf{Biological and Chemical Discovery.} We utilize the TFbind8 task to compare the effect of local search on the \teacher model. Specifically, we employ the local search method suggested by \citep{kim2024local}, which involves backtracking and reconstructing sequences using the forward and backward policies of GFlowNets. The decision to accept adjusted samples is based on whether $R(x') > R(x)$, where $x'$ is the new sample. For the \teacher model, as described in \cref{sec:ls}, we perform local search to mitigate non-stationary in the student and to optimize $R_{\text{teacher}}(x)$. For both the \student and \teacher models, we employ Prioritized Experience Replay (PER).

We visualize the results in \cref{fig:tfbind8_metrics_over_rounds_ls}. As shown in the figure, the \teacher model without local search still outperforms the \student model with local search. This demonstrates that the exploration capability of the \teacher model is far more efficient than conducting local search with the current policy. Moreover, we observe that integrating local search with the \teacher model leads to further improvement in terms of EUBO, highlighting the synergistic effect of combining the \teacher model with local search.

\clearpage

\subsection{2D plots of \teacher and \student over training}

\begin{figure}[t!]
    \captionsetup[subfigure]{labelformat=empty}
    \centering
    \begin{minipage}[t]{1.0\textwidth} 
        \begin{subfigure}[b]{0.16\textwidth}
            \centering
            \includegraphics[width=\textwidth]{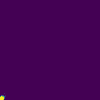}
            \caption{\student (0/5)}
        \end{subfigure}\hfill
        \begin{subfigure}[b]{0.16\textwidth}
            \centering
            \includegraphics[width=\textwidth]{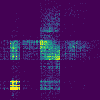}
            \caption{\student (1/5)}
        \end{subfigure}\hfill
        \begin{subfigure}[b]{0.16\textwidth}
            \centering
            \includegraphics[width=\textwidth]{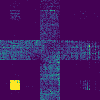}
            \caption{\student (2/5)}
        \end{subfigure}\hfill
        \begin{subfigure}[b]{0.16\textwidth}
            \centering
            \includegraphics[width=\textwidth]{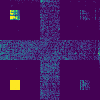}
            \caption{\student (3/5)}
        \end{subfigure}\hfill
        \begin{subfigure}[b]{0.16\textwidth}
            \centering
            \includegraphics[width=\textwidth]{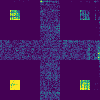}
            \caption{\student (4/5)}
        \end{subfigure}\hfill
        \begin{subfigure}[b]{0.16\textwidth}
            \centering
            \includegraphics[width=\textwidth]{figures/grid/student_5-5.png}
            \caption{\student (5/5)}
        \end{subfigure}       
    \end{minipage}
    
    \vspace{1em} 
    
    \begin{minipage}[t]{1.0\textwidth} 
        \begin{subfigure}[b]{0.16\textwidth}
            \centering
            \includegraphics[width=\textwidth]{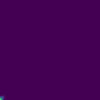}
            \caption{\teacher (0/5)}
        \end{subfigure}\hfill
        \begin{subfigure}[b]{0.16\textwidth}
            \centering
            \includegraphics[width=\textwidth]{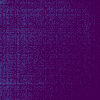}
            \caption{\teacher (1/5)}
        \end{subfigure}\hfill
        \begin{subfigure}[b]{0.16\textwidth}
            \centering
            \includegraphics[width=\textwidth]{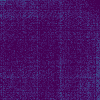}
            \caption{\teacher (2/5)}
        \end{subfigure}\hfill
        \begin{subfigure}[b]{0.16\textwidth}
            \centering
            \includegraphics[width=\textwidth]{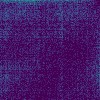}
            \caption{\teacher (3/5)}
        \end{subfigure}\hfill
        \begin{subfigure}[b]{0.16\textwidth}
            \centering
            \includegraphics[width=\textwidth]{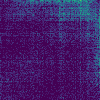}
            \caption{\teacher (4/5)}
        \end{subfigure}\hfill
        \begin{subfigure}[b]{0.16\textwidth}
            \centering
            \includegraphics[width=\textwidth]{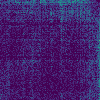}
            \caption{\teacher (5/5)}
        \end{subfigure}
    \end{minipage}
    
    \caption{Illustration of the distribution dynamics between the \teacher and \student models, along with their stationary distributions. The \student(\textit{ratio}) represents the fraction of completed training steps. }
    \label{fig:grid-teacher-student-dynamics-appendix}
\end{figure}

\begin{figure}[t!]
    \captionsetup[subfigure]{labelformat=empty}

    \centering
    \begin{minipage}[t]{1.0\textwidth} 
        \begin{subfigure}[b]{0.16\textwidth}
            \centering
            \includegraphics[width=\textwidth]{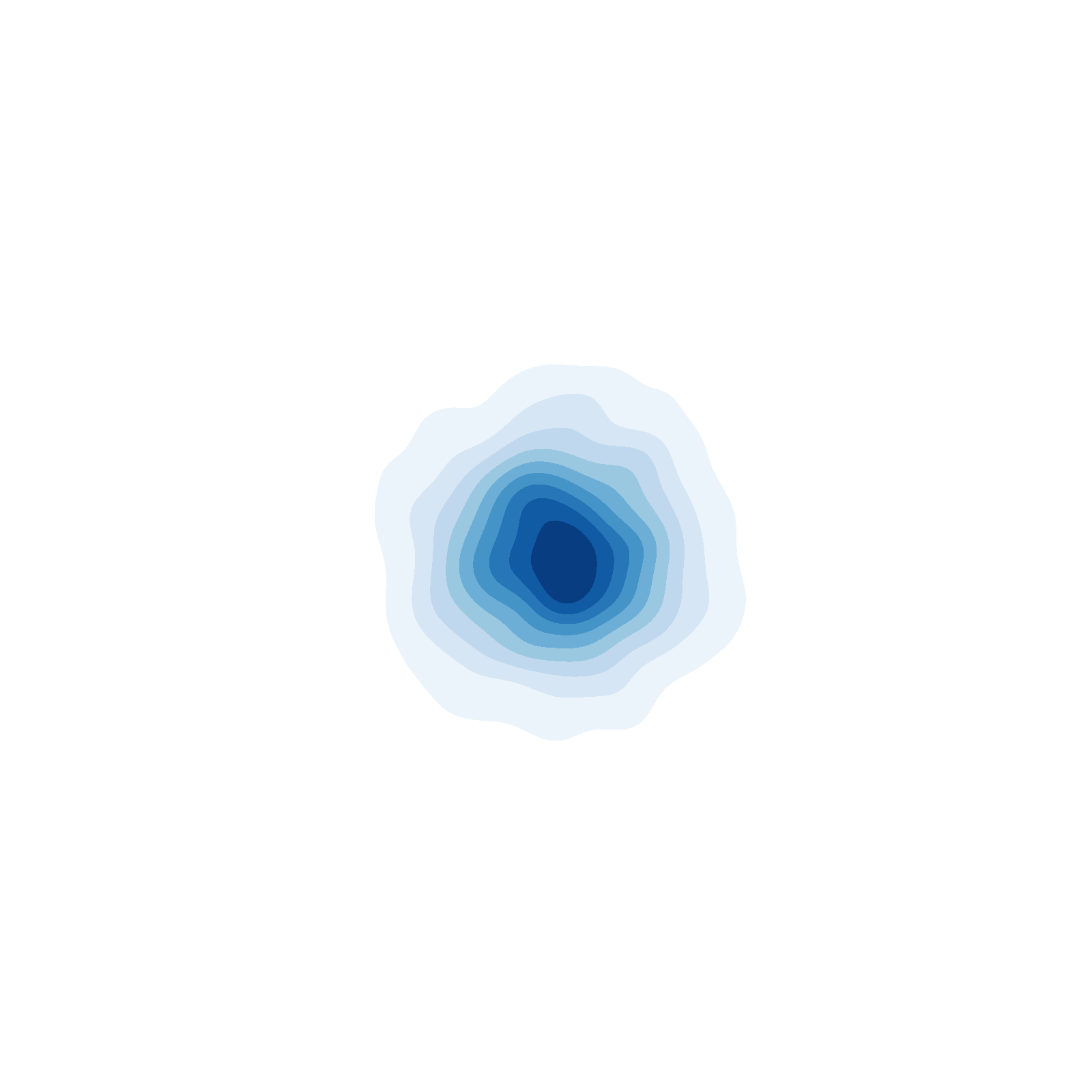}
            \caption{\student (0/5)}
        \end{subfigure}\hfill
        \begin{subfigure}[b]{0.16\textwidth}
            \centering
            \includegraphics[width=\textwidth]{figures/gmm/student-gmm.png}
            \caption{\student (1/5)}
        \end{subfigure}\hfill
        \begin{subfigure}[b]{0.16\textwidth}
            \centering
            \includegraphics[width=\textwidth]{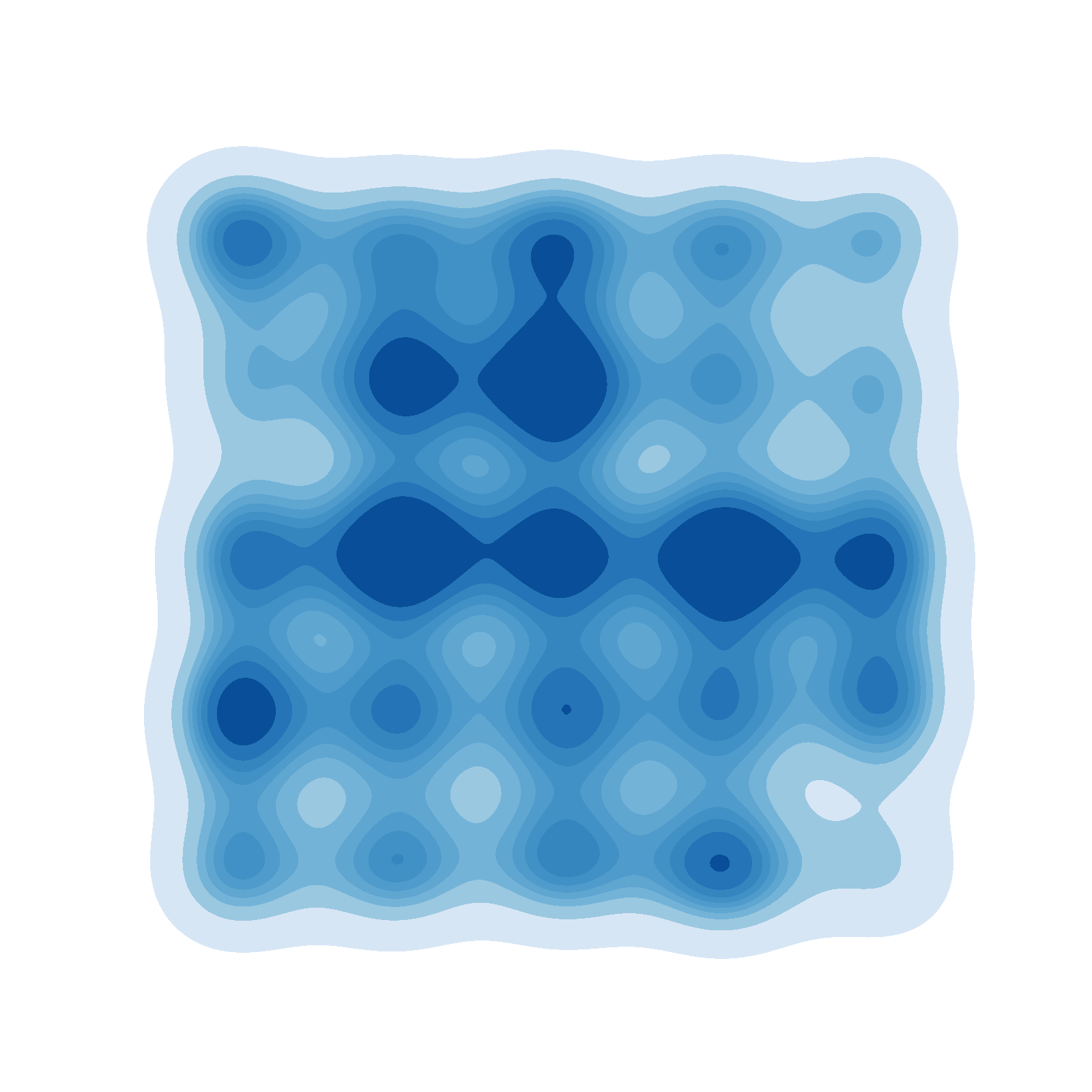}
            \caption{\student (2/5)}
        \end{subfigure}\hfill
        \begin{subfigure}[b]{0.16\textwidth}
            \centering
            \includegraphics[width=\textwidth]{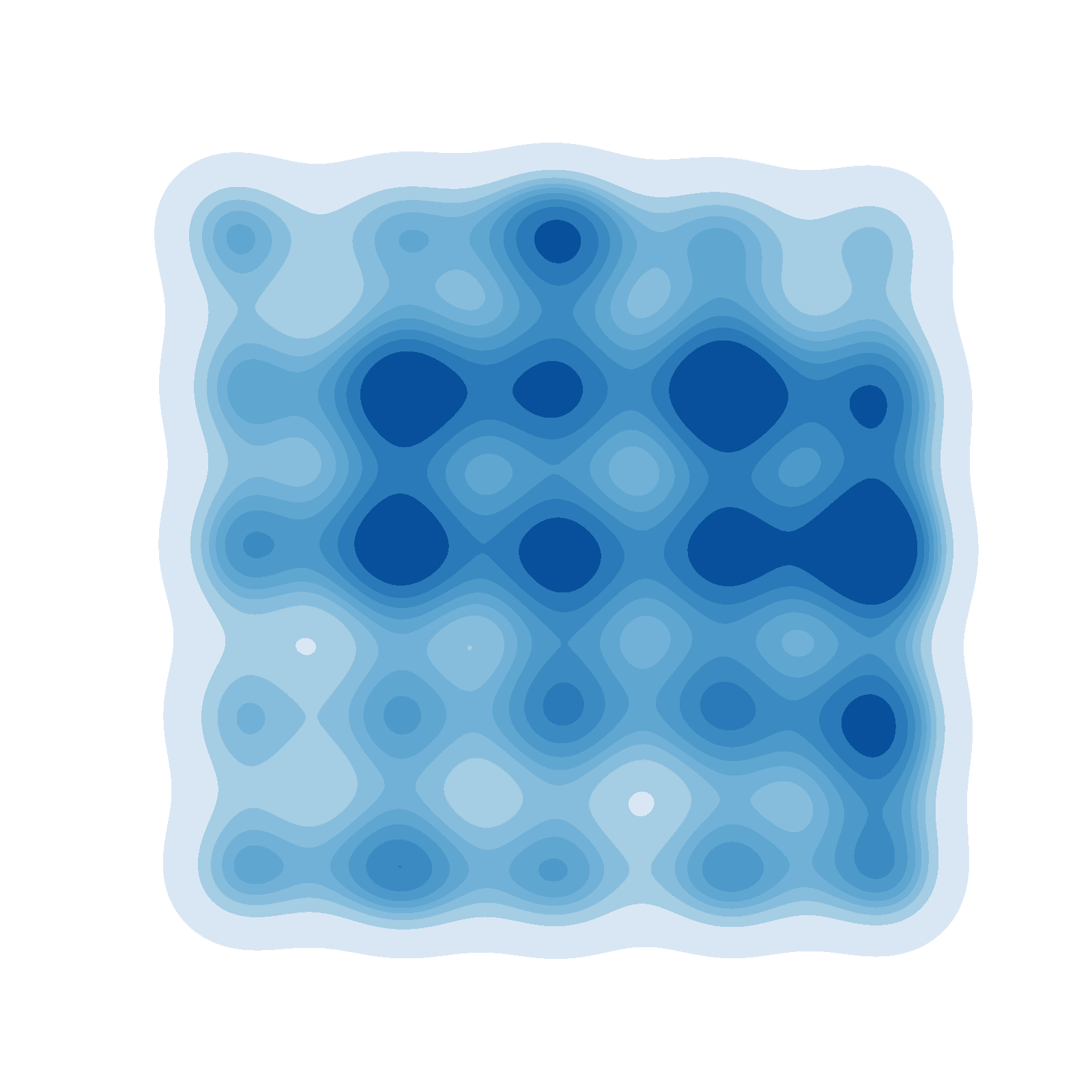}
            \caption{\student (3/5)}
        \end{subfigure}\hfill
        \begin{subfigure}[b]{0.16\textwidth}
            \centering
            \includegraphics[width=\textwidth]{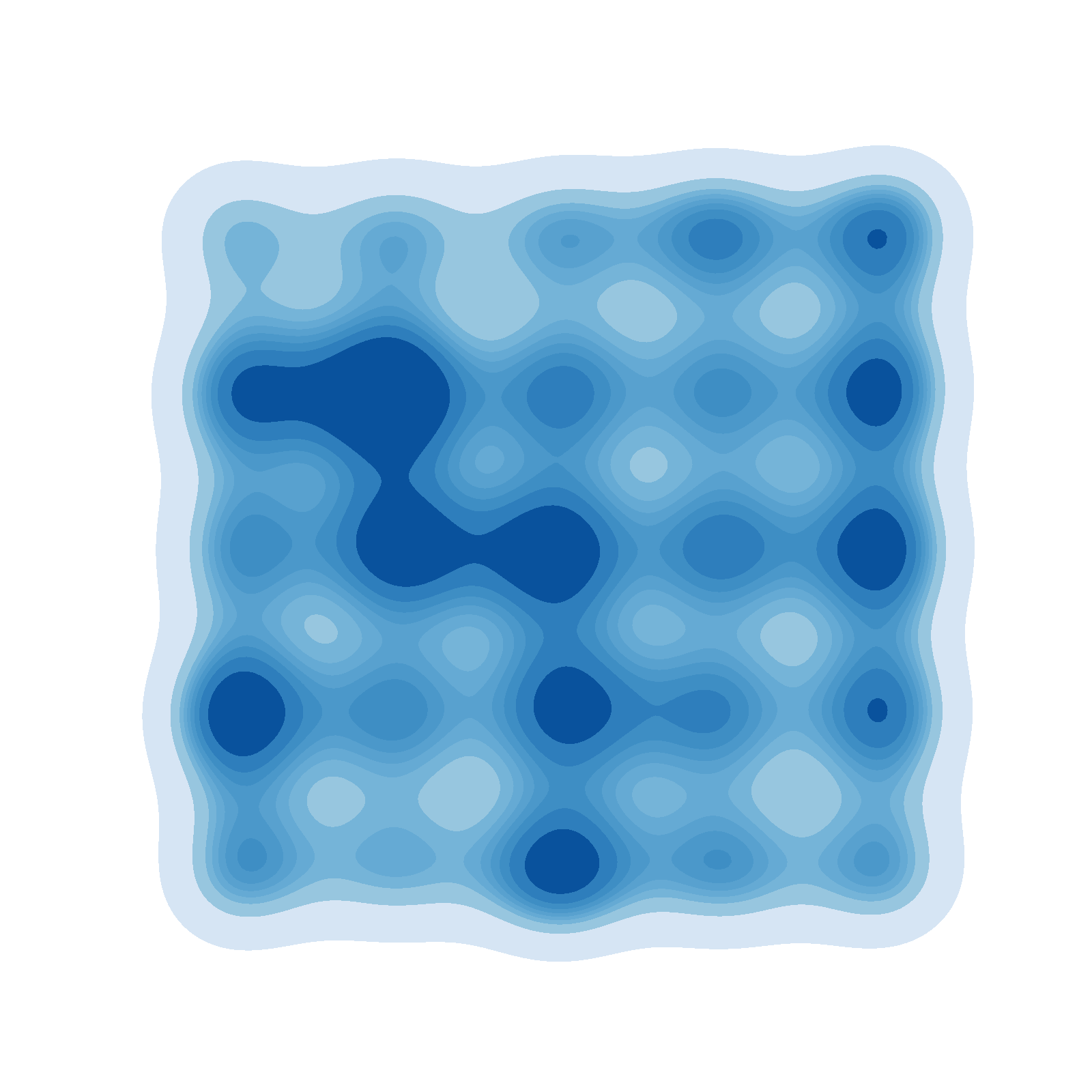}
            \caption{\student (4/5)}
        \end{subfigure}\hfill
        \begin{subfigure}[b]{0.16\textwidth}
            \centering
            \includegraphics[width=\textwidth]{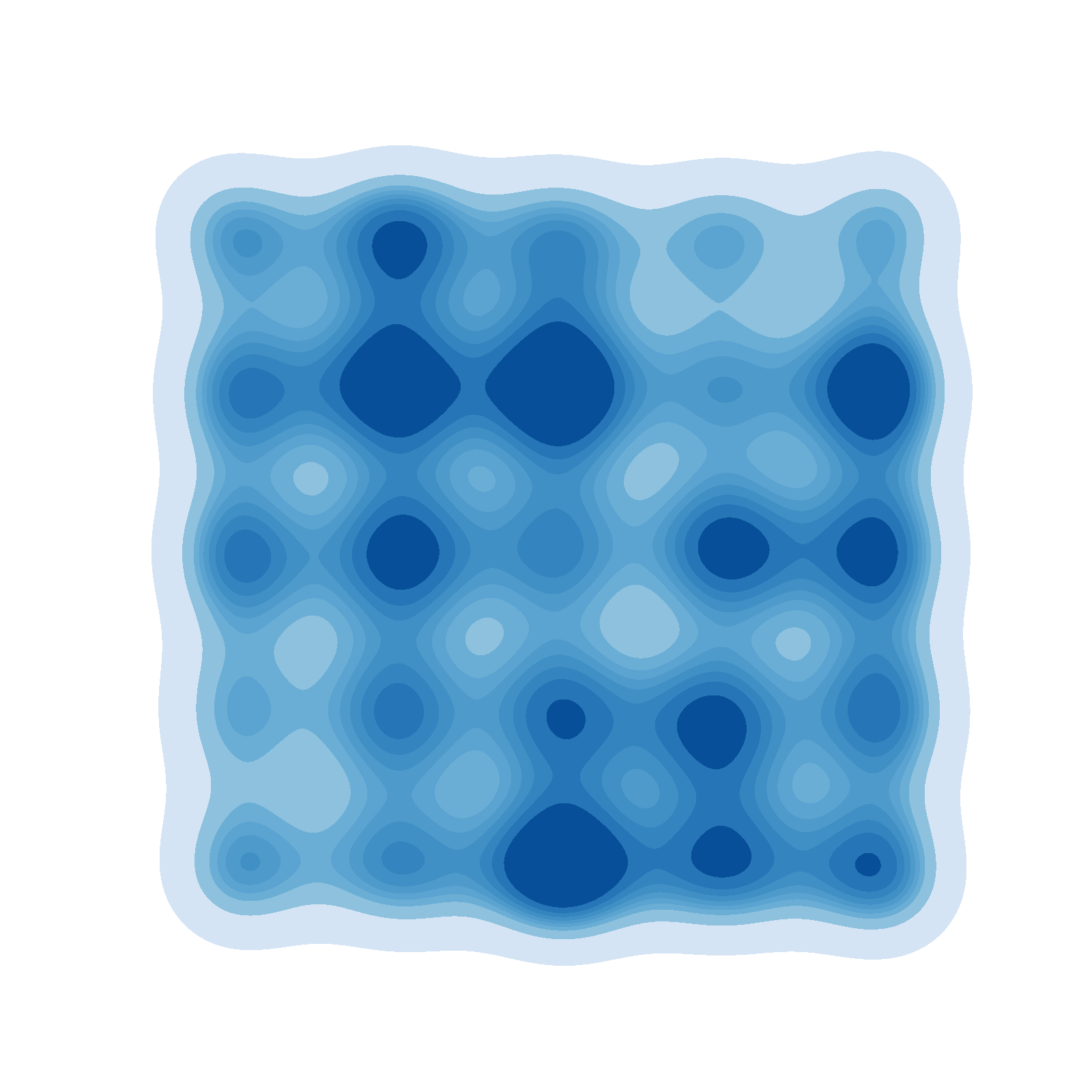}
            \caption{\student (5/5)}
        \end{subfigure}       
    \end{minipage}
    
    \vspace{1em} 
    
    \begin{minipage}[t]{1.0\textwidth} 
        \begin{subfigure}[b]{0.16\textwidth}
            \centering
            \includegraphics[width=\textwidth]{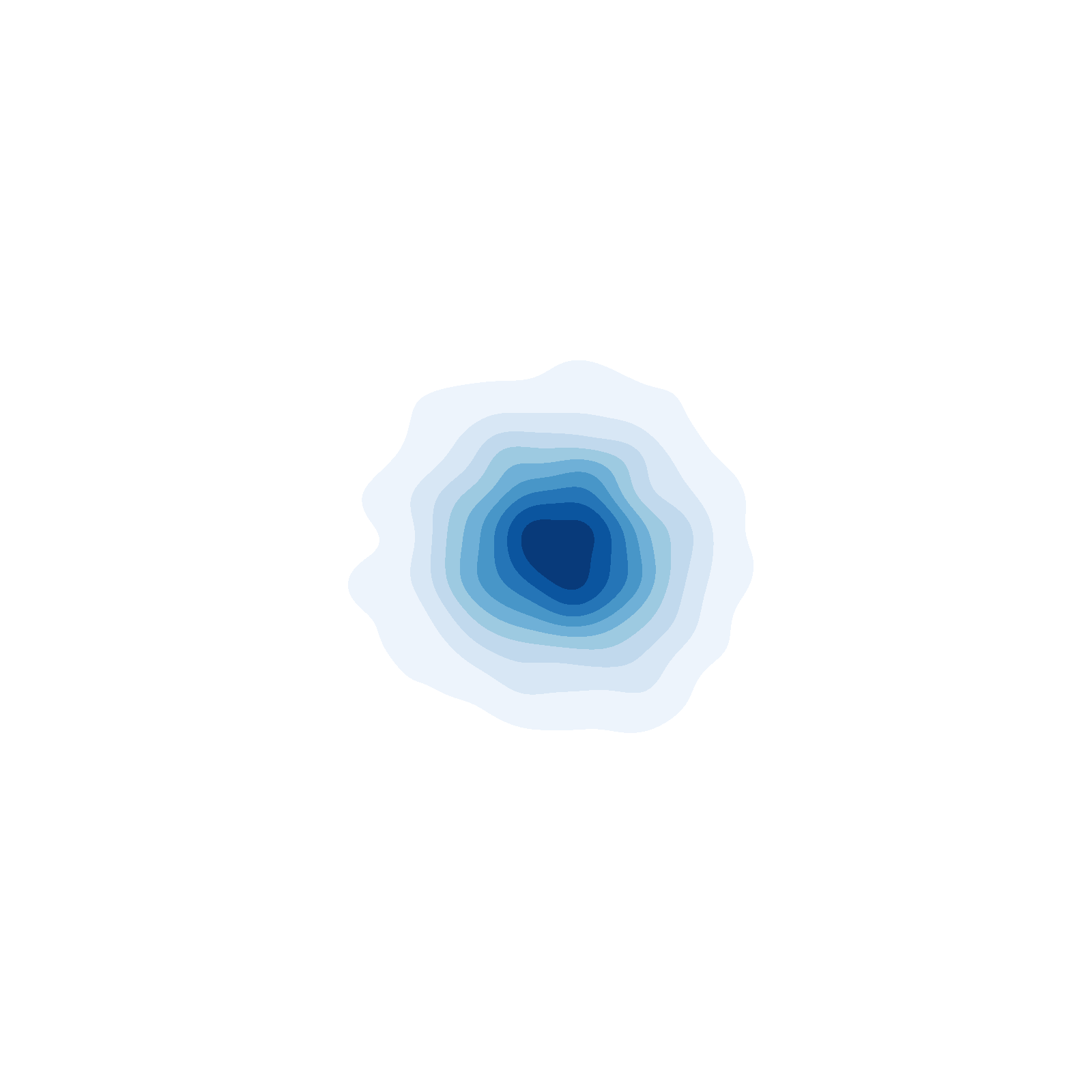}
            \caption{\teacher (0/5)}
        \end{subfigure}\hfill
        \begin{subfigure}[b]{0.16\textwidth}
            \centering
            \includegraphics[width=\textwidth]{figures/gmm/teacher-gmm.png}
            \caption{\teacher (1/5)}
        \end{subfigure}\hfill
        \begin{subfigure}[b]{0.16\textwidth}
            \centering
            \includegraphics[width=\textwidth]{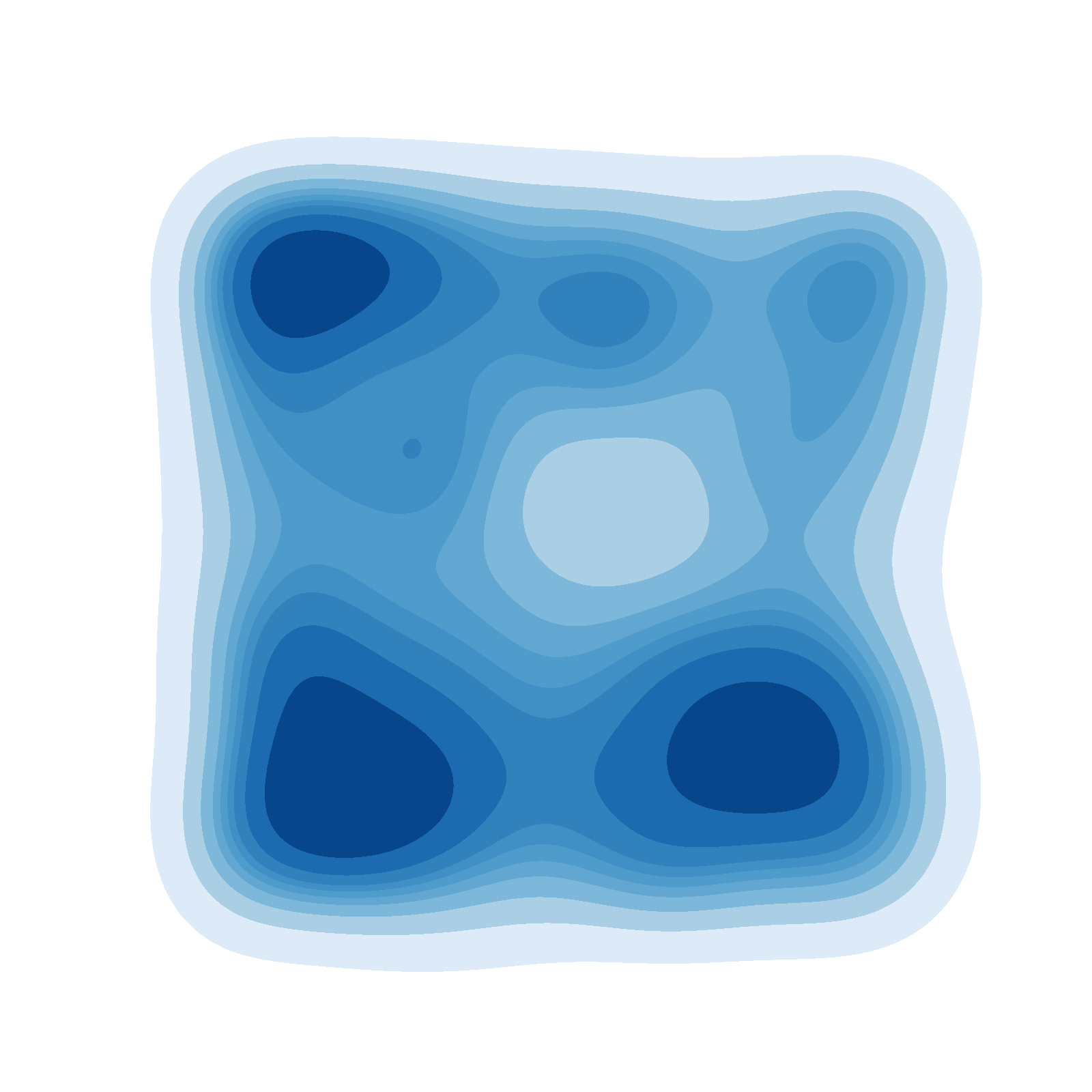}
            \caption{\teacher (2/5)}
        \end{subfigure}\hfill
        \begin{subfigure}[b]{0.16\textwidth}
            \centering
            \includegraphics[width=\textwidth]{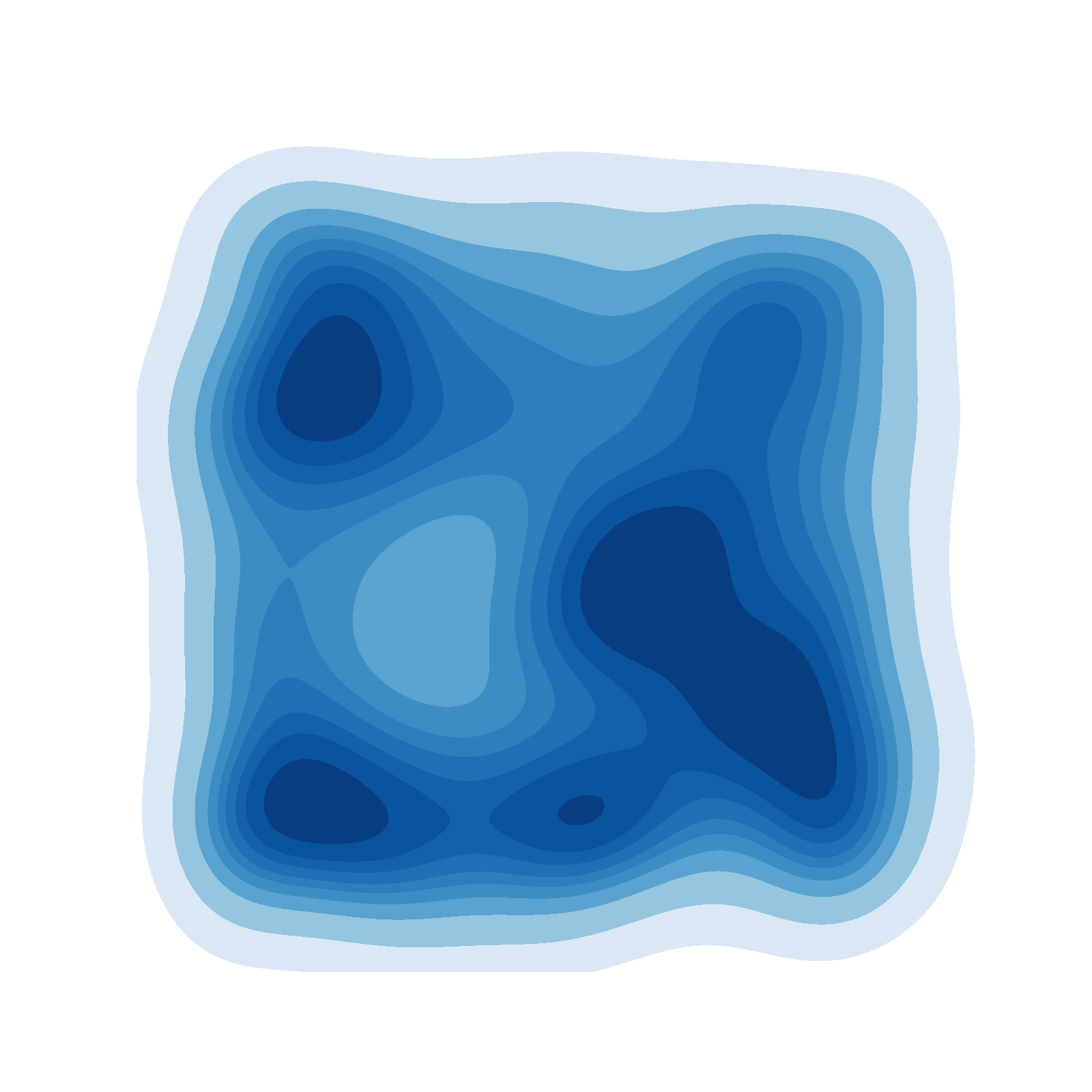}
            \caption{\teacher (3/5)}
        \end{subfigure}\hfill
        \begin{subfigure}[b]{0.16\textwidth}
            \centering
            \includegraphics[width=\textwidth]{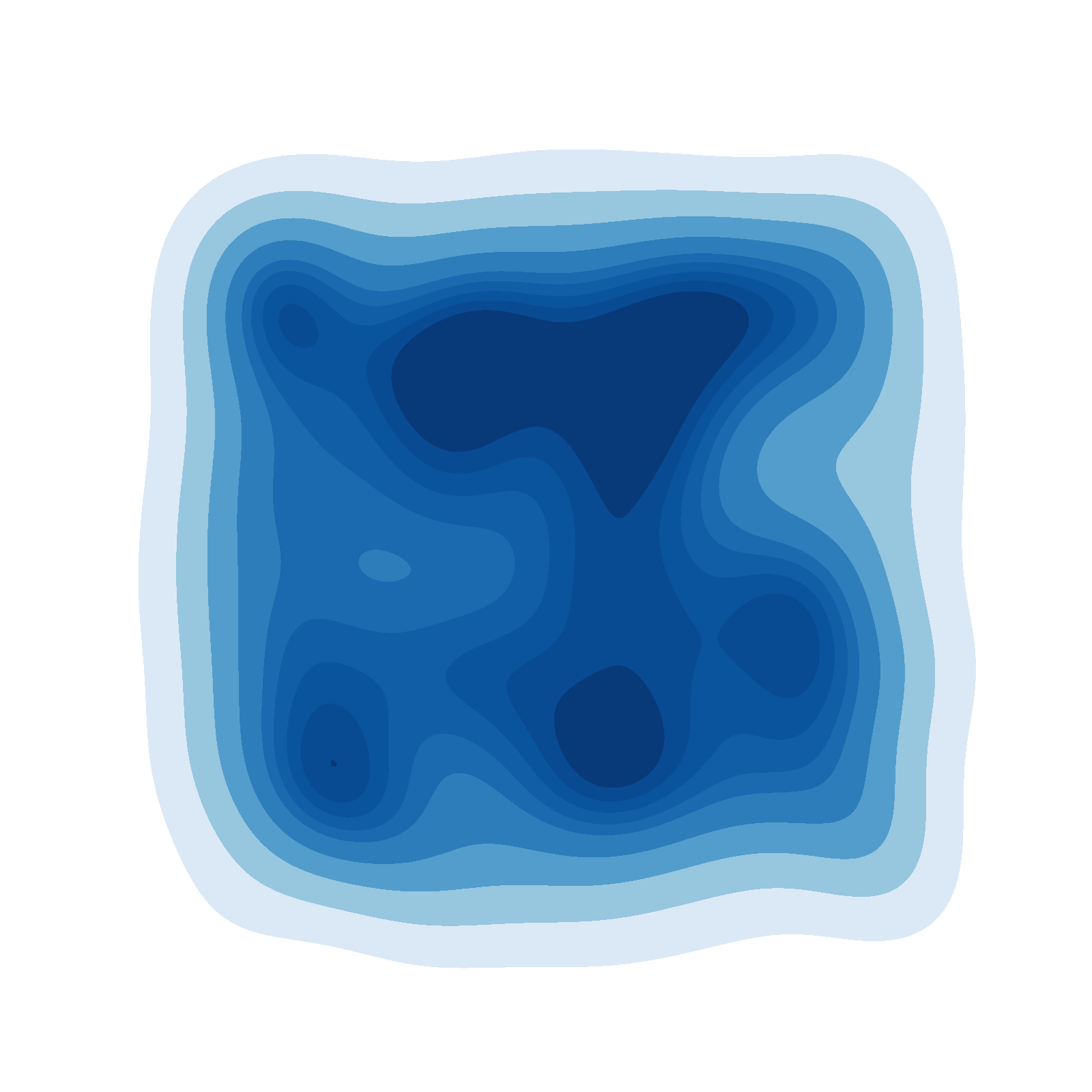}
            \caption{\teacher (4/5)}
        \end{subfigure}\hfill
        \begin{subfigure}[b]{0.16\textwidth}
            \centering
            \includegraphics[width=\textwidth]{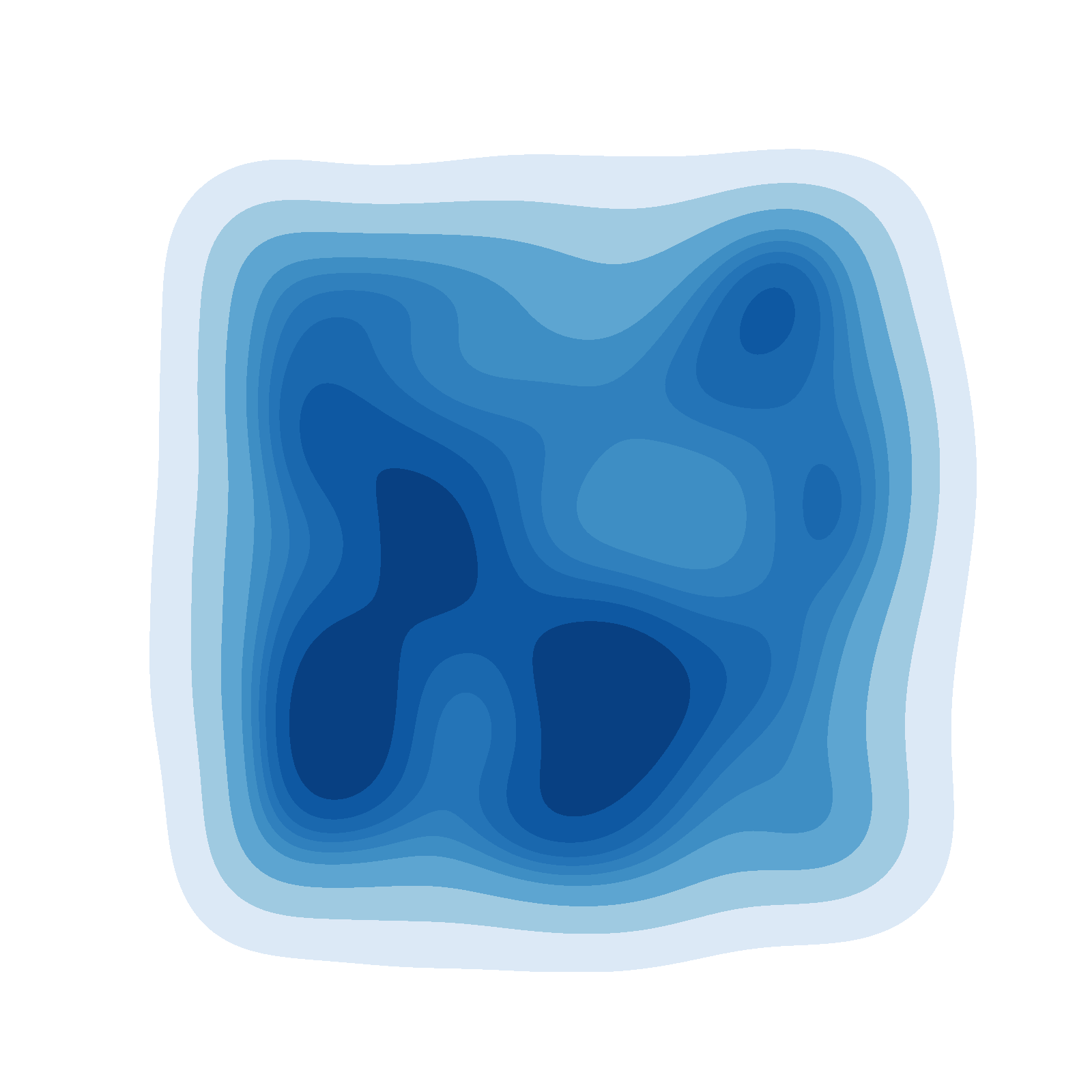}
            \caption{\teacher (5/5)}
        \end{subfigure}       
    \end{minipage}
    
    \caption{Illustration of the distribution dynamics between the \teacher and \student models, along with their stationary distributions. The \student(\textit{ratio}) represents the fraction of completed training steps. }
    \label{fig:diffusion-teacher-student-dynamics-appendix}
\end{figure}

This section presents the distributions of the \teacher and \student models during training. For this visualization, we set the mixing component $\alpha = 0$, meaning the \teacher's major objective is to explore the \student's loss regions. As shown in \cref{fig:grid-teacher-student-dynamics-appendix} and \cref{fig:diffusion-teacher-student-dynamics-appendix}, the \teacher effectively identifies the \student's missing modes, providing a suitable training distribution throughout the epochs, ultimately enabling the \student to discover all the modes successfully.

\clearpage
\section{\teacher for detailed balance}
\label{app:db}

In this section, we extend the proposed idea in~\cref{sec:method} to detailed balance~\cite[DB;][]{bengio2023gflownet}, another GFlowNet learning objective. 

\subsection{Detailed balance and \teacher's reward for DB}

The general problem settings, including MDP formulation and reward function, are the same as~\cref{sec:preliminaries}. Unlike TB, which requires parameterizing $Z_\theta$, DB parameterizes the \emph{state flow function} $F(s;\theta)$ for each state, along with $P_F$ and $P_B$. Note that $F(x;\theta)=R(x)$ for every terminal state $x \in \gX$, and the initial state flow $F(s_0)$ is an estimate of the total reward. The detailed balance discrepancy is defined for any transition of states $(s, a, s')$ as
\begin{equation}\label{eq:db}
\delta_{\text{DB}}(s, a, s'; \theta):= \underbrace{[\log F(s'; \theta) + \log P_B(s \mid s'; \theta)]}_{\text{backward edge flow}} - \underbrace{[\log F(s; \theta) + \log P_F(s' \mid s; \theta)]}_{\text{forward edge flow}}.
\end{equation}

Same as TB, when $\delta_{\text{DB}}(s, a, s'; \theta) = 0$ for all $(s, a, s')$, then $P_F^\top(x)=R(x)/Z$ is achieved for all $x$, where $P_F^\top$ defined as~\cref{eq:terminating_dist}. We can naturally define a DB loss as $\delta_{\text{DB}}(s, a, s'; \theta)^2$ on each transition $(s, a, s')$ sampled from a behavior policy $\pi$. For more formal derivation, please refer to \citet{bengio2023gflownet}.

Analogous to \cref{eq:teacher-reward0} and \cref{eq:teacher-reward1}, we define the basic form as and the re-weighted version of \teacher's reward for DB. The basic form is 
\begin{equation}
    \log R^{\rm basic}_{\text{\teacher-DB}}(x; \theta) = \mathbb{E}_{P_B( \tau \mid x)} \left[\log \left( \sum_{(s, a, s') \in \tau} \delta_{\text{DB}} \left(s, a, s'; \theta\right)^2 \right)\right],
    \label{eq:db-teacher-reward0} 
\end{equation}
and the re-weighted version is
\begin{equation}
    \log R_{\text{\teacher-DB}}^{\rm weighted}(x; \theta) = \mathbb{E}_{P_B(\tau \mid x; \theta)} \left[\log \left( \epsilon + \sum_{(s, a, s') \in \tau} \left(1 + C\mathbb{I}_{\delta_{\text{DB}} \left(s, a, s'; \theta\right)>0}\right)\delta_{\text{DB}} \left(s, a, s'; \theta\right)^2  \right)\right] ,\label{eq:db-teacher-reward1}
\end{equation}
where we approximate the expectation using a single trajectory.

The reward mixing in~\cref{eq:mixing} is not directly available in the DB case since it entails a non-trivial credit assignment problem. Thus, we set $R_{\text{\teacher-DB}}=R_{\text{\teacher-DB}}^{\rm weighted}$.

The overall training procedure is similar to~\cref{sec:training} and Algorithm~\ref{algorithm1}, except we use DB loss for both \student and \teacher training. For a given transition $(s, a, s') \sim \pi$, the DB-loss functions are defined by
\begin{align}
    \mathcal{L}_{\text{\student-DB}}(s, a, s'; \theta) &= \delta_{\text{DB}}(s, a, s'; \theta)^2 = \left( \log \frac{F(s; \theta) P_F(s' \mid s; \theta)}{F(s'; \theta) P_B(s \mid s')} \right)^2,
    \label{eq:db-student_loss}
    \\
    \mathcal{L}_{\text{\teacher-DB}}(s, a, s'; \phi) &= \delta_{\text{\teacher-DB}}(s, a, s'; \phi)^2 = \left( \log \frac{F(s; \phi) P_F(s' \mid s; \phi)}{F(s'; \phi) P_B(s \mid s')} \right)^2,
    \label{eq:db-teacher_loss}
\end{align}
where $F(x;\theta)=R(x)$ and $F(x;\phi)=R_{\text{\teacher-DB}}$.


\begin{table*}[t!]
    \caption{Evaluation results of DB algorithms on deceptive grid worlds with dimension $d$ and grid length $H$. Mean and standard deviation from 5 independent runs are reported. The \textbf{bold} is applied to the best mean value among DB-based methods.}\vspace*{-0.75em}
    \label{tab:db-grid-results}
    \resizebox{1\linewidth}{!}{
    \begin{tabular}{@{}lcccccccc}
        \toprule
        Grid config. $\rightarrow$
        & \multicolumn{2}{c}{$d = 2$, $H = 128$}
        & \multicolumn{2}{c}{$d = 2$, $H = 256$}
        & \multicolumn{2}{c}{$d = 4$, $H = 16$}
        & \multicolumn{2}{c}{$d = 4$, $H = 32$}
        \\
        \cmidrule(lr){2-3}\cmidrule(lr){4-5}\cmidrule(lr){6-7}\cmidrule(lr){8-9}
        Algorithm $\downarrow$ Metric $\rightarrow$ 
        & \# modes ($\uparrow$) & $L_1 \times 10^{-5}$ ($\downarrow$) 
        & \# modes ($\uparrow$) & $L_1 \times 10^{-5}$ ($\downarrow$) 
        & \# modes ($\uparrow$) & $L_1 \times 10^{-5}$ ($\downarrow$) 
        & \# modes ($\uparrow$) & $L_1 \times 10^{-6}$ ($\downarrow$) \\
        \midrule
        TB \hfill(on-policy $\rightarrow$)
        & $645.4 \pm \text{\footnotesize 41.5}$ & $2.20 \pm \text{\footnotesize 0.58}$ 
        & $733.6 \pm \text{\footnotesize 25.1}$ & $1.74 \pm \text{\footnotesize 0.04}$ 
        & $6.6 \pm \text{\footnotesize 2.5}$ & $1.027 \pm \text{\footnotesize 0.012}$ 
        & $16.6 \pm \text{\footnotesize 4.8}$ & $1.635 \pm \text{\footnotesize 0.000}$ 
        \\
        \hspace{2pt} + \teacher
        & $676.0 \pm \text{\footnotesize 0.0}$ & $2.13 \pm \text{\footnotesize 0.18}$ 
        & $2452.6 \pm \text{\footnotesize 21.7}$ & $0.94 \pm \text{\footnotesize 0.03}$ 
        & $51.4 \pm \text{\footnotesize 4.0}$ & $1.019 \pm \text{\footnotesize 0.016}$ 
        & $246.6 \pm \text{\footnotesize 14.7}$ & $1.634 \pm \text{\footnotesize 0.001}$ 
        \\

        \midrule
        DB \hfill(on-policy $\rightarrow$)
        & $644.0 \pm \text{\footnotesize 13.3}$ & $4.57 \pm \text{\footnotesize 0.12}$ 
        & $1025.4 \pm \text{\footnotesize 132.8}$ & $1.69 \pm \text{\footnotesize 0.03}$ 
        & $1.8 \pm \text{\footnotesize 1.7}$ & $\mathbf{1.003} \pm \text{\footnotesize 0.009}$ 
        & $6.8 \pm \text{\footnotesize 2.9}$ & $1.634 \pm \text{\footnotesize 0.000}$ 
        \\
        \hspace{2pt} + $\epsilon$-expl.
        & $578.2 \pm \text{\footnotesize 25.9}$ & $5.23 \pm \text{\footnotesize 0.07}$ 
        & $814.6 \pm \text{\footnotesize 23.4}$ & $1.75 \pm \text{\footnotesize 0.03}$ 
        & $2.4 \pm \text{\footnotesize 0.8}$ & $1.010 \pm \text{\footnotesize 0.004}$ 
        & $25.2 \pm \text{\footnotesize 2.6}$ & $1.635 \pm \text{\footnotesize 0.000}$ 
        \\
        \hspace{2pt} + PER*
        & $316.6 \pm \text{\footnotesize 166.5}$ & $6.01 \pm \text{\footnotesize 0.62}$ 
        & $899.4 \pm \text{\footnotesize 241.7}$ & $1.75 \pm \text{\footnotesize 0.03}$ 
        & $2.6 \pm \text{\footnotesize 1.0}$ & $1.005 \pm \text{\footnotesize 0.012}$ 
        & $17.2 \pm \text{\footnotesize 12.5}$ & $1.634 \pm \text{\footnotesize 0.000}$ 
        \\
        \hspace{2pt} + \teacher
        & $\mathbf{675.4} \pm \text{\footnotesize 0.5}$ & $\mathbf{4.42} \pm \text{\footnotesize 0.15}$ 
        & $\mathbf{1817.6} \pm \text{\footnotesize 49.8}$ & $\mathbf{1.59} \pm \text{\footnotesize 0.01}$ 
        & $\mathbf{58.2} \pm \text{\footnotesize 5.1}$ & $1.009 \pm \text{\footnotesize 0.007}$ 
        & $\mathbf{345.2} \pm \text{\footnotesize 41.2}$ & $\mathbf{1.632} \pm \text{\footnotesize 0.004}$ 
        \\

        \bottomrule
    \end{tabular}
    }
    \vspace{-5pt}
\end{table*}

\subsection{Experiments in deceptive grid world}

We use the same experimental settings as~\cref{sec:grid}. We incorporate a transition-based replay buffer, meaning that we save all state transitions along the trajectory rather than saving only the terminal state $x$ as in the TB case. This allows a closer implementation of the original PER, where the prioritization is performed with a TD error for each transition. To distinguish the PER used in \cref{sec:grid}, we call the transition-based PER as PER*. Regarding the baselines, we do not benchmark PRT as it is not trivial to assign an episodic reward at the terminal state to each transition. We omit the GAFN since its source code only supports the TB algorithm. We also include TB, TB with \teacher for reference. 

The result is summarized in~\cref{tab:db-grid-results}. Similar to the TB case (\cref{tab:grid-results}), \teacher provides a significant improvement over baselines for DB. This confirms that our method offers flexibility across different GFlowNets objective functions.

\section{Scaling experiments}

In this section, we demonstrate the scalability of our method. We first test it on larger-scale tasks on Deceptive Gridworlds, to show that its effectiveness remains consistent as the scale increases. Then we apply our method to a real-world task of prompt sampling on large language models (LLMs), where we discover desirable prompt sentences that require effective exploration of the combinatorially large language search space.

\begin{table*}[htbp]
    \centering
    \caption{Evaluation results on large-scale deceptive grid worlds with dimension $d$ and grid length $H$. The mean and standard deviation of the number of modes discovered (\# modes) from 3 independent runs are reported. Due to the computational expense of obtaining the exact target distribution in large-scale problems, $L_1$ distance is excluded from the analysis. The best mean values are highlighted in \textbf{bold}, while the second-best are marked with an \underline{underline}.}
    \vspace*{-0.5em}
    \label{tab:large-grid-results}
    \resizebox{0.9\linewidth}{!}{
    \begin{tabular}{@{}lcccc}
        \toprule
        Grid config. $\rightarrow$
        & $d = 4$, $H = 64$ 
        & $d = 4$, $H = 128$ 
        & $d = 6$, $H = 32$ 
        & $d = 6$, $H = 64$ 
        \\
        \cmidrule(lr){2-2}\cmidrule(lr){3-3}\cmidrule(lr){4-4}\cmidrule(lr){5-5}
        Num. terminal states $\vert \mathcal{X} \vert$ $\rightarrow$
        & $1.68 \times 10^7$
        & $2.68 \times 10^8$
        & $1.06 \times 10^9$
        & $6.85 \times 10^{10}$
        \\
        \midrule
        TB \hfill(on-policy $\rightarrow$)
        & $24.0 \pm \text{\footnotesize 5.7}$
        & $228.7 \pm \text{\footnotesize 38.1}$
        & $0.3 \pm \text{\footnotesize 0.5}$
        & $3.7 \pm \text{\footnotesize 2.6}$
        \\
        \hspace{2pt} + $\epsilon$-expl.
        & $49.7 \pm \text{\footnotesize 17.2}$
        & $\mathbf{866.7} \pm \text{\footnotesize 154.4}$
        & $0.7 \pm \text{\footnotesize 0.5}$
        & $\underline{11.3} \pm \text{\footnotesize 4.0}$
        \\
        \hspace{2pt} + GAFN
        & $42.0 \pm \text{\footnotesize 6.5}$
        & $180.0 \pm \text{\footnotesize 51.9}$
        & $1.3 \pm \text{\footnotesize 1.2}$
        & $4.0 \pm \text{\footnotesize 2.2}$
        \\
        \hspace{2pt} + PRT
        & $\underline{119.3} \pm \text{\footnotesize 16.0}$
        & $222.7 \pm \text{\footnotesize 11.6}$
        & $\underline{6.7} \pm \text{\footnotesize 0.5}$
        & $7.3 \pm \text{\footnotesize 0.9}$
        \\
        \hspace{2pt} + PER
        & $70.7 \pm \text{\footnotesize 12.3}$
        & $164.3 \pm \text{\footnotesize 23.7}$
        & $2.0 \pm \text{\footnotesize 1.4}$
        & $5.3 \pm \text{\footnotesize 0.9}$
        \\
        \hspace{2pt} + \teacher (\emph{ours})
        & $\mathbf{299.0} \pm \text{\footnotesize 10.7}$
        & $\underline{728.0} \pm \text{\footnotesize 192.0}$
        & $\mathbf{9.7} \pm \text{\footnotesize 3.4}$
        & $\mathbf{21.3} \pm \text{\footnotesize 6.0}$
        \\
        \bottomrule
    \end{tabular}
    }
\end{table*}

\subsection{Large-scale experiment on deceptive hypergrids}

We evaluate our algorithm on larger-scale settings of deceptive grid worlds. Details of the grid configurations and the total number of terminal states (representing the size of the search space) are provided in \cref{tab:large-grid-results}. We use the same experimental settings we used for the grid with ($d=4, H=32$), which are described in~\cref{app:gridworld}. Note that we omit $L_1$ distance from our metrics, as calculating the exact target distribution is computationally infeasible for these large-scale problems.

As shown in the results~\cref{tab:large-grid-results}, our algorithm generally outperforms other baselines even in environments with a much larger search space. We conjecture the strong performance of $\epsilon$-exploration in the ($d=4, H=128$) configuration is because, when $H$ is large, clusters of adjacent high-reward states are large. In such cases, the random, brittle actions from $\epsilon$-exploration can be advantageous in discovering multiple adjacent modes within the same region.

\subsection{Sampling LLM attack prompts}

\begin{figure}[htbp]
  \centering
  \begin{minipage}[b]{0.45\textwidth}
    \centering
    \includegraphics[width=\textwidth]{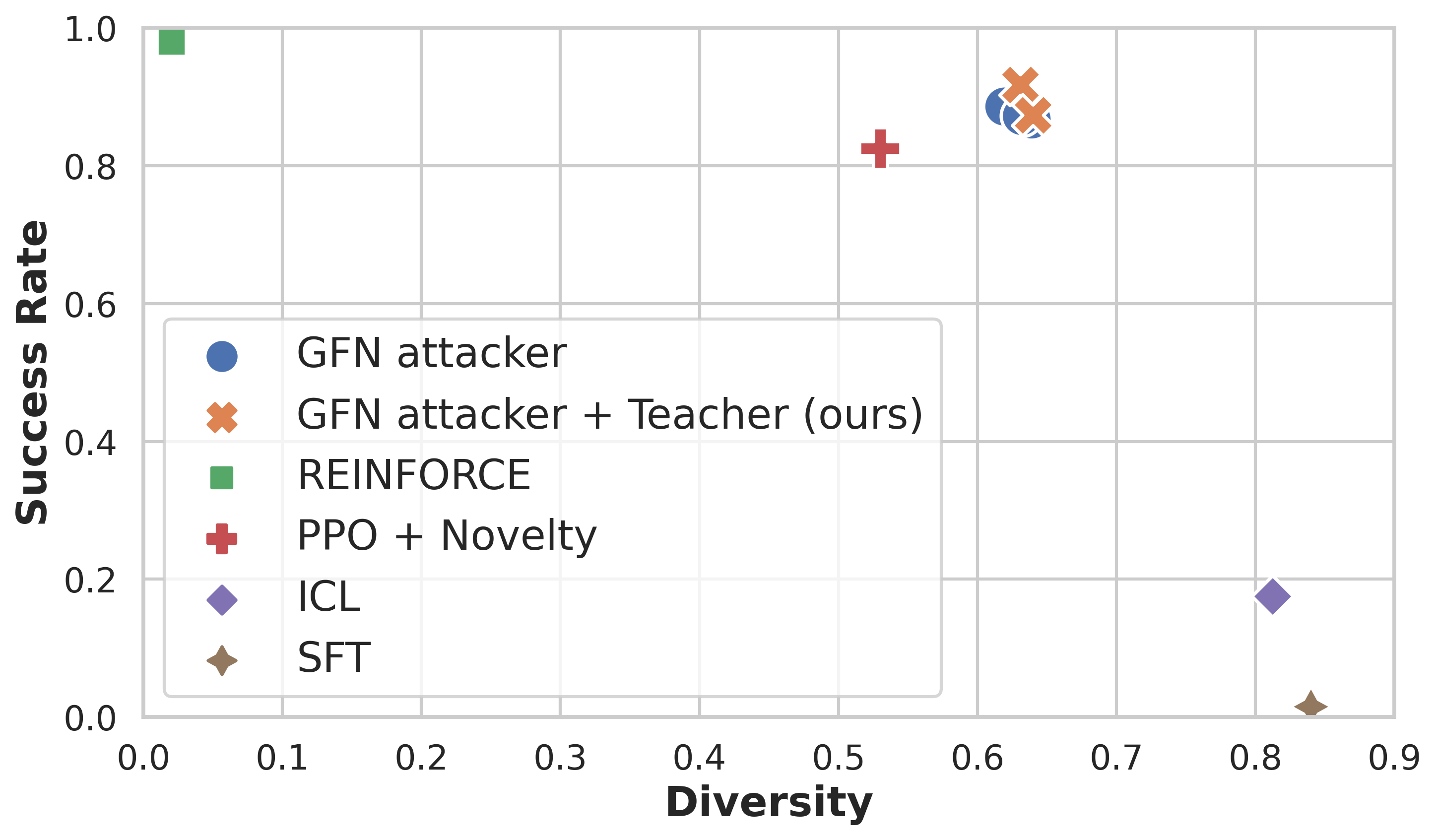}
    \subcaption{vs. All baselines}
    \label{fig:overall_result_redteam}
  \end{minipage}
  \hfill
  \begin{minipage}[b]{0.45\textwidth}
    \centering
    \includegraphics[width=\textwidth]{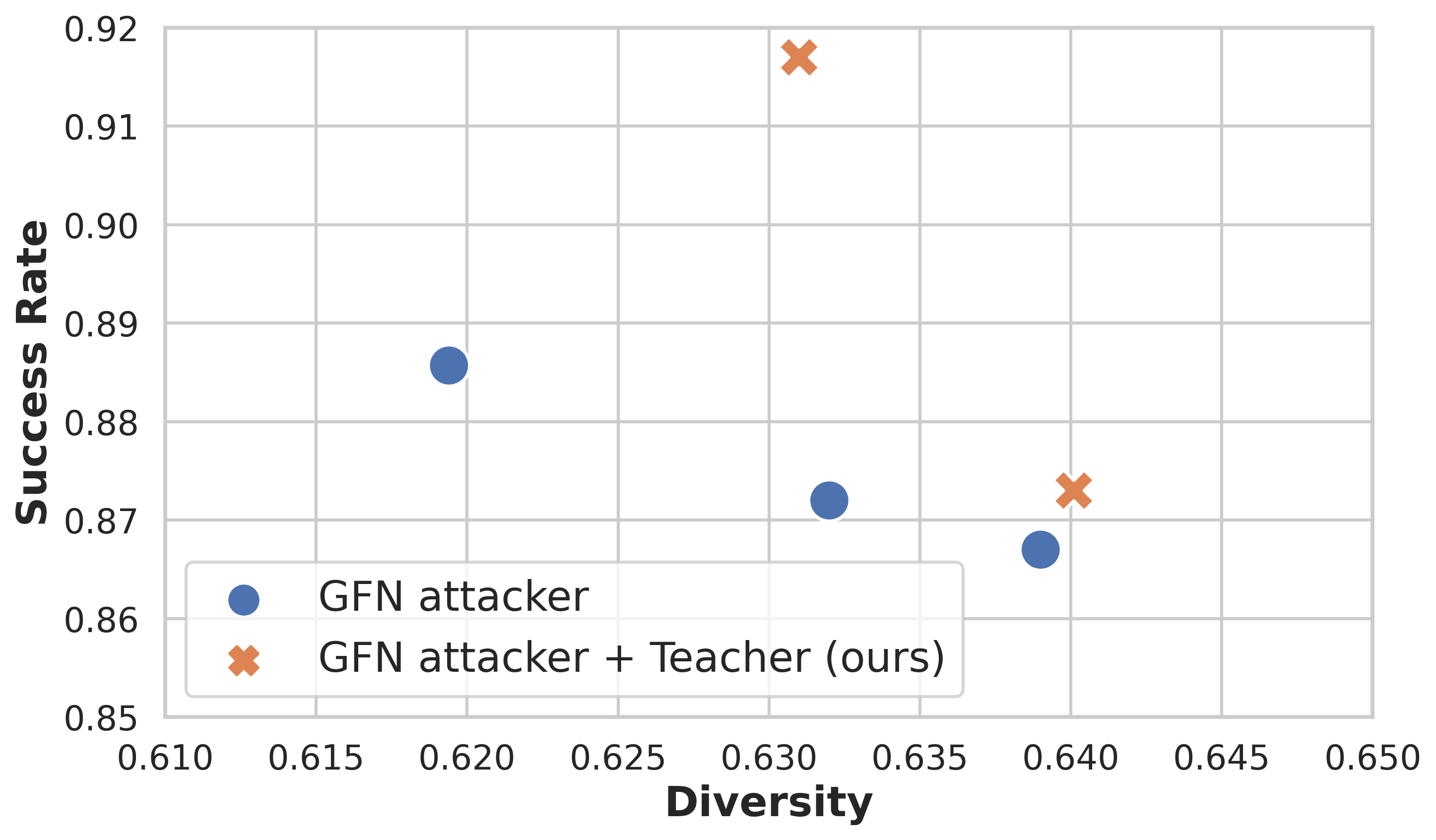} 
    \subcaption{vs. GFN methods}
    \label{fig:zoom_in_results}
  \end{minipage}
\label{fig:results_redteam}
\caption{Percentage of toxic prompts (y-axis) versus prompt diversity (x-axis, measured as 1$-$cosine similarity) for attack methods on GPT-2.}
\end{figure}

We benchmark our off-policy training method on the automated red-teaming task using GFlowNets, following the approach of \citet{lee2024learning}. In this task, the problem is formulated as inference over prompt sequences proportional to a target reward, which is computed by a toxicity classifier evaluated on the response the prompt induces from a fixed victim model. We adhere to the same settings, baselines, evaluation model, and model architecture as the mentioned work. Specifically, we fine-tune GPT-2 as a GFlowNet policy to serve as the attack model.

\paragraph{Setting.} The log-reward is defined as a weighted mixture of the log-likelihood of the language model and the toxicity score of the prompt. Toxicity is evaluated by a classifier that measures how likely the prompt is to induce toxic outputs from the victim model.

\paragraph{Baselines.} We compare our method with several baselines: In-Context Learning (ICL), Supervised Fine-Tuning (SFT), REINFORCE, and Proximal Policy Optimization (PPO) with a novelty reward as suggested by \citet{hong2024curiosity}. We also include the GFlowNet attacker proposed by \citet{lee2024learning}, which is the state-of-the-art method leveraging a GFlowNet sampler with sophisticated off-policy techniques using a replay buffer.

\paragraph{Implementation.} We implement a teacher network over the GFlowNet attacker. While the original GFlowNet attacker uses on-policy updates and a 1:1 ratio in the replay buffer, we use Teacher, Student (on-policy), and replay buffer trajectories in a 2:1:3 ratio. We aim to observe whether this adjustment provides any benefits over the baseline. We reproduce the GFN attacker results with $\beta \in {0.06, 0.07, 0.08}$, where $\beta$ is the temperature parameter for the toxicity reward. For our teacher method, we test with $\beta \in {0.02, 0.05}$. For other baseline results, we directly use the data reported in the figures by \citet{lee2024learning}.

\paragraph{Results.} As shown in \cref{fig:overall_result_redteam}, the teacher network achieves slightly higher diversity and success rates compared to the state-of-the-art GFlowNet attacker. Other baselines fail to produce both diverse and toxic prompts: REINFORCE leads to mode collapse, and ICL and SFT do not generate meaningful toxic prompts (see \citet{lee2024learning} for a more detailed analysis of these baselines). The GFlowNet attacker provides well-balanced results, achieving high toxicity in successful prompt sentences with diversity. Our Teacher network offers a slight improvement over the basleine GFlowNet attacker method (see \cref{fig:zoom_in_results}), demonstrating that our approach can be flexibly applied to real-world tasks. Notably, even when using a lower $\beta$ (indicating a peaky toxicity reward) than the GFlowNet attacker, the diversity achieved is higher. This suggests that the teacher encourages exploration into missing modes to enhance diversity.

These findings show the potential applicability of the teacher concept in large language model reasoning tasks where amortized inference using off-policy RL has been applied, including automated red-teaming, infilling, chain-of-thought reasoning, and planning \citep[as studied in][]{hu2023amortizing,song2024latent,yu2024flow}.

\end{document}